%% file: main.tex
\documentclass[preprint,12pt,nopreprintline]{elsarticle}
\usepackage{amsmath, amssymb, amsthm, bbm}
\usepackage{enumitem}
\usepackage{graphicx}
\usepackage[UKenglish]{isodate}
\usepackage[font=small,labelfont=it]{caption}
\usepackage[margin=1in]{geometry}
\usepackage{setspace}
\usepackage{mathtools}
\usepackage{booktabs}
\usepackage{array}
\usepackage{placeins}
\usepackage[hidelinks]{hyperref}
\hypersetup{
  pdftitle={Online Learning of Scale Parameters in Score-Driven Filters},
  pdfauthor={Fabrizio Lillo, Giulia Livieri, and Gianluca Palmari},
  pdfsubject={Online gain learning for score-driven filters},
  pdfkeywords={score-driven filters, online learning, mirror descent, dynamic regret}
}
\mathtoolsset{showonlyrefs}
\DeclareMathOperator{\inter}{int}
\DeclareMathOperator{\diag}{diag}

\renewcommand{\phi}{\varphi}

\let\emptyset\varnothing

\DeclareMathOperator*{\argmin}{arg\,min}

\newcommand{\ud}{\mathrm{d}}

\def \ud{\mathrm{d}}

\newtheorem*{theorem*}{Theorem}
\newtheorem{theorem}{Theorem}[section]

\newtheorem{proposition}[theorem]{Proposition}
\newtheorem{corollary}[theorem]{Corollary}
\newtheorem{remark}[theorem]{Remark}

\newtheorem{example}[theorem]{Example}
\newtheorem{assumption}[theorem]{Assumption}
\newtheorem{definition}[theorem]{Definition}
\newcommand{\proofref}[1]{\par\noindent The proof appears in \ref{#1}.\par}

\makeatletter
\@namedef{subjclassname@2020}{\textup{2020} Mathematics Subject Classification}
\makeatother

\begin{document}
\onehalfspacing
\raggedbottom
\begin{frontmatter}

\title{Online Learning of Scale Parameters in Score-Driven Filters}

\author[sns]{Fabrizio Lillo}
\ead{fabrizio.lillo@sns.it}
\author[lse]{Giulia Livieri}
\ead{g.livieri@lse.ac.uk}
\author[sns]{Gianluca Palmari\corref{cor1}}
\cortext[cor1]{Corresponding author.}
\ead{gianluca.palmari@sns.it}

\affiliation[sns]{organization={Scuola Normale Superiore},
  addressline={Piazza dei Cavalieri 7}, city={Pisa}, country={Italy}}
\affiliation[lse]{organization={Department of Statistics, London School of Economics and Political Science},
  addressline={Columbia House, 69 Aldwych}, city={London}, country={United Kingdom}}

\begin{abstract}
A score-driven filter multiplies its scaled log-likelihood score by a scale parameter. We call this coefficient the gain and learn it online. Given the current state and realised scaled score, each admissible gain selects a reachable next state and predictive density. A scalar gain moves along a line; diagonal gains control coordinatewise transmission and may change direction. We evaluate gain selection using a one-step predictive Kullback--Leibler objective. In the scalar unscaled case, the negative consecutive-score product is a stochastic gradient; the positive product used in accelerated recursions is a descent direction. Positive scalar score scaling changes only the effective learning rate. Strictly increasing $C^1$ gain links with positive derivative induce mirror-descent geometry, while persistence adds a Bregman pull towards a reference gain. Under convexity, compactness, integrability, and schedule conditions, projected and discounted mirror updates satisfy dynamic-regret bounds relative to time-varying, current-information comparators. Simulations isolate score scaling, link geometry, persistence, and coordinatewise gains. Across twelve equity indices, the bounded discounted-logistic gain records a lower out-of-sample mean negative log score than the constant gain in eleven markets, although market-level evidence is mixed. It also avoids the extreme transients of the numerically capped exponential-link benchmark.
\end{abstract}


\end{frontmatter}

\section{Introduction}\label{section::Introduction}
Score-driven models\footnote{This class is also known as the generalised
autoregressive score or dynamic conditional score class.}, introduced by
\cite{creal2013gas} and~\cite{harvey2013dynamic}, are observation-driven
time-varying parameter models in Cox's~\cite{cox1981statistical} classification.
Their time-varying parameters are deterministic functions of past observations and
previous parameter values, unlike the latent processes of parameter-driven models
\cite{koopman2016predicting}. Updates use the conditional log-likelihood score,
possibly scaled by a local information or curvature matrix. This structure supports
tractable likelihood inference and applications to dynamic volatility, intensities,
dependence, and other economic and financial
features (e.g.,~\cite{creal2013gas,creal2011dynamic,harvey2013dynamic,koopman2015dynamic,lucas2014conditional}).\footnote{Applications are catalogued at \url{https://www.gasmodel.com}.}

For likelihood-based\footnote{Filters based on other proper scoring rules instead use the corresponding scoring-rule divergence (e.g.,~\cite{depunder2026prada, de2024kullback, dawid2016minimum, catania2026noise,livieri2026scoring}).}
score-driven models, expected log loss equals Kullback--Leibler divergence up to a
term independent of the model parameter. Under misspecification, minimising it
selects a pseudo-true parameter whose model density is closest to the true conditional
density; under correct specification and identification, it selects the true parameter.
Existing optimality results use this criterion. Blasques et
al.~\cite{blasques2015information,blasques2018correction} show that the score direction
locally reduces the divergence between the true and model-implied densities. Gorgi et
al.~\cite{gorgi2024optimality} show, under regularity and monotonicity conditions, that
the expected update approaches the pseudo-true value. De Punder et
al.~\cite{de2024kullback} characterise small-step improvement through alignment
between the expected update and the score.

These results address whether the score provides a suitable update direction, but
not how far the filter should move along it. This distinction matters: even a suitable
direction can make the filter react too slowly or too aggressively when multiplied by
a poorly chosen coefficient. In multivariate models, a common coefficient rescales a
fixed direction, whereas componentwise coefficients can also change that direction.
We therefore take the working score and its scaling rule as given and ask how the
remaining coefficient should be learned and evaluated.

Blasques et al.~\cite{blasques2019accelerating} refer to the static coefficient
multiplying the scaled score as a scale parameter and make it time varying in
\textrm{aGAS}. To distinguish this coefficient from the scaling applied to the score,
we call it the \emph{gain} throughout. In a Gaussian local-level model, the gain
specialises to the Kalman gain~\cite{durbin2012time}. We ask which predictive objective,
geometry, and benchmark should govern its online updating.

Both our formulation and \textrm{aGAS} use \(\mathcal F_t\)-timed score-product
feedback: after observing \(y_t\), the recursion forms
\(\alpha_{\lambda,t}=g(f_{t+1})\) and uses it to update \(\lambda_{t+1}\).
\textrm{aGAS} motivates its latent autoregression through acceleration and local
same-period Kullback--Leibler reduction. We give the score product a complementary
interpretation: in the scalar unscaled case, consecutive score innovations form a
stochastic descent signal for next-period conditional predictive loss; positive
score scaling preserves this direction and changes only the effective step size.
The gain link determines mirror geometry, persistence adds a Bregman pull towards
a reference gain, and dynamic regret evaluates performance against
\(\mathcal F_t\)-measurable moving gains on score-conditioned reachable sets. The
exact \textrm{aGAS} recursive representation does not by itself establish regret
coverage; Section~\ref{sec::SectionFive} states the additional geometry and schedule
conditions.

Conditional on the current state, realised score, and scaling rule, admissible gains
index reachable next states and hence next-period predictive densities. We therefore
evaluate gain selection with a conditional one-step predictive criterion.

Our contributions are threefold.
\begin{enumerate}
\item \emph{A conditional gain objective and its score-product gradient.} Each
admissible gain produces a counterfactual post-observation state and hence a
predictive density. We evaluate this density by its conditional expected negative
log-likelihood. Up to a gain-independent conditional entropy term, the objective
selects, from the densities reachable through the realised score direction, the
one closest in Kullback--Leibler divergence to the true next-period conditional
density. A gain magnitude is therefore meaningful only relative to the score and
scaling rule that determine its state displacement. In the scalar unscaled case,
the negative product of consecutive scores is a stochastic gradient of the gain
objective. Positive score scaling, as used in
\textrm{aGAS}~\cite{blasques2019accelerating}, preserves the descent direction
and changes only the effective learning rate.
\item \emph{Link-induced mirror geometry and persistence.} The link used to keep gains
positive, bounded, or otherwise admissible also determines the geometry in which
they are learned. Changing the link changes how the update moves through the
admissible gain domain without changing the predictive objective; the link
therefore acts as a mirror map. A persistent latent recursion adds a Bregman pull
towards a reference gain. In the usual constant-parameter intercept form, the
intercept and autoregressive coefficient jointly determine that reference gain,
while the autoregressive coefficient also controls memory.
\item \emph{A dynamic-regret benchmark.} Under convexity, compactness,
integrability, and the stated schedule conditions, projected and localised mirror
updates satisfy dynamic-regret bounds against time-varying,
$\mathcal F_t$-measurable comparator gains. For the persistent update, the bound
includes an explicit remainder involving the reference gain. The original
data-dependent \textrm{aGAS} recursion has the same
algebraic mirror representation, but the regret bounds apply to it only when all
of the stated additional conditions are verified.
\end{enumerate}

The paper connects score-driven optimality and stability
\cite{blasques2015information,gorgi2024optimality,beutner2026consistency,donker2025stability},
time-varying-gain \textrm{aGAS} models~\cite{blasques2019accelerating}, and online
convex optimisation and dynamic regret
\cite{nemirovsky1983problem,beck2003mirror,zinkevich2003online,besbes2015nonstationary,hazan2016introduction}.
The regret results compare the learned and comparator gain sequences on the same
realised, score-conditioned reachable sets. They are not policy-regret guarantees
for counterfactual filter paths, which would generate different states and scores.
Gain magnitudes and admissible intervals are therefore interpreted relative to the
chosen score scaling. Within each numerical or empirical comparison, methods use
the same observations and model form but generally generate their own states,
scores, and reachable sets. The comparisons therefore evaluate complete filtering
procedures rather than directly instantiate the theorem's common-context benchmark.

Section~\ref{sec::SectionOne} formulates the conditional gain objective,
Section~\ref{sec::SectionTwo} identifies its score-pullback gradient, and
Section~\ref{sec::SectionThree} develops the mirror geometry. Section~\ref{sec::SectionFive}
presents the benchmark and regret theory. Sections~\ref{sec::SectionSix}
and~\ref{sec:empirical} report numerical and empirical evidence, respectively,
and Section~\ref{sec::Conclusion} concludes.
\section{A conditional gain problem for score-driven filters}\label{sec::SectionOne}
Once the realised scaled score is known, admissible gains index reachable next
states: a line segment parallel to that score direction for a common scalar gain
and a coordinate-aligned box for a diagonal gain. This section constructs the
reachable set and represents gain choice as a conditional Kullback--Leibler
projection onto the corresponding predictive family.
\paragraph{\textbf{Notation and conventions}}
We distinguish four objects throughout. The \emph{data-generating process}
determines the true conditional law $\widetilde Q_t$ of $y_t$ given
$\mathcal{F}_{t-1}$, with density $\tilde p_t$. The \emph{working model} is the
possibly misspecified parametric family
$P_\lambda(\mathrm dy)=p(y\mid\lambda)\nu(\mathrm dy)$. The \emph{filter} is the
score-driven recursion that produces the state $\lambda_t$. The \emph{gain
problem} concerns the choice of $\alpha_t\in\mathcal{A}$. We write
$\mathcal{F}_{t-1}=\sigma(y_{t-1},y_{t-2},\ldots)$ and
$\mathsf{E}_s[\cdot]=\mathsf{E}[\cdot\mid\mathcal{F}_s]$. After $y_t$ is observed,
the scaled score $d_t=S(\lambda_t)s(y_t,\lambda_t)$ is revealed and is
$\mathcal{F}_t$-measurable. The gain $\alpha_t$ is then chosen as an
$\mathcal{F}_t$-measurable decision and committed before the forecast outcome
$y_{t+1}$ is observed. In the theoretical framework, the realised predictive loss
is negative log-likelihood, and its conditional expectation equals predictive
Kullback--Leibler divergence up to a term independent of the gain. We reserve
$\alpha$ for the
\emph{gain}; Section~\ref{sec::SectionThree} introduces the learning rate,
persistence, discounting, reference gain, and latent coordinate.

\subsection{Setting and assumptions}
Let \((\mathcal Y,\mathcal B_{\mathcal Y})\) be the observation space and
\((y_t)_{t\in\mathbb Z}\) the observed process, with natural filtration
\(\mathcal F_{t-1}=\sigma(y_{t-1},y_{t-2},\ldots)\).\footnote{In the truncated case $\mathcal{F}_{t-1}:=\sigma(y_{t-1},y_{t-2},\ldots,y_{1})$ with $\mathcal{F}_{0}=\{\emptyset,\Omega\}$, non-trivial time-zero randomness is accounted for by enlarging the initial sigma-field.}
Assume that the true conditional law \(\widetilde Q_t\) is dominated by a fixed
sigma-finite measure \(\nu\) (for example, Lebesgue or counting measure), and
denote its conditional density by \(\tilde p_t\). Conditional expectations are
taken under the true law, with
\(\mathsf E_s[\cdot]:=\mathsf E[\,\cdot\mid\mathcal F_s]\). Hence, for every
integrable measurable function \(h\),
\[
    \mathsf E_{t-1}[h(y_t)]
    =
    \int_{\mathcal Y} h(y)\tilde p_t(y)\nu(\mathrm dy).
\]

We specify the score-driven filter relative to the parametric predictive family
\[
    P_\lambda(\mathrm dy)=p(y\mid\lambda)\nu(\mathrm dy),
    \qquad \lambda\in\Lambda\subset\mathbb R^m,
    \quad m\in\mathbb N_{\ge 1}.
\]
This working model may be misspecified relative to the true conditional density
\(\tilde p_t\).

Following \cite{creal2013gas} and~\cite{harvey2013dynamic}, define the likelihood score
\[
    s(y_t,\lambda_t)
    :=
    \frac{\partial \log p(y_t\mid\lambda_t)}{\partial\lambda_t},
\]
whenever the derivative exists. A first-order score-driven filter then updates
the time-varying state according to
\begin{equation}\label{eq::ScoreDrivenCrealandHarvey}
    \lambda_{t+1}
    =
    c+B\lambda_t+A\, S(\lambda_t)s(y_t,\lambda_t),
\end{equation}
where \(c\in\mathbb R^m\), the \emph{gain} matrix \(A\) and persistence matrix
\(B\) lie in \(\mathsf M_m(\mathbb R)\), and
\(S:\Lambda\to\mathsf M_m(\mathbb R)\) is a deterministic scaling rule. In
general, \(\mathsf M_{m,n}(\mathbb R)\) denotes the set of \(m\times n\) real
matrices, with \(\mathsf M_m(\mathbb R)=\mathsf M_{m,m}(\mathbb R)\). In the
Newton-score specification, \(S(\lambda_t)\) is an inverse information or inverse
curvature matrix that preconditions the raw score before it enters the recursion
(see
\cite{blasques2015information, blasques2019accelerating, gorgi2024optimality}).

To isolate the gain while retaining the standard intercept and state persistence,
define the predictable no-gain baseline
\(
    b_t:=c+B\lambda_t.
\)
We use a time-varying componentwise gain as the running case,
\(A_t=\diag(\alpha_t)\) with \(\alpha_t\in\mathbb R^m\). The adaptive update
becomes
\begin{equation}\label{eq::ScoreDrivenCrealandHarvey2}
    \lambda_{t+1}
    =
    b_t+\alpha_t\odot S(\lambda_t)s(y_t,\lambda_t),
\end{equation}
where \(\odot\) denotes the Hadamard product.
The score and scaling rule determine the signal
\(S(\lambda_t)s(y_t,\lambda_t)\); the gain \(\alpha_t\) determines its
coordinatewise transmission to the next state.

At time \(t-1\), \(\lambda_t\) is \(\mathcal F_{t-1}\)-measurable. After
observing \(y_t\), the scaled score
\begin{equation}
\label{eq:new_score_signal}
d_t:=S(\lambda_t)s(y_t,\lambda_t)
\end{equation}
is realised and \(\mathcal F_t\)-measurable. When \(S(\lambda_t)\) is an
inverse-information or inverse-curvature matrix, \(d_t\) is the Newton-score
signal. The gain determines how this realised signal is transmitted to the next
state.

We represent both common scalar and componentwise gains through a general linear
loading. Let \(\mathcal A\subset\mathbb R^k\) be non-empty, compact, and convex,
let \(a\in\mathcal A\), and let \(H_t\in\mathbb R^{m\times k}\) be an
\(\mathcal F_t\)-measurable linear loading of the realised scaled score \(d_t\).
Conditional on \(\mathcal F_t\), define the candidate next state
\begin{equation}\label{eq::postobservation}
    \hat\lambda_t(a)
    :=
    b_t+H_t\,a.
\end{equation}
Two gains \(a,b\in\mathcal A\) are observationally equivalent at date \(t\)
whenever \(H_t(a-b)=0\). Two cases recur. For a \emph{diagonal} gain, \(k=m\) and
\(H_t=\diag(d_t)\), so \(H_ta=a\odot d_t\). For a \emph{common scalar} gain,
\(k=1\) and \(H_t=d_t\), so \(a\in\mathcal A\subset\mathbb R\) scales the entire
score. Interiors are always taken in the native gain space \(\mathbb R^k\);
otherwise, for \(m>1\), embedding a scalar interval as
\(\{a\mathbf 1_m:a\in[\alpha_L,\alpha_H]\}\) in \(\mathbb R^m\) would give it
empty interior. Accordingly, all differentiation and mirror statements below use
\(\inter(\mathcal A)\) relative to \(\mathbb R^k\). The corresponding reachable
set
\begin{equation}\label{eq::attainableset}
    \mathcal R_t
    :=
    \{\hat\lambda_t(a):a\in\mathcal A\}
\end{equation}
is the affine image of \(\mathcal A\) under \(a\mapsto b_t+H_t a\). We require
\(\mathcal R_t\subset\Lambda\) almost surely. With a common scalar gain,
\(\mathcal R_t=b_t+d_t\mathcal A\), which is a possibly degenerate line segment.
It contains the no-gain baseline \(b_t\) only if \(0\in\mathcal A\). With a
diagonal gain, a product set \(\mathcal A\), and \(H_t=\diag(d_t)\), the reachable
set is an axis-aligned box, again possibly degenerate. In all cases, compactness
of \(\mathcal A\) makes \(\mathcal R_t\) bounded.

\begin{definition}[Gain policy]\label{def::gain_policy}
Let \(\mathcal A\subset\mathbb R^k\) be non-empty, compact, and convex, with
non-empty interior in \(\mathbb R^k\). A gain policy is a sequence
\((\alpha_t)_{t\ge 1}\) in which each \(\alpha_t\) is
\(\mathcal F_t\)-measurable and takes values in \(\mathcal A\).
\end{definition}
Thus \(\alpha_t\) may react to \(y_t\) through \(d_t\), but is committed
\emph{before} the forecast outcome \(y_{t+1}\). The timing is
\[
\begin{aligned}
    \lambda_t\ (\mathcal F_{t-1})
    &\ \longrightarrow\ y_t\ \text{observed}
    \ \longrightarrow\ \underbrace{d_t=S(\lambda_t)s(y_t,\lambda_t)\ \text{revealed}}_{\mathcal F_t\text{-measurable}} \\
    &\ \longrightarrow\ \underbrace{\alpha_t\ \text{chosen}}_{\mathcal F_{t}\text{-measurable}}
    \ \longrightarrow\ \lambda_{t+1}=b_t+H_t\alpha_t .
\end{aligned}
\]
For the diagonal running case, the last expression is
\(\lambda_{t+1}=b_t+\alpha_t\odot d_t\).
This matches the \textrm{aGAS} timing of~\cite{blasques2019accelerating}; see
Remark~\ref{rem::agas}. More explicitly, the calendar-time convention of
\cite{blasques2019accelerating} is
\[
\underbrace{(\lambda_t,f_t)}_{\mathcal F_{t-1}}
\ \longrightarrow\ y_t
\ \longrightarrow\ u_{\lambda,t}
\ \longrightarrow\ f_{t+1}
\ \longrightarrow\ \alpha_{\lambda,t}=g(f_{t+1})
\ \longrightarrow\ \lambda_{t+1}.
\]
Thus \(\alpha_{\lambda,t}\) reacts to \(y_t\), is \(\mathcal F_t\)-measurable,
and is used immediately in the update from \(\lambda_t\) to
\(\lambda_{t+1}\). We identify it with our \(\alpha_t\). The translation to the
standard online-learning transition \(\alpha_t\mapsto\alpha_{t+1}\) is stated
once in Remark~\ref{rem::agas}. Once \(y_t\) is observed, \(d_t\) and
\(\alpha\mapsto\hat\lambda_t(\alpha)\) are known. We study
\[
    \lambda_{t+1}=\hat\lambda_t(\alpha_t).
\]
In this local analysis, the baseline intercept \(c\), persistence matrix \(B\),
and any static density parameters are held fixed. The pair \((c,B)\) determines
the $\mathcal F_{t-1}$-measurable baseline $b_t$; conditional on \(\lambda_t\),
neither changes the realised score direction $d_t$.
Section~\ref{sec::SectionThree} introduces the separate intercept and persistence
terms used in the gain recursion.

\subsection{Predictive loss for gain selection and Kullback--Leibler representation}
Once \(y_t\) is observed, each gain \(\alpha\in\mathcal A\) yields the state
\(\hat\lambda_t(\alpha)\) and predictive density
\(p(\cdot\mid\hat\lambda_t(\alpha))\) for \(y_{t+1}\). Its realised log-loss is
\begin{equation}\label{eq::scoringrule}
    L_{t+1}(\alpha)
    :=
    -\log p(y_{t+1}\mid\hat\lambda_t(\alpha)).
\end{equation}
The corresponding conditional predictive loss for gain selection is
\begin{equation}\label{eq::scoringruleexpected}
    J_{t+1\mid t}(\alpha)
    :=
    \mathsf E_t
    \big[
        L_{t+1}(\alpha)
    \big]
    =
    \mathsf E_t
    \big[
        -\log p(y_{t+1}\mid\hat\lambda_t(\alpha))
    \big].
\end{equation}
Throughout, we work with a fixed jointly measurable version of this
conditional objective that is stable under measurable substitution: for every
$\mathcal F_t$-measurable selected gain $\alpha_t$ satisfying
$\mathsf E_t[|L_{t+1}(\alpha_t)|]<\infty$ almost surely,
\[
J_{t+1\mid t}(\omega,\alpha_t(\omega))
=\mathsf E_t[L_{t+1}(\alpha_t)]
\qquad\text{almost surely}.
\]
Under the conditional-kernel and density conditions of
Proposition~\ref{prop::kl_representation}, the conditional-integral version
\[
J_{t+1\mid t}(\omega,\alpha)
=-\int_{\mathcal Y}
  \log p\!\left(y\mid\hat\lambda_t(\omega,\alpha)\right)
  \tilde p_{t+1}(\omega,y)\nu(\mathrm dy)
\]
has this substitution property. All evaluations of the conditional objective
at random gains use this same version.
The learner does not know this $\mathcal F_t$-conditional objective. Under the
conditions stated in Section~\ref{sec::SectionTwo}, the score pullback observed
at $t+1$ is a conditionally unbiased stochastic gradient. If the gain were
instead chosen before \(y_t\), the relevant objective would be the pre-score
expectation
\begin{equation}\label{eq::decision_time_gain_loss}
    \bar J_{t\mid t-1}(\alpha)
    :=
    \mathsf E_{t-1}
    \big[
        J_{t+1\mid t}(\alpha)
    \big].
\end{equation}
Under our timing, for any gain policy \((\alpha_t)\), the
\(\mathcal F_t\)-measurability of \(\alpha_t\) and the tower property give,
whenever these expectations are finite,
$
    \mathsf E
    \big[
        J_{t+1\mid t}(\alpha_t)
    \big]
    =
    \mathsf E
    \big[
        \mathsf E_t[L_{t+1}(\alpha_t)]
    \big]
    =
    \mathsf E
    \big[
        L_{t+1}(\alpha_t)
    \big],
$
where \(\mathsf E\) denotes unconditional expectation.

The two objectives use different information sets. The post-score objective
conditions on the realised score direction and reachable set facing the gain
choice, whereas the pre-score objective averages over score directions before
they are observed. The former therefore pairs with an $\mathcal F_t$-measurable
comparator, the latter with an $\mathcal F_{t-1}$-measurable (predictable)
comparator. We use the post-score formulation because it matches the timing of
the adaptive recursion; the pre-score criterion shows how the oracle and gradient
change when the gain must be selected earlier.

Conditional on \(\mathcal F_t\), the admissible gain set restricts the predictive
criterion to states reachable under the realised score and chosen scaling rule.
The objective is therefore the negative log-likelihood pulled back through
\(\alpha\mapsto\hat\lambda_t(\alpha)\). Proposition~\ref{prop::kl_representation}
decomposes this objective into conditional entropy, which is gain independent,
and Kullback--Leibler divergence from the true conditional law to each reachable
model law.
\begin{proposition}[Conditional Kullback--Leibler representation]\label{prop::kl_representation}
Let $\widetilde Q_{t+1}(\ud y)$ be a version of the conditional law of
$y_{t+1}$ given $\mathcal F_t$. Suppose that $\widetilde Q_{t+1}$ is dominated
by a fixed sigma-finite measure $\nu$ and admits an
$\mathcal F_t\otimes\mathcal B(\mathcal Y)$-measurable density
$\tilde p_{t+1}(y)$ with respect to $\nu$. For
$\lambda\in\Lambda\subset\mathbb R^m$, let
\[
P_\lambda(\ud y)=p(y\mid\lambda)\nu(\ud y)
\]
denote the model law. Let $\hat\lambda_t(\alpha)$ be the candidate next state in
Eq.~\eqref{eq::postobservation}. Define the conditional entropy
$C_{t+1}:=-\int_{\mathcal Y}\tilde p_{t+1}(y)\log\tilde p_{t+1}(y)\nu(\ud y)$,
and assume that $C_{t+1}\in\mathbb R$ almost surely. For every
$\alpha\in\mathcal A$, assume that the cross-entropy
$-\int_{\mathcal Y}\tilde p_{t+1}(y)
\log p(y\mid\hat\lambda_t(\alpha))\nu(\ud y)$ is a well-defined (extended)
random variable taking values in $\mathbb R\cup\{+\infty\}$. Then,
\begin{equation}
    J_{t+1\mid t}(\alpha)=C_{t+1} + \mathsf{KL}(\widetilde Q_{t+1}(\ud y), P_{\hat \lambda_t(\alpha)})
\end{equation}
almost surely, where $J_{t+1\mid t}(\alpha)$ is defined in Eq.~\eqref{eq::scoringruleexpected} and $\mathsf{KL}$ denotes the Kullback--Leibler divergence; we use the convention $0 \log(0/q)=0$ and $q \log(q/0)=+\infty$.
\end{proposition}
\proofref{app::proof_kl_representation}
\noindent Since \(C_{t+1}\) is gain independent, minimising
\(J_{t+1\mid t}\) selects a reachable model law with minimal
Kullback--Leibler divergence from the true conditional law. The admissible gain
set determines the collection of model laws over which this minimisation is
carried out.

The pullback makes gain normalisation explicit. For $c>0$, replacing $d_t$ by
$c d_t$ and a common scalar gain $\alpha$ by $\alpha/c$ leaves the state update
and predictive loss unchanged; the same holds coordinatewise for diagonal gains.
Thus gain intervals, reported magnitudes, and regret comparisons are specific to
the score and scaling rule.

The gain is an observation-driven scale parameter learned online. Under
misspecification, its \(\mathcal F_t\)-measurable oracle minimises predictive loss
over the working model's reachable family.

\section{The score signal as a gain gradient}\label{sec::SectionTwo}
In the scalar unscaled case, the negative product of consecutive scores is the
realised gradient of the one-step gain loss, so the positive product gives its
descent direction. Under the dominated-differentiation conditions stated below,
the same signal is an unbiased stochastic gradient of the conditional predictive
objective. The positive scalar multiplier used in \textrm{aGAS} preserves the
direction and rescales the effective learning rate, giving the feedback mechanism
of~\cite{blasques2019accelerating} an explicit variational interpretation.

After \(y_t\) is observed, a candidate gain \(\alpha\in\mathcal A\) affects the predictive
density of \(y_{t+1}\) through the reachable next state
$
    \hat\lambda_t(\alpha)
$ in Eq.~\eqref{eq::postobservation}.
The realised one-step loss is therefore the composition of the gain-to-state map
\(\alpha\mapsto\hat\lambda_t(\alpha)\) and the state-loss map
\(\lambda\mapsto-\log p(y_{t+1}\mid\lambda)\).

For each candidate gain \(\alpha\), define the next-period score at its reachable
state by
\begin{equation}
\label{eq:counterfactual_next_score}
    s_{t+1}(\alpha)
    :=
    s(y_{t+1},\hat\lambda_t(\alpha)).
\end{equation}
At the realised gain \(\alpha_t\), this becomes
$
    s_{t+1}(\alpha_t)
    =
    s(y_{t+1},\lambda_{t+1}).
$
The chain rule then gives
\begin{equation}\label{eq::gradientlogloss}
    G_{t+1}(\alpha)
    :=
    \nabla_{\alpha}L_{t+1}(\alpha)
    =
    -H_t^{\top}s_{t+1}(\alpha).
\end{equation}
Here \(H_t\) is the linear gain loading from Eq.~\eqref{eq::postobservation}.
Equation~\eqref{eq::gradientlogloss} pairs each gain coordinate with the
next-period score through the corresponding column of \(H_t\). For a diagonal
gain, \(H_t=\diag(d_t)\) and
\(G_{t+1}(\alpha)=-d_t\odot s_{t+1}(\alpha)\), so coordinate \(j\) uses the
product of \((d_t)_j\) and the corresponding next-score component at the
reachable state. For a common scalar gain, \(H_t=d_t\) and
\(G_{t+1}(\alpha)=-d_t^{\top}s_{t+1}(\alpha)\), which aggregates these
componentwise products into a single gradient.

In the scalar unscaled case, substituting \(d_t=s(y_t,\lambda_t)\) and
evaluating at the realised gain \(\alpha_t\) gives
\[
    G_{t+1}(\alpha_t)
    =
    -s(y_t,\lambda_t)s(y_{t+1},\lambda_{t+1}).
\]
Thus the \textrm{aGAS} score product signals local serial dependence and, through
the identity above, provides the descent direction for the realised one-step gain
loss.

Equation~\eqref{eq::gradientlogloss} is a pathwise identity. Its
stochastic-gradient interpretation for the conditional predictive objective
requires differentiation under conditional expectation. We impose two
dominated-differentiation conditions: Assumption~\ref{ass::local_condition} is
local, while Assumption~\ref{ass::global_condition} is uniform and supports
evaluation at \(\mathcal F_t\)-measurable random gains.
\begin{assumption}[Local conditional differentiation]\label{ass::local_condition}
Fix $t\in\mathbb Z$ and $\alpha\in\inter(\mathcal A)$. There exists an open
neighbourhood $U_\alpha\subset\inter(\mathcal A)$ of $\alpha$ such that:
\begin{itemize}
\item[(i)] $u \mapsto L_{t+1}(u)$ is almost surely differentiable on $U_{\alpha}$.
\item[(ii)] There exists a non-negative $\mathcal F_{t+1}$-measurable random
variable $M_{t+1,\alpha}$ such that $\mathsf E_t[M_{t+1,\alpha}]<\infty$ and
\[
    \sup_{u\in U_\alpha}\|\nabla_uL_{t+1}(u)\|
    \le M_{t+1,\alpha}
    \quad\text{almost surely}.
\]
\item[(iii)] The loss at the evaluated gain is conditionally integrable:
\[
    \mathsf E_t\!\left[|L_{t+1}(\alpha)|\right]<\infty
    \quad\text{almost surely}.
\]
\end{itemize}
\end{assumption}
\begin{assumption}[Uniform conditional differentiation]\label{ass::global_condition}
For every non-empty compact convex set
\(\mathcal K\subset\inter(\mathcal A)\):
\begin{itemize}
\item[(i)] The map \(u\mapsto L_{t+1}(u)\) is almost surely differentiable on
an open neighbourhood of \(\mathcal K\).
\item[(ii)] There exists a non-negative \(\mathcal F_{t+1}\)-measurable random
variable \(M_{t+1,\mathcal K}\) such that
\[
    \mathsf E_t[M_{t+1,\mathcal K}]<\infty,
    \qquad
    \sup_{u\in\mathcal K}
    \|\nabla_u L_{t+1}(u)\|
    \leq M_{t+1,\mathcal K}
    \quad\text{a.s.}
\]
\item[(iii)] For some \(u^\circ_{\mathcal K}\in\mathcal K\),
\[
    \mathsf E_t\!\left[|L_{t+1}(u^\circ_{\mathcal K})|\right]<\infty
    \quad\text{almost surely}.
\]
\end{itemize}
In addition, the map \((\omega,u)\mapsto G_{t+1}(\omega,u)\) admits a jointly
measurable version that is continuous in \(u\) almost surely.
\end{assumption}

The conditions can be verified through the score identity in
Eq.~\eqref{eq::gradientlogloss}. On a compact convex set \(\mathcal K\), suppose
the state loss is differentiable on the corresponding reachable set, the score
is jointly measurable and continuous in the state, and
\[
    \sup_{u\in\mathcal K}
    \|s(y_{t+1},\hat\lambda_t(u))\|
    \le \Phi_{t+1}
    \quad\text{almost surely},
    \qquad
    \mathsf E_t[\Phi_{t+1}]<\infty.
\]
Because \(H_t\) is \(\mathcal F_t\)-measurable,
Eq.~\eqref{eq::gradientlogloss} shows that
\(M_{t+1,\mathcal K}=\|H_t\|_{\mathrm{op}}\Phi_{t+1}\) is a conditionally
integrable envelope for the gain gradient. Together with conditional absolute
integrability of \(L_{t+1}\) at one point of \(\mathcal K\), these properties
verify Assumption~\ref{ass::global_condition}; their local counterparts verify
Assumption~\ref{ass::local_condition}.

Smooth exponential-family densities satisfy these requirements on compact
reachable sets when the corresponding conditional moments and base loss are
integrable. In Example~\ref{ex::example},
\(\mathsf E_t[y_{t+1}^2]<\infty\) controls both the affine gain gradient and the
quadratic Gaussian loss. The conditions may fail at a non-smooth boundary of
\(\Lambda\) or without a suitable conditional moment bound. We therefore
localise the analysis to compact regions on which the score is smooth and
dominated.

Conditional expectation turns the pathwise identity into a gradient of the
conditional objective. For statements evaluated at a current-information random
gain, fix a compact convex localisation set
$\mathcal A_\varepsilon\subset\inter(\mathcal A)$ bounded away from
$\partial\mathcal A$, as in Remark~\ref{rem::geometry_constants}.
Assumption~\ref{ass::global_condition} then supplies the conditional envelope and
jointly measurable version needed to evaluate the identity at any
$\mathcal F_t$-measurable $\alpha_t\in\mathcal A_\varepsilon$.
\begin{proposition}[Conditional gain-gradient identity]\label{prop::proposition_1}
    Fix $\alpha \in \inter(\mathcal{A})$ and suppose Assumption~\ref{ass::local_condition} holds at $\alpha$. Then
    \begin{equation}\label{eq::nabla_gradient}
        \nabla_{\alpha} J_{t+1\mid t}(\alpha)=\mathsf{E}_t[G_{t+1}(\alpha)]=-\mathsf{E}_t[H_t^{\top}s_{t+1}(\alpha)].
    \end{equation}
Under Assumption~\ref{ass::global_condition}, there exists a version of $\nabla_{\alpha} J_{t+1\mid t}$ that is continuous in $\alpha$ and for which the identity \eqref{eq::nabla_gradient} holds simultaneously on every compact subset of $\inter(\mathcal{A})$. In particular, let $\mathcal{A}_{\varepsilon}$ be the localisation set fixed above. If $\alpha_t$ is $\mathcal{F}_{t}$-measurable, $\alpha_t\in\mathcal{A}_{\varepsilon}$ almost surely, and $\lambda_{t+1}=\hat\lambda_t(\alpha_t)$, then
    \begin{equation}\label{eq::nabla_gradient_2}
        \nabla_{\alpha} J_{t+1\mid t}(\alpha_t)=\mathsf{E}_t[G_{t+1}(\alpha_t)]=-\mathsf{E}_t[H_t^{\top}s(y_{t+1},\lambda_{t+1})]\,.
    \end{equation}
\end{proposition}
\proofref{app::proof_proposition_1}

Equations~\eqref{eq::nabla_gradient}--\eqref{eq::nabla_gradient_2} show that the
gain-gradient signal \(G_{t+1}\), available when \(y_{t+1}\) arrives, is an
unbiased stochastic gradient of \(J_{t+1\mid t}\); its negative gives the descent
direction. This is the relevant objective because \(\alpha_t\) is formed after
observing \(y_t\). If the gain is committed before \(y_t\), the corresponding
pre-score objective is
\[
    \bar J_{t\mid t-1}(\alpha)
    =
    \mathsf E_{t-1}
    \big[
        J_{t+1\mid t}(\alpha)
    \big]
    =
    \mathsf E_{t-1}
    \big[
        L_{t+1}(\alpha)
    \big].
\]
Under analogous dominated-differentiation conditions conditional on
\(\mathcal F_{t-1}\), the same score signal is an unbiased stochastic gradient of
this objective.
\begin{corollary}[Pre-score (predictable-timing) gain-gradient identity]
\label{cor::corollary_1}
Fix \(\alpha\in\inter(\mathcal A)\) and suppose
Assumption~\ref{ass::local_condition} holds at \(\alpha\), with
\(\mathsf E_{t-1}[\cdot]\) in place of \(\mathsf E_t[\cdot]\). Then
\begin{equation}\label{eq::decision_time_gradient}
    \nabla_\alpha \bar J_{t\mid t-1}(\alpha)
    =
    \mathsf E_{t-1}
    \big[
        G_{t+1}(\alpha)
    \big]
    =
    -
    \mathsf E_{t-1}
    \big[
        H_t^{\top}s_{t+1}(\alpha)
    \big].
\end{equation}
Under Assumption~\ref{ass::global_condition} with the same replacement, there
exists a version of \(\nabla_\alpha\bar J_{t\mid t-1}\) that is continuous in
\(\alpha\) and for which the identity holds simultaneously on every compact
subset of \(\inter(\mathcal A)\). In particular, if \(\alpha_t\) is
\(\mathcal F_{t-1}\)-measurable, \(\alpha_t\in\mathcal A_\varepsilon\) almost
surely, and \(\lambda_{t+1}=\hat\lambda_t(\alpha_t)\), then
\begin{equation}\label{eq::decision_time_gradient_policy}
    \nabla_\alpha \bar J_{t\mid t-1}(\alpha_t)
    =
    \mathsf E_{t-1}
    \big[
        G_{t+1}(\alpha_t)
    \big]
    =
    -
    \mathsf E_{t-1}
    \big[
        H_t^{\top}s(y_{t+1},\lambda_{t+1})
    \big].
\end{equation}
\end{corollary}
\proofref{app::proof_predictable_gradient}
The Gaussian location model makes both timing conventions transparent. After
\(y_t\) is observed, the predictive objective is quadratic and its constrained
oracle places the reachable state as close as possible to the next conditional
mean. If the gain is committed before \(y_t\), the unconstrained oracle is a
projection coefficient. Under a Gaussian local-level data-generating process,
the post-observation mean oracle coincides with the Kalman gain whenever the
admissibility constraint does not bind.

\begin{example}[Gaussian location model]\label{ex::example}
Let the working family \(P_\lambda\) of Section~\ref{sec::SectionOne} be Gaussian
with fixed variance \(\sigma^2>0\):
\begin{equation}
    p(y\mid\lambda)
    =
    \frac{1}{\sqrt{2\pi\sigma^2}}
    \exp\left\{-\frac{(y-\lambda)^2}{2\sigma^2}\right\},
    \qquad \lambda\in\mathbb R.
\end{equation}
The conditional law of \(y_{t+1}\) given \(\mathcal F_t\) may nevertheless be
non-Gaussian; we require only that
\(\widetilde m_{t+1}:=\mathsf E_t[y_{t+1}]\) exists and
\(\mathsf E_t[y_{t+1}^2]<\infty\). We set \(c=0\), \(B=1\), and
\(S(\lambda)=\sigma^2\), so \(b_t=\lambda_t\). The Gaussian score is
\(s(y,\lambda)=(y-\lambda)/\sigma^2\), and the scaled score direction
\(d_t=S(\lambda_t)s(y_t,\lambda_t)=y_t-\lambda_t\) is the current forecast
error. For a scalar gain \(\alpha\in\mathcal A\subset\mathbb R\),
Eq.~\eqref{eq::ScoreDrivenCrealandHarvey2} gives the reachable post-observation
state \(\hat\lambda_t(\alpha)=\lambda_t+\alpha d_t\). As
Eq.~\eqref{eq::attainableset} shows, varying \(\alpha\) traces a line in state
space, and the gain places the updated state along that line. The
post-observation predictive objective for the gain in
Eq.~\eqref{eq::scoringruleexpected} is
\begin{equation*}
    J_{t+1\mid t}(\alpha)=C^{G}_{t+1} + \frac{(\lambda_t + \alpha d_t - \widetilde m_{t+1})^2}{2 \sigma^2},
\end{equation*}
where \(C^G_{t+1}\) is independent of \(\alpha\) and collects the Gaussian
normalising constant and the conditional variance of \(y_{t+1}\). Under correct
specification, it equals the conditional entropy \(C_{t+1}\) in
Proposition~\ref{prop::kl_representation}; under misspecification, it also
contains a gain-independent misspecification term. The map
\(\alpha\mapsto J_{t+1\mid t}(\alpha)\) is convex on \(\mathcal A\), with
curvature \(d_t^2/\sigma^2\). When \(d_t\neq0\), it is strictly convex and its
unconstrained gain oracle is
\begin{equation*}
    \alpha_t^{\mathrm{unc}}=\frac{\widetilde m_{t+1}-\lambda_t}{d_t}.
\end{equation*}
Projecting this value onto \(\mathcal A\) gives the constrained oracle,
\begin{equation*}
    \alpha_t^{\star}=\Pi_{\mathcal{A}}\left(\frac{\widetilde m_{t+1}-\lambda_t}{d_t}\right).
\end{equation*}
When \(d_t=0\), the reachable set is the singleton \(\{\lambda_t\}\), the
objective is flat, and every gain in \(\mathcal A\) is observationally equivalent
at date \(t\). At the realised gain \(\alpha_t\), let
\(\lambda_{t+1}=\lambda_t+\alpha_t d_t\) and
\(d_{t+1}:=y_{t+1}-\lambda_{t+1}\). Then
\begin{equation}\label{eq::equation_2}
    G_{t+1}(\alpha_t)=\frac{\partial L_{t+1}(\alpha_t)}{\partial \alpha}=-\frac{d_t d_{t+1}}{\sigma^2} =-d_t s(y_{t+1},\lambda_{t+1}).
\end{equation}
Hence \(d_td_{t+1}/\sigma^2=-G_{t+1}(\alpha_t)\) is the descent direction for the
realised predictive log-loss. The product-of-scores feedback used in
\textrm{aGAS}-type gain recursions therefore has a direct optimisation
interpretation: up to positive scaling, it is both a descent signal for the
gain-induced predictive loss and a local serial-dependence
diagnostic~\cite{blasques2019accelerating}.

\noindent The post-observation oracle uses the realised score and follows the
\(\mathcal F_t\) timing in Definition~\ref{def::gain_policy}. Under the pre-score
timing, the gain is committed before \(y_t\) is observed. Let
\(\delta_t:=\widetilde m_{t+1}-\lambda_t\). Then
\begin{equation}\label{eq::scoringruleexpected2}
    \bar J_{t\mid t-1}(\alpha)
    :=
    \mathsf E_{t-1}[J_{t+1\mid t}(\alpha)]
    =
    \bar C_t^G
    +
    \frac{1}{2\sigma^2}
    \mathsf E_{t-1}\!\left[(\delta_t-\alpha d_t)^2\right],
\end{equation}
where \(\bar C_t^G\) is independent of \(\alpha\). If
\(\mathsf E_{t-1}[d_t^2]>0\), the unconstrained pre-score oracle is
\begin{equation*}
    \bar\alpha_t^{\text{unc}}=\frac{\mathsf{E}_{t-1}[\delta_t d_t]}{\mathsf{E}_{t-1}[d_t^2]}=\frac{\mathsf{E}_{t-1}[(\widetilde m_{t+1}-\lambda_t)  d_t]}{\mathsf{E}_{t-1}[d_t^2]}.
\end{equation*}
The constrained oracle is
\(\bar\alpha_t^\star=\Pi_{\mathcal A}(\bar\alpha_t^{\mathrm{unc}})\). If
\(\mathsf E_{t-1}[d_t^2]=0\), the pre-score objective is flat on \(\mathcal A\).
Thus, before \(y_t\) is observed, the optimal unconstrained gain is the
linear-projection coefficient of \(\delta_t\) on \(d_t\); admissibility projects
that coefficient onto \(\mathcal A\).
\end{example}

\begin{remark}[Kalman-gain interpretation]\label{rem::kalman}
Under the Gaussian local-level data-generating process, the unconstrained
post-observation mean oracle in Example~\ref{ex::example} is the Kalman gain,
while the constrained oracle projects that gain onto \(\mathcal A\)
(see~\cite{durbin2012time} for the state-space and Kalman-filtering background).
Let
\[
    x_t=x_{t-1}+\zeta_t,
    \qquad
    y_t=x_t+\varepsilon_t,
\]
where \(\zeta_t\) and \(\varepsilon_t\) are independent Gaussian noises. Set
\(\lambda_t:=\mathsf E_{t-1}[x_t]\) and \(d_t:=y_t-\lambda_t\). The Kalman
update gives
\[
    \mathsf E_t[y_{t+1}]-\lambda_t=K_td_t,
\]
where \(K_t\) is the Kalman gain. Hence \(\alpha_t^{\mathrm{unc}}=K_t\) and
\(\alpha_t^\star=\Pi_{\mathcal A}(K_t)\); when \(K_t\in\mathcal A\), the
admissibility constraint does not bind. For this conditional-mean problem,
adaptive gain learning therefore amounts to learning the Kalman gain online,
using the product of consecutive forecast errors as the descent signal.
\end{remark}
\section{Mirror geometry of adaptive gains}
\label{sec::SectionThree}
The previous section identified the score-product gradient. We now use that
signal to construct an adaptive gain recursion. The gain link determines the
mirror geometry, while persistence and discounting govern how past gradient
signals are retained. The latent coordinate represents the gain on the link
scale.

The admissible set $\mathcal A$ may impose positivity, boundedness, or
coordinatewise restrictions. Mirror descent (MD), and its online form OMD,
handles these restrictions by minimising over $\mathcal A$ and measuring movement
with a Bregman divergence. We follow Nemirovsky and
Yudin~\cite{nemirovsky1983problem}; see~\cite{beck2003mirror} for the nonlinear
projected-subgradient formulation. Let
$\mathcal{X}_{\mathcal{A}} \subset \mathbb{R}^{k}$ be an open convex set containing
$\mathcal{A}$. A continuously differentiable, strictly convex map
$\Psi:\mathcal{X}_{\mathcal{A}}\rightarrow\mathbb{R}$ is a distance-generating
function. It induces the \emph{Bregman divergence}
\begin{equation}\label{eq::Bregman}
    \mathcal{B}_{\Psi}(\alpha,\beta)=\Psi(\alpha)-\Psi(\beta)-\langle \nabla \Psi(\beta), \alpha-\beta \rangle,
\end{equation}
where $\langle\cdot,\cdot\rangle$ denotes the standard scalar product. The corresponding MD step, for a current gain $\beta \in \mathcal{A}$, input vector $\xi \in \mathbb{R}^{k}$, and step size $\eta>0$, is
\begin{equation}\label{eq::stepMD}
    \mathcal{M}_{\eta,\Psi}(\beta,\xi)=\argmin_{\alpha \in \mathcal{A}} \left\{\eta\langle \xi, \alpha\rangle + \mathcal{B}_{\Psi}(\alpha,\beta)\right\}.
\end{equation}
The first term is the linearised gain loss; the Bregman term keeps $\alpha$
close to $\beta$. The choice $\Psi(\alpha)=\|\alpha\|^2/2$ gives projected
gradient descent, while other $\Psi$ generate non-Euclidean movement. If the
minimiser $\alpha^{\star}$ is interior, its first-order condition is\footnote{In general the minimiser may lie on the boundary of the set $\mathcal{A}$; denoting by $\mathcal{N}_{\mathcal{A}}(\alpha^{\star})$ the normal cone to $\mathcal{A}$ at $\alpha^{\star}$, in this case the correct optimality condition is $0 \in \eta \xi +  \nabla \Psi(\alpha^{\star})-\nabla \Psi(\beta)+\mathcal{N}_{\mathcal{A}}(\alpha^{\star})$.}
\begin{equation}
    \nabla \Psi(\alpha^{\star})=\nabla \Psi(\beta)-\eta \xi.
\end{equation}
Since $\Psi$ is strictly convex and $\mathcal{A}$ is compact and convex, the minimiser exists and is unique; we denote it $\mathcal{M}_{\eta,\Psi}(\beta,\xi)$.
The gain $\alpha_t$ generates \(\lambda_{t+1}=\hat\lambda_t(\alpha_t)\) and
incurs loss \(L_{t+1}(\alpha_t)\). When \(y_{t+1}\) arrives, the learner
observes
\(\xi_t:=G_{t+1}(\alpha_t)=\nabla_\alpha L_{t+1}(\alpha_t)\) and updates
\(\alpha_{t+1}=\mathcal M_{\eta_t,\Psi}(\alpha_t,\xi_t)\). If \(L_{t+1}\) is
convex and differentiable, \(\xi_t\) is its unique subgradient at \(\alpha_t\),
so the Section~\ref{sec::SectionFive} convention
\(\xi_t\in\partial\ell_t(\alpha_t)\), with \(\ell_t:=L_{t+1}\), is satisfied.
This is the observe-then-update timing used throughout
Section~\ref{sec::SectionFive}.

We next connect the gain link to mirror geometry. Adaptive gains such as those
in \textrm{aGAS}~\cite{blasques2019accelerating} are often restricted to be
positive or bounded, so we write \(\alpha=g(\vartheta)\), where
\(\vartheta\in\mathbb R^{k}\) is unconstrained and \(g\) maps into the interior
of the admissible set. The inverse link will serve as the dual mirror
coordinate. Our implementation uses componentwise logistic, cloglog, and
reverse-cloglog links.

Projected MD and a link recursion enforce admissibility differently. Projected
MD minimises over the closed set \(\mathcal A\) and may produce a boundary
iterate, whereas every finite link coordinate maps into
\(\inter(\mathcal A)\). Let
\(\mathcal A=\prod_{j=1}^{k}[\alpha_{L,j},\alpha_{H,j}]\), with
\(\alpha_{L,j}<\alpha_{H,j}\), and let
\(g:\mathbb R^{k}\rightarrow\inter(\mathcal A)\) be componentwise, strictly
increasing, and continuously differentiable (twice continuously differentiable
when required). Then \(\alpha=g(\vartheta)\) and
\(\vartheta=g^{-1}(\alpha)\). The predictive criterion of
Section~\ref{sec::SectionOne} remains fixed; the link determines the update
geometry, including mobility near the boundary.

For the regret analysis, finite geometric constants are obtained on a compact
localisation \(\mathcal A_\varepsilon\subset\inter(\mathcal A)\). Applying these
constants to the link recursion requires either implementing the
Bregman-proximal projection used in the regret theorem, using a truncation that
coincides with this projection, or verifying that its iterates and comparators
remain in \(\mathcal A_\varepsilon\). For componentwise links and box
constraints, coordinatewise clipping of the unconstrained link update has this
property. We now state the link conditions precisely.
\begin{assumption}[Componentwise link]\label{ass::link}
Each coordinate link $g_j:\mathbb{R}\to(\alpha_{L,j},\alpha_{H,j})$ is a
strictly increasing $C^1$ diffeomorphism onto its admissible interval and
satisfies $g_j'(\vartheta)>0$ for every $\vartheta\in\mathbb{R}$. Write
$g=(g_1,\ldots,g_k)$ and $\alpha=g(\vartheta)$. Whenever a second-order Taylor
expansion is used, assume additionally that each $g_j\in C^2$.
\end{assumption}
\noindent Assumption~\ref{ass::link} strengthens strict monotonicity by requiring
a nonvanishing derivative. Each inverse
$g_j^{-1}:(\alpha_{L,j},\alpha_{H,j})\rightarrow\mathbb R$ is therefore $C^1$,
with $(g_j^{-1})'=1/g_j'(g_j^{-1}(\cdot))$ finite and positive. Hence $g$ is
the dual-to-primal mirror map, $\nabla\Psi_g=g^{-1}$ (equivalently,
$g=\nabla\Psi_g^*$ under the usual Legendre assumptions). On any compact
localisation $\mathcal A_\varepsilon$, continuity and positivity of $g_j'$
provide finite upper and positive lower bounds. These local bounds yield the
strong-convexity and Bregman-diameter constants used in
Remark~\ref{rem::geometry_constants}. Fix any
$\alpha^\circ\in\inter(\mathcal A)$ and define
\begin{equation}\label{eq::phi_g}
    \Psi_{g}(\alpha) = \sum_{j=1}^{k} \int_{\alpha_j^{\circ}}^{\alpha_j} g_{j}^{-1}(u)\,\ud u,\quad \alpha \in \inter(\mathcal{A}).
\end{equation}
Changing $\alpha^{\circ}$ adds only a constant to $\Psi_g$. By the fundamental
theorem of calculus,
\begin{equation*}
    \frac{\partial \Psi_g(\alpha)}{\partial \alpha_j}
    =g_j^{-1}(\alpha_j),\qquad
    \nabla \Psi_g(\alpha)=g^{-1}(\alpha).
\end{equation*}
Differentiating once more gives
\begin{equation}\label{eq::Hessian}
    \nabla^2 \Psi_g(\alpha) = \diag\left(\frac{1}{g_j^{'}(g_{j}^{-1}(\alpha_j))}\right)_{j=1}^{k}.
\end{equation}
Because $g_j'>0$, the Hessian in Eq.~\eqref{eq::Hessian} is positive definite
throughout $\inter(\mathcal A)$, so $\Psi_g$ is strictly convex there.
Substituting the potential in Eq.~\eqref{eq::phi_g} into
Eq.~\eqref{eq::Bregman} gives the link-induced Bregman divergence
\begin{equation*}
    \mathcal{B}_{\Psi_{g}}(\alpha,\beta)=\Psi_{g}(\alpha)-\Psi_{g}(\beta)-\langle g^{-1}(\beta),\alpha-\beta\rangle,
\end{equation*}
for $\alpha,\beta\in\inter(\mathcal A)$. The divergence measures gain-scale
displacement in the geometry induced by the link. Since the Hessian in
Eq.~\eqref{eq::Hessian} is diagonal, the curvature in coordinate $j$ is
\begin{equation*}
    \frac{1}{g_j'(g_j^{-1}(\alpha_j))}.
\end{equation*}
A larger link derivative lowers curvature and increases local mobility; a
smaller derivative raises curvature and reduces mobility. The following
recursion makes this relation explicit. For $\eta_t>0$, consider
\begin{equation*}
   \vartheta_{t+1} = \vartheta_{t} - \eta_t \xi_t,\quad \alpha_t = g(\vartheta_t).
\end{equation*}
Since $\vartheta_t=g^{-1}(\alpha_t)=\nabla\Psi_g(\alpha_t)$, the recursion is
equivalent to
\begin{equation*}
   \nabla \Psi_g(\alpha_{t+1})=\nabla \Psi_g(\alpha_{t})-\eta_t \xi_t,
\end{equation*}
and therefore
\begin{equation*}
   \alpha_{t+1}=g(g^{-1}(\alpha_t)-\eta_t \xi_t).
\end{equation*}
Strict convexity of $\Psi_g$ shows that $\alpha_{t+1}$ is the unique minimiser
of the mirror step:
\begin{equation}\label{eq::update}
    \alpha_{t+1} = \argmin_{\alpha \in \inter(\mathcal{A})}\{\eta_t \langle \xi_t, \alpha\rangle + \mathcal{B}_{\Psi_g}(\alpha,\alpha_t)\}.
\end{equation}
Equation~\eqref{eq::update} shows that the link supplies the geometry through
which the predictive criterion updates the gain. A Euclidean gradient step in
the latent coordinate would use the chain-rule gradient
\begin{equation*}
    \nabla_{\vartheta} L(g(\vartheta))
    =\diag\!\bigl(g_1'(\vartheta_1),\ldots,g_k'(\vartheta_k)\bigr)
    \nabla_{\alpha}L(\alpha).
\end{equation*}
The mirror recursion applies $\xi_t=\nabla_{\alpha}L(\alpha_t)$ directly in the
dual coordinate $\vartheta=g^{-1}(\alpha)$. It is therefore mirror descent on
the gain loss expressed in link coordinates.

A Taylor expansion makes local mobility explicit. For any coordinate $j$ with
$g_j\in C^2$, Taylor's theorem gives
\begin{equation*}
    \alpha_{j,t+1}-\alpha_{j,t}
    =-\eta_t g_j'(\vartheta_{j,t})\xi_{j,t}
    +\frac{\eta_t^2}{2}g_j''(\widetilde\vartheta_{j,t})\xi_{j,t}^2,
\end{equation*}
for some $\widetilde\vartheta_{j,t}$ between $\vartheta_{j,t}$ and
$\vartheta_{j,t}-\eta_t\xi_{j,t}$. For any compact gain localisation
$\mathcal K\subset\inter(\mathcal A)$, its inverse image $g^{-1}(\mathcal K)$
is compact in latent-coordinate space. Continuity of $g_j''$ bounds it on a
sufficiently small compact latent-coordinate neighbourhood. Holding
$\xi_{j,t}$ fixed and taking $\eta_t$ sufficiently small keeps the Taylor
intermediate point in this neighbourhood and therefore gives
\begin{equation}\label{eq::mobility}
    \alpha_{j,t+1}-\alpha_{j,t}
    =-\eta_t g_j'(g_j^{-1}(\alpha_{j,t}))\xi_{j,t}+O(\eta_t^2).
\end{equation}
Thus,
\begin{equation*}
    m_j(\alpha_{j,t}):=g_j'(g_j^{-1}(\alpha_{j,t}))
\end{equation*}
is the first-order mobility of the gain.

The first-order expression explains why links sharing the same admissible
interval can generate different gain paths. They act on the same linearised loss
and transmit the common dual signal through link-specific local mobilities. Near
an endpoint, lower mobility produces more gradual movement, while higher
mobility transmits the signal more strongly.
\begin{remark}[Interior versus projected mirror steps]\label{remarkdue}
The link recursion remains in $\inter(\mathcal A)$ for every $t$ and therefore
preserves the gain constraints. It is an interior mirror recursion.
Section~\ref{sec::SectionFive} states the regret results on a closed decision
set; their application to the link recursion uses the compact interior
localisation $\mathcal A_\varepsilon$ described above.
\end{remark}

Empirical adaptive-gain models such as \textrm{aGAS} include persistence and an
intercept. We first state the generic online transition; Remark~\ref{rem::agas}
gives its exact calendar-time representation for \textrm{aGAS}.
\begin{equation}\label{eq::recursionpersistence}
    \vartheta_{t+1}=(1-\rho_t) \bar{\vartheta} + \rho_t \vartheta_t - \eta_t \xi_t,\quad 0 \leq \rho_t \leq 1.
\end{equation}
Here $\bar\vartheta\in\mathbb R^k$ is a reference dual coordinate, with
associated reference gain $\bar\alpha=g(\bar\vartheta)$. Subtracting
$\bar\vartheta$ from Eq.~\eqref{eq::recursionpersistence} gives
\begin{equation}\label{eq::recursionpersistence2}
    \vartheta_{t+1}-\bar{\vartheta}=\rho_t (\vartheta_t - \bar{\vartheta}) - \eta_t \xi_t.
\end{equation}
Setting $\xi_t=0$ in the centred recursion shows that $\rho_t$ controls memory:
values near one retain deviations from $\bar\vartheta$, whereas values near zero
pull rapidly towards it. For constant $0\leq\rho<1$, iteration gives
\begin{equation}\label{eq::recursionpersistence3}
    \vartheta_{t+1}-\bar{\vartheta}=\rho^{t} (\vartheta_1 - \bar{\vartheta}) - \sum_{i=1}^{t} \rho^{t-i} \eta_i \xi_i.
\end{equation}
Past gradients therefore receive geometrically declining weights
$\rho^{t-i}$; $\rho$ controls persistence and $1-\rho$ one-step discounting. The
equivalent intercept form is
\begin{equation}\label{eq::recursionpersistence4}
    \vartheta_{t+1} = \omega + \rho \vartheta_t - \eta_t \xi_t,
    \qquad 0\leq\rho<1.
\end{equation}
Since $\bar\vartheta=\omega/(1-\rho)$ and
$\bar\alpha=g(\omega/(1-\rho))$, $\omega$ is a dual-coordinate intercept. Under
bounded links, the gain magnitude depends jointly on $\omega$, $\rho$, and $g$.

\vspace{0.3cm}

The following proposition gives the exact MD interpretation of the discounted recursion.
\begin{proposition}[Discounted link recursion as mirror descent]\label{prop::discounted_md}
Let \(\mathcal I=\prod_{j=1}^k I_j\) be a product of open, possibly unbounded
intervals, let \(g:\mathbb R^k\to\mathcal I\) be a componentwise \(C^1\)
diffeomorphism with positive derivative, and let \(\Psi_g\) be defined by
Eq.~\eqref{eq::phi_g} on \(\mathcal I\). Fix \(\bar\vartheta\in\mathbb R^k\),
set \(\bar\alpha=g(\bar\vartheta)\), and let \(\eta_t>0\),
\(0\le\rho_t\le1\), and \(\xi_t\in\mathbb R^k\). Suppose
\(\alpha_t\in\mathcal I\), \(\vartheta_t=g^{-1}(\alpha_t)\), and define
\begin{equation*}
    \vartheta_{t+1}=(1-\rho_t)\bar{\vartheta} + \rho_t\vartheta_t-\eta_t \xi_t,\quad\quad \alpha_{t+1}=g(\vartheta_{t+1}).
\end{equation*}
Then \(\alpha_{t+1}\) is the unique minimiser over \(\mathcal I\) of
\begin{equation}\label{eq::objective_gas}
    \alpha\mapsto \eta_t \langle \xi_t, \alpha \rangle + (1-\rho_t) \mathcal{B}_{\Psi_{g}}(\alpha,\bar{\alpha})+\rho_t \mathcal{B}_{\Psi_{g}}(\alpha,\alpha_t).
\end{equation}
\end{proposition}
\proofref{app::proof_discounted_md}
Equation~\eqref{eq::objective_gas} separates the linearised gain loss, the pull
towards the reference gain $\bar\alpha$, and persistence around the previous
gain $\alpha_t$.

Persistence also interacts with the learning rate. The centred form of
Eq.~\eqref{eq::recursionpersistence} is
\begin{equation*}
    \vartheta_{t+1}
    =\vartheta_t-\eta_t\{\xi_t+\kappa_t(\vartheta_t-\bar\vartheta)\},
    \qquad
    \kappa_t=\frac{1-\rho_t}{\eta_t}.
\end{equation*}
Thus $\kappa_t$ is the effective shrinkage strength in the dual coordinate.
Coupling the schedules through $\rho_t=1-\kappa\eta_t$, with
$0\leq\kappa\eta_t\leq1$, keeps this strength equal to $\kappa$. If
$\rho_t\equiv\rho$ while $\eta_t$ varies, then $\kappa_t$ varies as well.
Persistence governs memory, while its regularising effect depends on the
learning-rate schedule.

\begin{remark}[Recovery of \textrm{aGAS}]\label{rem::agas}
The calendar-time convention of~\cite{blasques2019accelerating} is
\[
 f_{t+1}=\omega_f+\beta_f f_t
 +\alpha_f C_{f,t}u_{\lambda,t}u_{\lambda,t-1},
 \qquad
 \alpha_{\lambda,t}=g(f_{t+1}).
\]
Thus \(f_t\) and \(\lambda_t\) are \(\mathcal F_{t-1}\)-measurable;
\(y_t\) is observed; and \(\alpha_{\lambda,t}=g(f_{t+1})\) is formed and
used immediately in the update from \(\lambda_t\) to \(\lambda_{t+1}\).
For \(0\leq\beta_f<1\), translate to the online transition by setting
\(\vartheta_t:=f_{t+1}\), \(\alpha_t:=\alpha_{\lambda,t}\),
\(\rho:=\beta_f\), and \(\bar\vartheta:=\omega_f/(1-\beta_f)\). The endpoint
\(\beta_f=1\) is covered when \(\omega_f=0\); if \(\omega_f\neq0\), the
recursion contains an additional dual-coordinate drift outside the two-centre
update in Proposition~\ref{prop::discounted_md}. In the scalar unscaled case,
\[
 \xi_t=G_{t+1}(\alpha_t)=-u_{\lambda,t}u_{\lambda,t+1},
 \qquad
 \eta_t=\alpha_f C_{f,t+1},
\]
so \(\vartheta_t\mapsto\vartheta_{t+1}\) has exactly the form of
Eq.~\eqref{eq::recursionpersistence}. Equivalently, the calendar-time transition
\(\alpha_{t-1}\mapsto\alpha_t\) uses
\(\xi_{t-1}=G_t(\alpha_{t-1})=-u_{\lambda,t-1}u_{\lambda,t}\) and
\(\eta_{t-1}=\alpha_fC_{f,t}\). These are the same indexing convention, shifted
once; in particular, \(C_{f,t+1}\), not \(C_{f,t}\), belongs to the transition
\(\alpha_t\mapsto\alpha_{t+1}\).

For a scalar scaled direction \(d_t=S_tu_{\lambda,t}\), the gain gradient is
\(-S_tu_{\lambda,t}u_{\lambda,t+1}\). When \(S_t>0\), the raw aGAS score
product is recovered by setting
\[
 \xi_t=-S_tu_{\lambda,t}u_{\lambda,t+1},
 \qquad
 \eta_t=\frac{\alpha_fC_{f,t+1}}{S_t},
\]
since \(-\eta_t\xi_t
=\alpha_fC_{f,t+1}u_{\lambda,t}u_{\lambda,t+1}\). Proposition
\ref{prop::discounted_md} applies to bounded links and also to the exponential
link on \((0,\infty)\), whose potential for \(g(f)=\exp(f)\) is
\(\Psi_g(\alpha)=\alpha\log\alpha-\alpha\), up to affine terms.

The exact recursive representation establishes the algebraic correspondence,
including for signed values of $\alpha_f$. Its descent interpretation
additionally requires $\eta_t>0$: when the other scaling factors are positive,
$\alpha_f=0$ removes score-product learning and $\alpha_f<0$ reverses its
direction. Regret coverage under Theorem~\ref{thm:dynamic_regret_dmd} requires
\(0\le\beta_f<1\), or \(\beta_f=1\) with \(\omega_f=0\), together with
\(\eta_t>0\), projection or verified compact localisation,
finite convex losses, finite gradient energy---with integrability for the
expected bound---and a non-decreasing composite schedule. In the unscaled
representation, this schedule is
\[
 w_t=\frac{\beta_f}{\alpha_fC_{f,t+1}};
\]
with scalar scaling \(S_t>0\), it is
\[
 w_t=\frac{\beta_fS_t}{\alpha_fC_{f,t+1}}.
\]
These conditions remain to be verified for the original data-dependent aGAS
scaling. Untruncated exponential-link aGAS remains outside the compact theorem.
\end{remark}

\section{Benchmarks and dynamic regret for gain learning}\label{sec::SectionFive}
The benchmark follows the information and reachability structure of the gain
problem. At date $t$, the $\mathcal F_t$-measurable gain reacts to the realised
score, is committed before $y_{t+1}$, and moves the current state within the
realised reachable set $\mathcal R_t$. A different gain history would generate a
different sequence of states, scores, and reachable sets. We therefore study
local gain regret on the realised score-conditioned path; policy regret for
counterfactual filter paths is a separate object.

The regret theorem is formulated for convex losses in gain space. In
score-driven filters, this condition can be checked in state space on the
realised reachable set $\mathcal R_t$. Define the realised state loss
\begin{equation}\label{eq::realizedstate}
    \ell_{t+1}^{\text{st}}(\lambda)=-\log p(y_{t+1}\mid \lambda),
\end{equation}
and its conditional risk after observing $y_t$,
\begin{equation}\label{eq::realizedstaterisk}
    R_{t+1}(\lambda)=\mathsf{E}_{t}[\ell_{t+1}^{\text{st}}(\lambda)].
\end{equation}
In the general linear formulation,
$L_{t+1}(\alpha)=\ell_{t+1}^{\text{st}}(b_t+H_t\alpha)$ and
$J_{t+1\mid t}(\alpha)=R_{t+1}(b_t+H_t\alpha)$. Since
$\alpha\mapsto H_t\alpha$ is affine, convexity of either state criterion on
$\mathcal R_t$ carries over to the corresponding gain criterion on $\mathcal A$.
\begin{proposition}[Convexity transfer to gain space]\label{prop::convexity_transfer}
    Let $t$ be fixed and let $\mathcal R_t$ be as in Eq.~\eqref{eq::attainableset}. Then the following statements hold:
    \begin{enumerate}
        \item If $\lambda \mapsto \ell_{t+1}^{\text{st}}(\lambda)$ is convex on $\mathcal R_t$, then $\alpha\mapsto L_{t+1}(\alpha)$ is convex on $\mathcal{A}$.
        \item If $\lambda \mapsto R_{t+1}(\lambda)$ is convex on $\mathcal R_t$, then $\alpha\mapsto J_{t+1\mid t}(\alpha)$ is convex on $\mathcal{A}$.
        \item If the relevant state criterion is strictly convex on $\mathcal R_t$, and $H_t(\alpha-\beta)\neq0$ for every distinct pair $\alpha,\beta\in\mathcal A$, then the corresponding gain criterion is strictly convex on $\mathcal A$.
        \item If $\lambda \mapsto \ell_{t+1}^{\text{st}}(\lambda)$ is $\mu$-strongly convex on $\mathcal R_t$ in the Euclidean norm, then, for every $\alpha, \beta \in \mathcal{A}$ and $\tau \in [0,1]$,
        \begin{equation}\label{eq::inequality}
            L_{t+1}(\tau \alpha + (1-\tau) \beta) \leq \tau L_{t+1}(\alpha)+(1-\tau)L_{t+1}(\beta)-\frac{\mu}{2}\tau(1-\tau)\|H_t(\alpha-\beta)\|^2.
        \end{equation}
        If, in addition, $\sigma_{\min}(H_t)^2\ge c_t>0$, then $L_{t+1}$ is $\mu c_t$-strongly convex on $\mathcal A$. The same conclusions hold for $J_{t+1\mid t}$ when $R_{t+1}$ is $\mu$-strongly convex on $\mathcal R_t$.
    \end{enumerate}
\end{proposition}
\proofref{app::proof_convexity_transfer}
Econometrically, Proposition~\ref{prop::convexity_transfer} shows that the
current scaled score controls how state-space curvature is transmitted to gain
space. The matrix $H_t$ therefore determines the reachable update directions
and the local curvature available for learning the gain coordinates. For a
diagonal gain, $\sigma_{\min}(H_t)=\min_j|(d_t)_j|$; for a common scalar gain,
$\sigma_{\min}(H_t)=\|d_t\|$. When the state criterion is strongly convex,
positive gain-space curvature in every coordinate additionally requires $H_t$
to have full column rank. Thus every $(d_t)_j$ must be non-zero in the diagonal
case, while $d_t\neq0$ suffices in the scalar case. The designed quadratic
experiment of Subsection~\ref{subsec::two} illustrates how score scaling affects
coordinatewise target feasibility and learnability.

\begin{remark}[When convexity holds]\label{rem::convexity_scope}
Proposition~\ref{prop::convexity_transfer} provides a convenient route for
verifying convexity in gain space: it is sufficient that the state log-loss
$\lambda\mapsto-\log p(y_{t+1}\mid\lambda)$ be convex on the reachable set
$\mathcal R_t$. This occurs when $p(y\mid\lambda)$ is log-concave as a function
of the state. In the Gaussian location model of Example~\ref{ex::example}, the
state log-loss has curvature $1/\sigma^2$; because
$\lambda=b_t+d_t\alpha$, the corresponding scalar gain loss has curvature
$d_t^2/\sigma^2$. More generally, the negative log-likelihood of an exponential
family is convex in its natural parameter. Whenever it is finite, the
conditional risk $R_{t+1}$ inherits the same convexity under conditional
expectation. The empirical volatility models also satisfy this condition when
$h$ is the log scale-squared state: $e^h$ is the Gaussian variance and the
Student-$t$ scale squared. Under the Gaussian density $\mathcal N(0,e^h)$, the
state log-loss is $C+\tfrac12(h+y^2e^{-h})$, with second derivative
$\tfrac12y^2e^{-h}\ge0$. Under the Student-$t$ density with
$\nu_{\mathrm{df}}$ degrees of freedom, distinct from the dominating measure
$\nu$, let $z=y^2e^{-h}/\nu_{\mathrm{df}}$. The state log-loss is
$C_{\nu_{\mathrm{df}}}+\tfrac12h+\tfrac{\nu_{\mathrm{df}}+1}{2}\log(1+z)$,
with second derivative
$\tfrac{\nu_{\mathrm{df}}+1}{2}z/(1+z)^2\ge0$. For
$\nu_{\mathrm{df}}>2$, its predictive variance is
$e^h\nu_{\mathrm{df}}/(\nu_{\mathrm{df}}-2)$. Both state losses are therefore
convex in $h$, becoming affine only when $y=0$. The Student-$t$ score with
respect to $h$ is bounded and saturating, while the Gaussian score is unbounded.
Direct variance parametrisations and some location--scale or shape
parametrisations can produce non-convex state losses. In the log scale-squared
parametrisation used here, the Gaussian and Student-$t$ return-density losses
satisfy the pathwise convexity condition. The pathwise regret theorem
additionally requires compact localisation, finite losses, and the stated
schedule conditions; its expected version also requires integrability of the
random terms in the bound. Both results compare gains on a common realised
context. For a state loss that is non-convex in the chosen parametrisation, the
gain recursion remains well defined. A regret bound then requires either
localisation to a region on which the loss is convex or re-expression of the
filter in a parametrisation, such as a natural parameter or log variance, that
makes the state log-loss convex.
\end{remark}

We use two related benchmarks. The reachable oracle chooses the best gain after
observing $y_t$ and the current score; the realisability gap measures the cost
of restricting the state choice to $\mathcal R_t$. To ensure that the oracle is
well defined and measurable, assume that $J_{t+1\mid t}$ is a Carath\'eodory
objective on $\mathcal A$: for every $\alpha\in\mathcal A$,
$J_{t+1\mid t}(\alpha)$ is $\mathcal F_t$-measurable, while
$\alpha\mapsto J_{t+1\mid t}(\alpha)$ is finite and continuous on
$\mathcal A$ almost surely. Since $\mathcal A$ is compact, the measurable
minimum theorem yields a non-empty argmin and an $\mathcal F_t$-measurable
selector. The corresponding
oracle value and a measurable oracle gain are
\begin{equation}\label{eq::unpredictableGain}
    R^{\star}_{t+1}=\inf_{\lambda \in \mathcal R_t} R_{t+1}(\lambda)=\inf_{\alpha \in \mathcal{A}} J_{t+1\mid t}(\alpha),\qquad \alpha_t^{\star}\in\argmin_{\alpha\in\mathcal A}J_{t+1\mid t}(\alpha),
\end{equation}
where $J_{t+1\mid t}$ is defined in Eq.~\eqref{eq::scoringruleexpected}. The
selector $\alpha_t^\star$ uses the same $\mathcal F_t$ information as the
learned gain $\alpha_t$ and serves as the moving comparator in the regret
theorem.

Allowing the state to range over the full state space $\Lambda$ gives the global
oracle value
\begin{equation}\label{eq::unpredictableGain2}
    R^{\dagger}_{t+1}=\inf_{\lambda \in \Lambda} R_{t+1}(\lambda).
\end{equation}
For this global comparison only, assume that $R^\dagger_{t+1}$ is
$\mathcal F_t$-measurable and finite almost surely. A sufficient condition is
that $\Lambda$ is a non-empty separable metric space,
$R_{t+1}(\omega,\lambda)$ admits a Carath\'eodory version on $\Lambda$, is
bounded below by a finite $\mathcal F_t$-measurable random variable, and is
finite almost surely at one fixed $\lambda_0\in\Lambda$. Under these
conditions, the infimum can be taken over a countable dense subset. We do not
require that the global infimum be attained.
For any gain $\alpha\in\mathcal A$, its excess conditional risk over this
global benchmark decomposes, conditionally on $\mathcal F_t$, as
\begin{equation}\label{eq:benchmark_decomposition_biom}
\begin{aligned}
  J_{t+1\mid t}(\alpha)-R_{t+1}^{\dagger}
  &=\underbrace{J_{t+1\mid t}(\alpha)-R_{t+1}^{\star}}
    _{:=\,\mathcal E_t^{\mathrm{alg}}(\alpha)}\\
  &\quad+\underbrace{R_{t+1}^{\star}-R_{t+1}^{\dagger}}
    _{:=\,\mathcal E_t^{\mathrm{real}}}.
\end{aligned}
\end{equation}
By construction, \(\mathcal E_t^{\mathrm{alg}}(\alpha)\geq 0\) and measures the
excess conditional risk of \(\alpha\) over the best reachable gain. Likewise,
\(\mathcal E_t^{\mathrm{real}}\geq 0\) because
\(R_{t+1}^{\dagger}\leq R_{t+1}^{\star}\); it measures the cost of restricting the
state space from \(\Lambda\) to \(\mathcal R_t\). If the gain were committed before
\(y_t\), comparison with the post-score oracle would also involve an
information-timing, or predictability, gap. The pathwise theorem bounds realised
regret relative to a reachable comparator. Corollary~\ref{cor::corollary} then
controls the expected algorithmic gap against \(\alpha_t^\star\). The
realisability gap is determined by the score-conditioned reachable set.
Let $\mathcal K\subseteq\mathcal A$ be a non-empty closed convex decision set.
Let $\Psi$ be differentiable on an open convex neighbourhood of $\mathcal K$
and $\sigma_\Psi$-strongly convex on $\mathcal K$ with respect to $\|\cdot\|$,
whose dual norm is $\|\cdot\|_*$. Define
\begin{equation}\label{eq::Omega}
    \Omega_{\Psi,\mathcal K}:=
    \sup_{\alpha,\beta\in\mathcal K}\mathcal B_\Psi(\alpha,\beta)<\infty,
\end{equation}
where $\mathcal B_\Psi$ is the Bregman divergence in
Eq.~\eqref{eq::Bregman}. For $u,v\in\mathcal K$, define the one-sided Bregman
path increment
\begin{equation}\label{eq::bregmanincrements}
  \Delta_{\Psi,\mathcal K}(u,v)
  :=\sup_{\alpha\in\mathcal K}
  \{\mathcal B_\Psi(u,\alpha)-\mathcal B_\Psi(v,\alpha)\}.
\end{equation}
The constants $\sigma_\Psi$ and $\Omega_{\Psi,\mathcal K}$ are deterministic
features of the fixed geometry on $\mathcal K$. In applications,
$\mathcal K=\mathcal A$ gives the full-set result, whereas a link-induced
potential uses the fixed interior localisation
$\mathcal K=\mathcal A_\varepsilon$ described in
Remark~\ref{rem::geometry_constants}.
The following result gives a projected mirror-descent bound for a moving
comparator in the Bregman geometry of $\mathcal K$. It adapts the path-variation
argument of Zinkevich's online convex programming analysis~\cite{zinkevich2003online};
Besbes et al.~\cite{besbes2015nonstationary} place dynamic-oracle comparisons in
a broader nonstationary framework.
\begin{theorem}[Dynamic regret of projected mirror descent]\label{thm:dynamic_regret_biom}
Let $\mathcal K\subseteq\mathcal A$ be a non-empty closed convex decision set,
and let $\Psi$ be differentiable on an open convex neighbourhood of $\mathcal K$
and $\sigma_\Psi$-strongly convex on $\mathcal K$ with respect to $\|\cdot\|$,
whose dual norm is $\|\cdot\|_*$. Assume
that $\Omega_{\Psi,\mathcal K}<\infty$ as in Eq.~\eqref{eq::Omega}. Let
$\ell_t:\mathcal K\to\mathbb R$, $t=1,\ldots,N$, be finite convex losses, and
let $(\eta_t)_{t=1}^{N}$ be positive and non-increasing. Starting from
$\alpha_1\in\mathcal K$, suppose that, after incurring $\ell_t(\alpha_t)$, the
algorithm receives a subgradient $\xi_t\in\partial\ell_t(\alpha_t)$ and updates
\begin{equation*}
    \alpha_{t+1}
    =
    \argmin_{\alpha\in\mathcal K}
    \big\{\eta_t\langle\xi_t,\alpha\rangle
    +\mathcal B_\Psi(\alpha,\alpha_t)\big\},
    \qquad t=1,\ldots,N.
\end{equation*}
Then, for every comparator path $\gamma_1,\ldots,\gamma_N\in\mathcal K$,
\begin{align}
\label{eq:dynamic_regret_bound_biom}
  \sum_{t=1}^N\{\ell_t(\alpha_t)-\ell_t(\gamma_t)\}
  &\le
  \frac{\Omega_{\Psi,\mathcal K}}{\eta_{N}}
  +
  \sum_{t=2}^N
  \frac{\Delta_{\Psi,\mathcal K}(\gamma_t,\gamma_{t-1})}{\eta_{t}}
  +
  \frac{1}{2\sigma_\Psi}\sum_{t=1}^N\eta_t\|\xi_t\|_*^2 .
\end{align}
\end{theorem}
\proofref{app::Appendix}
In our application, Theorem~\ref{thm:dynamic_regret_biom} bounds realised gain
regret relative to any comparator path evaluated on the same score-conditioned
path. Counterfactual filter policies generally generate different state and
score sequences and require a separate policy-regret analysis. The three terms
in Eq.~\eqref{eq:dynamic_regret_bound_biom} reflect the Bregman diameter of
$\mathcal K$, comparator movement, and accumulated gradient energy. If
$\|\xi_t\|_*\leq G$, the last term is at most
$\frac{G^2}{2\sigma_\Psi}\sum_{t=1}^N\eta_t$. Applying the theorem to the
conditional predictive loss gives the expected version below.

\begin{corollary}[Expected dynamic regret for gain learning]\label{cor::corollary}
Suppose that $L_{t+1}(\cdot)$ in Eq.~\eqref{eq::scoringrule} is finite and
convex on $\mathcal K$ almost surely for $t=1,\ldots,T-1$. Let the gain
sequence follow the projected MD recursion in
Theorem~\ref{thm:dynamic_regret_biom}, with $N=T-1$ and deterministic step
sizes, and let
$\xi_t\in\partial L_{t+1}(\alpha_t)$. Assume that $\alpha_t$ and
$\gamma_t\in\mathcal K$ are $\mathcal F_t$-measurable, while $\xi_t$ is
$\mathcal F_{t+1}$-measurable. If
$L_{t+1}(\alpha_t)$, $L_{t+1}(\gamma_t)$,
$\Delta_{\Psi,\mathcal K}(\gamma_t,\gamma_{t-1})$, and
$\|\xi_t\|_*^2$ are integrable, then
\begin{align*}
\sum_{t=1}^{T-1}
\mathsf E\!\left[
J_{t+1\mid t}(\alpha_t)-J_{t+1\mid t}(\gamma_t)
\right]
\le{}&
\frac{\Omega_{\Psi,\mathcal K}}{\eta_{T-1}}
+\sum_{t=2}^{T-1}
\frac{
\mathsf E\!\left[
\Delta_{\Psi,\mathcal K}(\gamma_t,\gamma_{t-1})
\right]}{\eta_t} \\
&+\frac{1}{2\sigma_\Psi}
\sum_{t=1}^{T-1}\eta_t\,
\mathsf E\!\left[\|\xi_t\|_*^2\right].
\end{align*}
The measurable reachable oracle may be used as
$\gamma_t=\alpha_t^\star$ whenever $\alpha_t^\star\in\mathcal K$ almost surely
and the same integrability conditions hold. The left-hand side then equals the
cumulative expected algorithmic gap.
\end{corollary}
\proofref{app::proof_expected_regret}

We consider two useful special cases. First, fix $\bar\alpha\in\mathcal K$ and
set $\gamma_t\equiv\bar\alpha$ in Corollary~\ref{cor::corollary}. Then
\begin{equation*}
    \sum_{t=1}^{T-1}
    \mathsf E\!\left[J_{t+1\mid t}(\alpha_t)-J_{t+1\mid t}(\bar\alpha)\right]
\end{equation*}
satisfies the same bound with zero comparator variation. This fixed comparator
measures the performance of online gain learning relative to a given constant
gain. Minimising
$\sum_{t=1}^{T-1}\mathsf E[J_{t+1\mid t}(\bar\alpha)]$ over
$\bar\alpha\in\mathcal K$ yields the best fixed gain whenever the minimum is
attained. A moving comparator instead allows the gain benchmark to respond to
breaks, regime changes, or shifts in the signal-to-noise ratio. The second
special case makes the learning-rate schedule explicit. Assume the hypotheses
of Theorem~\ref{thm:dynamic_regret_biom} and $\|\xi_t\|_{*} \leq G$ for every
input in the bound. If we let $\eta_t=\eta t^{-1/2}$ for some $\eta>0$, and
define the Bregman path variation
\begin{equation}
    V_{N}^{\Psi,\mathcal K}(\gamma):=
    \sum_{t=2}^{N}\Delta_{\Psi,\mathcal K}(\gamma_t,\gamma_{t-1}),
\end{equation}
then (see also the observation after Theorem~\ref{thm:dynamic_regret_biom}), using $\eta_N^{-1}=\eta^{-1}\sqrt N$, $\eta_t^{-1}\le\eta^{-1}\sqrt N$, and $\sum_{t=1}^N\eta_t\le 2\eta\sqrt N$,
\begin{equation}\label{eq:dynamic_regret_sqrt_stepsize}
      \sum_{t=1}^N\{\ell_t(\alpha_t)-\ell_t(\gamma_t)\}
  \le
  \frac{\Omega_{\Psi,\mathcal K}\sqrt{N}}{\eta}
  +\frac{\sqrt{N}}{\eta}V_{N}^{\Psi,\mathcal K}(\gamma)
  +\frac{\eta G^2}{\sigma_{\Psi}}\sqrt{N}.
\end{equation}
Equation~\eqref{eq:dynamic_regret_sqrt_stepsize} implies vanishing average
regret whenever $V_N^{\Psi,\mathcal K}(\gamma)=o(\sqrt N)$. For a fixed
comparator, $V_N^{\Psi,\mathcal K}(\gamma)=0$, giving average regret of order
$O(N^{-1/2})$. The same rate applies when total Bregman variation remains
bounded, as for a piecewise-stable oracle with finitely many bounded changes.
More generally, the path-variation term measures the difficulty of tracking a
moving gain oracle: gradual or infrequent changes preserve sublinear regret
when their accumulated variation is $o(\sqrt N)$, while sustained oscillation
raises the bound through the resulting path variation.

Theorem~\ref{thm:dynamic_regret_biom} applies to projected MD on any closed
convex decision set $\mathcal K$ with finite Bregman diameter. For logistic,
cloglog, and reverse-cloglog links, the result applies on a fixed interior
localisation $\mathcal A_\varepsilon$ whenever the Bregman-proximal projection,
an equivalent coordinatewise clipping rule, or an invariant-set argument keeps
both the iterates and comparator in $\mathcal A_\varepsilon$. On its full domain
$(0,\infty)^k$, the exponential-link potential has infinite Bregman diameter;
without verified compact localisation or compatible truncation, the
finite-diameter theorem does not apply. The discounted recursion below uses the
same geometric conditions and additionally requires $\rho_t/\eta_t$ to be
non-decreasing. These conditions concern the realised decision set and schedule
and must be verified for each implementation.

\begin{remark}[Geometry constants for link potentials]\label{rem::geometry_constants}
For a componentwise link, Eq.~\eqref{eq::phi_g} gives the diagonal Hessian
\[
\nabla^2\Psi_g(\alpha)
=\diag\!\left(
\frac{1}{g_j'(g_j^{-1}(\alpha_j))}
\right).
\]
If every link derivative is uniformly bounded above, then $\Psi_g$ is
$\sigma_\Psi$-strongly convex on $\inter(\mathcal A)$, with
\[
\sigma_\Psi
=\min_j\frac{1}{\sup_\vartheta g_j'(\vartheta)}>0.
\]
This modulus remains uniform as the gain approaches the boundary.
Assumption~\ref{ass::link} by itself allows link derivatives without a uniform
upper bound. For the logistic link
$g_j(\vartheta)=\alpha_{L,j}+D_j\sigma(\vartheta)$, where
$D_j=\alpha_{H,j}-\alpha_{L,j}$, we have
$\sup_\vartheta g_j'(\vartheta)=D_j/4$, so one may take
$\sigma_\Psi=\min_j4/D_j$. The Bregman diameter depends on the dual coordinate
$g^{-1}$, which diverges at the boundary. The supremum of
$\mathcal B_\Psi(\alpha,\beta)$ over interior pairs diverges as its second
argument approaches the boundary, so the logistic potential does not furnish
a finite Bregman diameter on the closed gain set $\mathcal A$. On the interior
localisation
\[
\mathcal A_\varepsilon
=\prod_j[\alpha_{L,j}+\varepsilon D_j,\,
         \alpha_{H,j}-\varepsilon D_j],
\]
the dual coordinate is bounded; for the logistic link,
$\Omega_{\Psi,\mathcal A_\varepsilon}=O(\log(1/\varepsilon))$.
Theorem~\ref{thm:dynamic_regret_biom} therefore applies on
$\mathcal A_\varepsilon$ with a strong-convexity modulus independent of
$\varepsilon$ and a finite Bregman diameter that grows only logarithmically as
$\varepsilon\downarrow0$. A Bregman-projected implementation, or an equivalent
coordinatewise-clipped link update on the box, inherits the bound with this
explicit boundary-margin dependence.
\end{remark}

When $\rho_t=1$, the two-centre update below reduces to the projected MD step
of Theorem~\ref{thm:dynamic_regret_biom}. We now allow $\rho_t\leq1$ and derive
a regret bound for the resulting discounted mirror step. The bound adds a
reference-gain remainder that vanishes when $\rho_t=1$.

\begin{theorem}[Dynamic regret of the discounted mirror step]\label{thm:dynamic_regret_dmd}
Let $\mathcal K\subseteq\mathcal A$ be a non-empty closed convex decision set,
and let $\Psi$ be differentiable on an open convex neighbourhood of $\mathcal K$
and $\sigma_\Psi$-strongly convex on $\mathcal K$ with respect to $\|\cdot\|$,
whose dual norm is $\|\cdot\|_*$. Define
\[
\Omega_{\Psi,\mathcal K}
:=\sup_{a,b\in\mathcal K}\mathcal B_\Psi(a,b)<\infty,
\qquad
\Delta_{\Psi,\mathcal K}(u,v)
:=\sup_{a\in\mathcal K}
\{\mathcal B_\Psi(u,a)-\mathcal B_\Psi(v,a)\}.
\]
The link-localised case uses $\Psi=\Psi_g$ and
$\mathcal K=\mathcal A_\varepsilon$, while Euclidean projection uses
$\Psi(\alpha)=\|\alpha\|^2/2$ and $\mathcal K=\mathcal A$.
Let $\ell_t:\mathcal K\to\mathbb R$, $t=1,\ldots,N$, be finite convex
losses. Let $\eta_t>0$ and $0\leq\rho_t\leq1$, and define
\[
\kappa_t:=\frac{1-\rho_t}{\eta_t},
\qquad
w_t:=\frac{\rho_t}{\eta_t}.
\]
Fix a reference gain $\bar\alpha\in\mathcal K$ and an initial gain
$\alpha_1\in\mathcal K$. After incurring $\ell_t(\alpha_t)$, the algorithm
receives $\xi_t\in\partial\ell_t(\alpha_t)$ and updates
\[
\alpha_{t+1}
=\argmin_{\alpha\in\mathcal K}
\left\{
\eta_t\langle\xi_t,\alpha\rangle
+(1-\rho_t)\mathcal B_\Psi(\alpha,\bar\alpha)
+\rho_t\mathcal B_\Psi(\alpha,\alpha_t)
\right\}.
\]
Assume that $(w_t)_{t=1}^N$ is non-decreasing. For the exact interior-link
recursion $\alpha_{t+1}=g(\vartheta_{t+1})$, the displayed update coincides
with the unconstrained minimiser of Proposition~\ref{prop::discounted_md}
whenever that minimiser lies in $\mathcal K$; otherwise, the projected
minimiser above defines the algorithm. Then, for every comparator path
$\gamma_1,\ldots,\gamma_N\in\mathcal K$,
\begin{align}
\label{eq:dynamic_regret_dmd_bound}
  \sum_{t=1}^N\{\ell_t(\alpha_t)-\ell_t(\gamma_t)\}
  \;\le\;&
  \frac{\rho_N\,\Omega_{\Psi,\mathcal K}}{\eta_{N}}
  +\sum_{t=2}^N\frac{\rho_{t-1}}{\eta_{t-1}}\,\Delta_{\Psi,\mathcal K}(\gamma_t,\gamma_{t-1})
  +\frac{1}{2\sigma_\Psi}\sum_{t=1}^N\eta_t\|\xi_t\|_*^2
  \nonumber\\
  &+\underbrace{\sum_{t=1}^N\frac{1-\rho_t}{\eta_t}\,
     \Big[\mathcal{B}_{\Psi}(\gamma_t,\bar\alpha)
          +\mathcal{B}_{\Psi}(\alpha_{t+1},\alpha_t)\Big]}_{\text{reference-gain remainder }\;\mathcal{B}^{\mathrm{ref}}_N}.
\end{align}
Using $\mathcal{B}_{\Psi}\le\Omega_{\Psi,\mathcal K}$, $\rho_N\le1$, and (when $(\eta_t)$ is also non-increasing) $\rho_{t-1}/\eta_{t-1}\le1/\eta_t$, the bound admits the cleaner form
{\small
\begin{align}
\label{eq:dynamic_regret_dmd_bound_loose}
  \sum_{t=1}^N\{\ell_t(\alpha_t)-\ell_t(\gamma_t)\}
  \;\le{}&
  \frac{\Omega_{\Psi,\mathcal K}}{\eta_{N}}
  +\sum_{t=2}^N\frac{\Delta_{\Psi,\mathcal K}(\gamma_t,\gamma_{t-1})}{\eta_{t}}
  +\frac{1}{2\sigma_\Psi}\sum_{t=1}^N\eta_t\|\xi_t\|_*^2
  \nonumber\\
  &\quad
  +\sum_{t=1}^N\kappa_t\,\mathcal{B}_{\Psi}(\gamma_t,\bar\alpha)
  +\Omega_{\Psi,\mathcal K}\sum_{t=1}^N\kappa_t,
\end{align}
}
The coefficient $\kappa_t$ is the effective dual shrinkage introduced after
Eq.~\eqref{eq::objective_gas}. When $\rho_t=1$ for every $t$, we have
$\kappa_t=0$ and $w_t=1/\eta_t$. The reference-gain remainder then vanishes,
the update becomes the projected MD recursion of
Theorem~\ref{thm:dynamic_regret_biom}, and the cleaner bound in
Eq.~\eqref{eq:dynamic_regret_dmd_bound_loose} coincides with
Eq.~\eqref{eq:dynamic_regret_bound_biom}.
\end{theorem}
\proofref{app::AppendixDMD}

\medskip
\begin{corollary}[Constant persistence, $\sqrt{N}$ schedule]\label{cor::dmd_sqrt}
Under the hypotheses of Theorem~\ref{thm:dynamic_regret_dmd}, take $\rho_t\equiv\rho\in[0,1)$ constant and $\eta_t=\eta\,t^{-1/2}$ with $\eta>0$, so that $w_t=\rho\sqrt{t}/\eta$ is non-decreasing. Assume $\|\xi_t\|_*\le G$ and write $V_N^{\Psi,\mathcal K}(\gamma)=\sum_{t=2}^N\Delta_{\Psi,\mathcal K}(\gamma_t,\gamma_{t-1})$. Then
\begin{equation*}
    \sum_{t=1}^N\{\ell_t(\alpha_t)-\ell_t(\gamma_t)\}
    \le\frac{\rho\,\Omega_{\Psi,\mathcal K}\sqrt{N}}{\eta}
    +\frac{\rho\sqrt{N}}{\eta}\,V_N^{\Psi,\mathcal K}(\gamma)
    +\frac{\eta G^2}{\sigma_\Psi}\sqrt{N}
    +\frac{1-\rho}{\eta}\sum_{t=1}^N\sqrt{t}\,\big[\mathcal{B}_\Psi(\gamma_t,\bar\alpha)+\Omega_{\Psi,\mathcal K}\big].
\end{equation*}
\end{corollary}
\proofref{app::proof_dmd_sqrt}
The first and third terms are $O(\sqrt{N})$, while the path-variation term is
$O\big(\sqrt{N}\,V_N^{\Psi,\mathcal K}(\gamma)\big)$, as in
Eq.~\eqref{eq:dynamic_regret_sqrt_stepsize}. The reference-gain remainder is
$O\big((1-\rho)\,\Omega_{\Psi,\mathcal K}\,N^{3/2}/\eta\big)$. Thus, with
fixed persistence $\rho<1$ and $\eta_t\propto t^{-1/2}$, this remainder
contributes $O(\sqrt{N})$ to the average-regret bound for a general comparator.
Equivalently, the effective shrinkage $\kappa_t=(1-\rho)/\eta_t$ grows as
$\sqrt{t}$ under this schedule; see the discussion after
Eq.~\eqref{eq::objective_gas}.

\medskip
\begin{corollary}[Coupled schedule, constant shrinkage $\rho_t=1-\kappa\eta_t$]\label{cor::dmd_coupled}
Under the hypotheses of Theorem~\ref{thm:dynamic_regret_dmd}, couple the schedules by $\rho_t=1-\kappa\eta_t$ for a constant $\kappa\ge0$ (so $\kappa_t\equiv\kappa$), with $\eta_t=\eta\,t^{-1/2}$ and the feasibility clause $\kappa\eta\le1$ (which ensures $\rho_t\in[0,1]$; then $w_t=1/\eta_t-\kappa$ is non-decreasing). Assume $\|\xi_t\|_*\le G$. Then
\begin{equation*}
    \sum_{t=1}^N\{\ell_t(\alpha_t)-\ell_t(\gamma_t)\}
    \le\frac{\Omega_{\Psi,\mathcal K}\sqrt{N}}{\eta}
    +\frac{\sqrt{N}}{\eta}\,V_N^{\Psi,\mathcal K}(\gamma)
    +\frac{\eta G^2}{\sigma_\Psi}\sqrt{N}
    +\kappa\sum_{t=1}^N\mathcal{B}_\Psi(\gamma_t,\bar\alpha)
    +\kappa\,N\,\Omega_{\Psi,\mathcal K},
\end{equation*}
and consequently the average regret satisfies
\begin{equation*}
    \frac{1}{N}\sum_{t=1}^N\{\ell_t(\alpha_t)-\ell_t(\gamma_t)\}
    \le\frac{\Omega_{\Psi,\mathcal K}}{\eta\sqrt{N}}+\frac{V_N^{\Psi,\mathcal K}(\gamma)}{\eta\sqrt{N}}+\frac{\eta G^2}{\sigma_\Psi\sqrt{N}}
    +\kappa\Big(\Omega_{\Psi,\mathcal K}+\frac{1}{N}\sum_{t=1}^N\mathcal{B}_\Psi(\gamma_t,\bar\alpha)\Big).
\end{equation*}
Suppose
\[
V_N^{\Psi,\mathcal K}(\gamma)=o(\sqrt N)
\quad\text{and}\quad
\frac{1}{N}\sum_{t=1}^N\mathcal B_\Psi(\gamma_t,\bar\alpha)
\longrightarrow\overline{\mathcal B}.
\]
Then
\[
\limsup_{N\to\infty}\frac{1}{N}\sum_{t=1}^N
\{\ell_t(\alpha_t)-\ell_t(\gamma_t)\}
\le
\kappa\big(\Omega_{\Psi,\mathcal K}+\overline{\mathcal B}\big).
\]
\end{corollary}
\proofref{app::proof_dmd_coupled}
The final term records the contribution of dual shrinkage to the upper bound.
For horizon-indexed schedules, taking $\kappa=\kappa_N\to0$ anneals this
contribution. For example, $\kappa_N=c/\sqrt N$ makes the shrinkage contribution
$O(N^{-1/2})$; if $V_N^{\Psi,\mathcal K}(\gamma)=O(1)$, the overall
average-regret bound has the same order.

\medskip
\noindent\emph{Expectation form (analogue of Corollary~\ref{cor::corollary}).}
Under the hypotheses of Theorem~\ref{thm:dynamic_regret_dmd}, suppose in
addition that $(\eta_t,\rho_t)$ is deterministic and $(\eta_t)$ is
non-increasing.
Let $L_{t+1}$ be finite and convex on $\mathcal K$ almost surely, with
$\xi_t\in\partial L_{t+1}(\alpha_t)$. Assume that
$\alpha_t,\gamma_t\in\mathcal K$ are $\mathcal F_t$-measurable, that $\xi_t$
is $\mathcal F_{t+1}$-measurable, and that
\[
L_{t+1}(\alpha_t),\quad L_{t+1}(\gamma_t),\quad
\Delta_{\Psi,\mathcal K}(\gamma_t,\gamma_{t-1}),\quad
\mathcal B_\Psi(\gamma_t,\bar\alpha),\quad \|\xi_t\|_*^2
\]
are integrable. Taking expectations in
Eq.~\eqref{eq:dynamic_regret_dmd_bound_loose} and using the tower property gives
\begin{align*}
\sum_{t=1}^N\mathsf E\!\left[
J_{t+1\mid t}(\alpha_t)-J_{t+1\mid t}(\gamma_t)\right]
\le{}&
\frac{\Omega_{\Psi,\mathcal K}}{\eta_N}
+\sum_{t=2}^N
\frac{\mathsf E[\Delta_{\Psi,\mathcal K}(\gamma_t,\gamma_{t-1})]}{\eta_t}\\
&+\frac{1}{2\sigma_\Psi}\sum_{t=1}^N
\eta_t\,\mathsf E[\|\xi_t\|_*^2]
+\sum_{t=1}^N\kappa_t\,
\mathsf E[\mathcal B_\Psi(\gamma_t,\bar\alpha)]\\
&+\Omega_{\Psi,\mathcal K}\sum_{t=1}^N\kappa_t.
\end{align*}

\section{Numerical diagnostics}\label{sec::SectionSix}
Experiment 6.1 studies the gradient identity and its Kalman-gain interpretation
in a Gaussian local-level model. Experiment 6.2 holds a designed quadratic state
target fixed to separate score scaling, link geometry, and gain dimension.
Experiment 6.3 evaluates the dynamic-regret mechanism as comparator variation
changes, and Experiment 6.4 compares persistent bounded and exponential-link
rules around breaks. Experiments 6.2 and 6.3 use controlled common contexts. The
recursive filters in Experiments 6.1 and 6.4 share observations, score formula,
and scaling rule, but their realised states, scores, and reachable sets are
method-specific.

Rule names combine the recursion --- memoryless \textrm{MD} or persistent
discounted mirror descent (\textrm{DMD}) --- with the link: unbounded exponential
(\emph{exp}), bounded logistic (\emph{logit}), or Euclidean projection
(\emph{proj}). With a positive learning step, \textrm{DMD-exp} and
\textrm{DMD-logit} are exponential- and logistic-link specialisations of
\textrm{aGAS}~\cite{blasques2019accelerating}; constant gain and
\textrm{AdaGrad} complete the set. The compact regret guarantee covers the
bounded or localised positive-step rules, not untruncated \textrm{DMD-exp}.
\ref{app::protocol} reports link-specific mobilities.

Specifically, \textrm{MD} uses
$\vartheta_{t+1}=\vartheta_t-\eta_t\xi_t$, whereas \textrm{DMD} uses
$\vartheta_{t+1}=(1-\rho_t)\bar\vartheta+\rho_t\vartheta_t-\eta_t\xi_t$.
The suffix \emph{exp} denotes an exponential gain map: the switching engine
uses $\alpha_t=\mathrm e^{f_{t+1}/2}$, whereas the empirical engine uses
$\alpha_t=\mathrm e^{f_{t+1}}$. Away from numerical clipping, these are
equivalent parameterisations after the corresponding rescaling of $f$, its
intercept, and its score-product coefficient; their engine-specific clips and
gain ceilings are not equivalent. The suffix \emph{logit} denotes a bounded
logistic map and \emph{proj} Euclidean projection onto $\mathcal A$. In the
empirical tables and figures, the label \textrm{DMD-exp} is retained for
compatibility with the frozen outputs and denotes the unrestricted, numerically
capped exponential-\textrm{aGAS} benchmark. It has a positive-step
\textrm{DMD} interpretation only when its fitted score-product coefficient is
positive.

Experiments~6.1 and~6.2 select a common bounded-gain interval from a broad range
and screen candidate links by a mobility-weighted path cost. Their gain-rule
hyperparameters are fitted in sample. The switching benchmark fixes
$\mathcal A=[0,1]$ for the bounded rules and estimates each specification
path by path using predictive maximum likelihood. The regret diagnostic fixes
its interval and learning-rate schedule. In the empirical panel, a
market-specific interval is selected once from the initial training window and
held fixed across bounded rules and subsequent refits. \ref{app::numdetails}
gives the details. We report Monte Carlo means of predictive loss, state RMSE
where available, and gain diagnostics, with paired standard errors for central
comparisons.

The synthetic realised-loss results are mechanism diagnostics rather than genuine
out-of-sample forecasts because their hyperparameters are fitted on the design
being evaluated. Pilot selection is nevertheless common across methods, so no
rule receives a different admissible range. When a latent state or oracle is
available, RMSE shows whether better prediction also means better tracking.

\subsection{Gaussian local-level gain tracking}\label{subsec::one}
We use Example~\ref{ex::example}'s Gaussian local-level score specification with fixed unit
observation variance and scaled-score direction $d_t=y_t-\lambda_t$. A
time-varying state-innovation variance generates a moving Kalman-gain schedule.
We simulate 256 paths of length $T=1080$, with initial prediction variance
$P_1=1/9$, initial mean $m_1=0$, and gain interval
$\mathcal A=[0.02,0.80]$.
Table~\ref{tab:exp1_fixed_score_line_biom} compares \textrm{MD-proj},
\textrm{MD-logit}, and \textrm{DMD-logit}; reverse-cloglog appears only in pilot
screening (\ref{app::protocol}).

The target gain is piecewise smooth and projected forward to satisfy the
Kalman-compatible constraint. Write
\begin{equation}\label{eq::schedule}
    \mathcal P_t(c_1,\ldots,c_K;b_1,\ldots,b_{K-1})=c_k
    \quad\text{when } b_{k-1}<t\leq b_k,
\end{equation}
where $b_0=0$ and $b_K=T$. Thus, $c_k$ is the level over the $k$th regime.
A slow sinusoid makes the
target move within regimes, not only at breaks. \ref{app::fixedscore}
gives the schedule, the constraint
$\alpha_{t+1}\ge\alpha_t/(1+\alpha_t)+10^{-4}$, and the inversion producing
$\{q_t\}$. We compare the infeasible Kalman oracle, a fitted constant gain,
projected OMD, logistic MD, and logistic DMD with fitted persistence.

The inversion selects state-innovation variances so that the Riccati recursion
started from $P_1=1/9$ returns the declared schedule. The latent state is drawn
from the matching prior $x_1\sim\mathcal N(0,P_1)$ and the oracle filter fixes
$m_1=0$. The benchmark is therefore the correctly specified Kalman filter for
the simulated DGP. \ref{app::fixedscore} gives the inversion and
initialisation.

In this experiment, $\lambda_t$ denotes the running state of a generic filter
and equals $m_t$ for the Kalman oracle.
At each date $t\geq2$, the adaptive rules use
$\xi_{t-1}=-e_{t-1}e_t$, where $e_t=y_t-\lambda_t$, to update the gain from
$\alpha_{t-1}$ to $\alpha_t$, and then apply $\alpha_t$ to the current forecast
error in the state update. The constant minimises realised predictive loss. Projected OMD uses
$\alpha_{t+1}=\Pi_{\mathcal A}(\alpha_t-\eta\xi_t)$; logistic MD maps an
unconstrained $\vartheta_t$ into $\mathcal A$ and updates
$\vartheta_{t+1}=\vartheta_t-\eta_t\xi_t$. Logistic DMD adds
$\vartheta_{t+1}=(1-\rho)\bar\vartheta+\rho\vartheta_t-\eta_t\xi_t$.
The logistic rules therefore share their geometry and differ only in memory,
while the projected rule supplies a Euclidean benchmark.

\begin{table}[t]
\centering
\small
\begin{tabular}{lrrrr}
\toprule
Method & Mean loss & State RMSE & Gain RMSE & Mean gain \\
\midrule
Kalman oracle & 2.0245 & 0.6292 & 0.0000 & 0.3981 \\
Constant & 2.1640 & 0.7330 & 0.2962 & 0.5367 \\
MD-proj & 2.1057 & 0.6917 & 0.1813 & 0.4352 \\
MD-logit & 2.1099 & 0.6942 & 0.1945 & 0.4586 \\
DMD-logit & 2.1001 & 0.6875 & 0.1722 & 0.4257 \\
\bottomrule
\end{tabular}
\caption{\textbf{Gaussian local-level gain tracking.} All feasible filters use the same observations, Gaussian score formula, unit scaling, and gain interval $\mathcal A=[0.02,0.80]$; their recursively generated states and score realisations are method-specific. Entries average $256$ paths of length $T=1080$. The infeasible Kalman oracle uses the correctly specified prior and time-varying gain schedule described in \ref{app::fixedscore}; gain RMSE is computed relative to that schedule.}
\label{tab:exp1_fixed_score_line_biom}
\end{table}

The constant gain remains competitive, but \textrm{DMD-logit} is the best feasible
adaptive rule. Relative to the constant, it lowers mean loss from $2.1640$ to
$2.1001$ (paired difference $-0.0639$, standard error $0.0016$) and state RMSE
from $0.7330$ to $0.6875$ (difference $-0.0456$, standard error $0.0009$).
Comparing the two logistic rows shows that persistence helps track the structured
target. Appendix Figure~\ref{fig:exp1_fixed_score_line_paths_biom} plots the paths.
Each adaptive gain changes distance along its own realised forecast-error line.

Because the sinusoidal component moves the target within regimes, the design
distinguishes continuous tracking from reactions only at breaks. All methods
use a common score formula and common observations, while recursive feedback
produces method-specific forecast errors. The comparison therefore evaluates
complete filters. Its interpretation is empirical tracking performance across
filtering rules; the common-context regret theorem addresses a different
comparison.

\subsection{Bivariate scaling, link geometry and vector gains}\label{subsec::two}
This experiment separates score scaling, link geometry, and gain dimension. Let
$e_t$ denote the simulated raw-error vector, with
$\Sigma_t=\diag(\tau_t^2,1)$, and set $d_t=B_te_t$, where
$B_t=\diag(b_{1,t},1)$. We compare unit scaling
($b_{1,t}=\tau_t^{-2}$), inverse-square-root Fisher scaling
($b_{1,t}=\tau_t^{-1}$), and inverse-Fisher scaling ($b_{1,t}=1$) while
holding a designed quadratic state target fixed. The design
uses $T=420$ and 256 Monte Carlo paths.
\ref{app::bivariate} gives the DGP algebra.

The time-varying scale $\tau_t$ changes the first coordinate's relative
volatility. The design also includes a moving designed gain target, AR(1) raw
errors, and heterogeneous coordinate speeds. The weighted state loss
reduces to a scalar gain problem for a common gain, making the scalar and diagonal
comparisons directly interpretable.

\subsubsection{Score scaling and target feasibility}
Score scaling determines whether the coordinatewise minimiser of the designed
state loss is feasible. Let
$\delta_t^{\mathrm{state}}=\beta_t\odot e_t$ and
$\alpha_t^{\mathrm{unc}}=\delta_t^{\mathrm{state}}\oslash d_t$. This quantity
minimises the designed weighted quadratic state loss specified in
\ref{app::bivariate}. The conditional predictive oracle
$\mathsf E_t[y_{t+1}]-\lambda_t$ of Section~\ref{sec::SectionTwo} is a distinct
object. When the minimiser lies outside the admissible interval, projection
leaves a state-target gap under the chosen scaling.
Table~\ref{tab:exp2_oracle_feasibility_biom} and
Figure~\ref{fig:exp2_component_oracle_feasibility_biom} show that unit scaling
often makes the first-coordinate target minimiser infeasible; inverse-square-root
Fisher scaling reduces but retains the problem, whereas inverse-Fisher scaling
removes it. Appendix
Figure~\ref{fig:exp2_score_geometry_biom} shows the underlying geometry.

The exceedance frequency measures how often the target lies outside the
admissible interval. Under unit scaling,
the first-coordinate minimiser exceeds the admissible range on nearly half the
dates. Inverse-square-root Fisher scaling lowers the exceedance magnitude while
leaving its frequency unchanged; inverse-Fisher scaling places the coordinatewise
gain targets in common units and restores feasibility. In this design, score scaling determines
target feasibility, while link geometry and persistence govern adaptation within
the admissible set.

\begin{table}[!ht]
\centering
\small
\begin{tabular}{lrrrr}
\toprule
Scaling & Max \(\alpha_1^{\mathrm{unc}}\) & \(\alpha_1^{\mathrm{unc}}>\alpha_H\) & Mean clipped \(\alpha_1\) & Mean clipped \(\alpha_2\) \\
\midrule
Unit & 5.5848 & 0.4881 & 0.6718 & 0.6171 \\
\(\mathrm{Fisher}^{-1/2}\) & 2.3270 & 0.4881 & 0.6902 & 0.6171 \\
\(\mathrm{Fisher}^{-1}\) & 1.0700 & 0.0000 & 0.6171 & 0.6171 \\
\bottomrule
\end{tabular}
\caption{\textbf{Designed target feasibility under alternative score scalings.} Entries hold the designed quadratic state target fixed while varying score scaling and hence the update geometry. The table reports the first-coordinate unconstrained loss minimiser, its exceedance frequency, and the mean clipped minimisers.}
\label{tab:exp2_oracle_feasibility_biom}
\end{table}

\begin{figure}[!ht]
\centering
\includegraphics[width=1\linewidth]{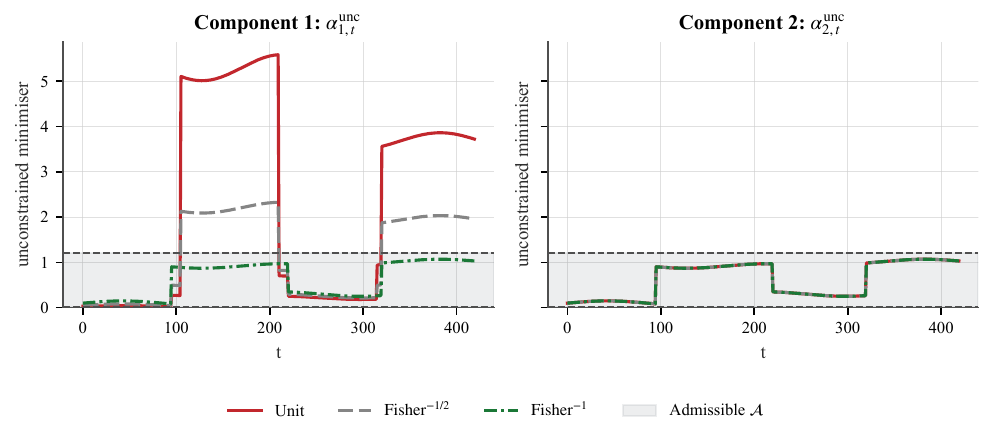}
\caption{\textbf{Componentwise target feasibility.} The panels plot the unconstrained designed-loss minimisers under the three scaling choices and the common admissible interval. Unit and inverse-square-root Fisher scaling frequently place the first-coordinate minimiser above the interval, whereas inverse-Fisher scaling keeps it feasible.}
\label{fig:exp2_component_oracle_feasibility_biom}
\end{figure}
\ref{app::bivariate} reports two additional unit-scaling comparisons. Among common
scalar gains, \textrm{DMD-logit} has the lowest mean loss and cloglog-OMD the
lowest state-target RMSE; relative to the constant, the \textrm{DMD-logit} loss
reduction is $0.2570$ (standard error $0.0044$). Allowing diagonal gains improves
the well-scaled second coordinate, while the first-coordinate feasibility gap
remains.

This comparison separates gain dimension from score normalisation. An additional
gain coordinate permits coordinate-specific adaptation where scores are well
scaled, while targets outside the feasible interval remain constrained by
scaling. Gain dimensionality and score normalisation therefore address distinct
aspects of update design.

\subsubsection{Heterogeneous inverse-Fisher-scaled design: scalar versus vector gains}
Under inverse-Fisher scaling, both coordinatewise targets are feasible in this
design. Heterogeneous adjustment speeds then place the designed target in the
diagonal gain class.
Table~\ref{tab:exp2_heterogeneous_biom} shows common gains near $0.51$ and
diagonal gains near $(0.86,0.21)$. Diagonal \textrm{MD-proj} has loss $0.0026$
and state-target RMSE $0.0675$, against $0.0186$ and $0.3800$ for common
\textrm{MD-logit}; paired reductions are $0.0160$ (standard error $0.0003$) and
$0.3102$ in mean pathwise RMSE (standard error $0.0029$). Figure~\ref{fig:exp2_heterogeneous_gain_paths_biom}
and Appendix Figure~\ref{fig:exp2_heterogeneous_state_biom} show why: a common
gain must compromise between a high and a low coordinate-specific target.

With a common target speed, the designed target lies in the common-gain class
and a scalar gain can pool information across coordinates. Heterogeneous target
speeds place it in the diagonal class, making coordinate-specific gains valuable.
The comparison therefore isolates gain dimension from a generic increase in
flexibility.

\vspace{0.1 cm}
These designed quadratic experiments mirror the benchmark decomposition in
Eq.~\eqref{eq:benchmark_decomposition_biom}. They isolate update geometry under
the state-target loss; Section~\ref{sec::SectionTwo} studies the distinct
conditional predictive Kullback--Leibler criterion.

\begin{table}[!ht]
\centering
\small
\begin{tabular}{llrrrr}
\toprule
Method & Gain & Mean loss & State-target RMSE & Mean \(\alpha_1\) & Mean \(\alpha_2\) \\
\midrule
MD-proj & common & 0.0188 & 0.3823 & 0.5131 & 0.5131 \\
MD-logit & common & 0.0186 & 0.3800 & 0.5125 & 0.5125 \\
DMD-logit & common & 0.0173 & 0.3886 & 0.5157 & 0.5157 \\
Constant & diagonal & 0.0298 & 0.2404 & 0.8706 & 0.2166 \\
MD-proj & diagonal & 0.0026 & 0.0675 & 0.8611 & 0.2124 \\
MD-logit & diagonal & 0.0032 & 0.0704 & 0.8608 & 0.2099 \\
DMD-logit & diagonal & 0.0029 & 0.0686 & 0.8647 & 0.2205 \\
\bottomrule
\end{tabular}
\caption{\textbf{Heterogeneous inverse-Fisher-scaled design: common scalar versus diagonal vector gains.} Inverse-Fisher scaling places the coordinatewise designed-loss minimisers inside the common admissible interval, while the coordinates require different gains. Common rules learn a compromise near $0.51$; diagonal rules learn approximately $(0.86,0.21)$ and reduce quadratic loss and state-target RMSE.}
\label{tab:exp2_heterogeneous_biom}
\end{table}

\begin{figure}[t]
\centering
\includegraphics[width=1\linewidth]{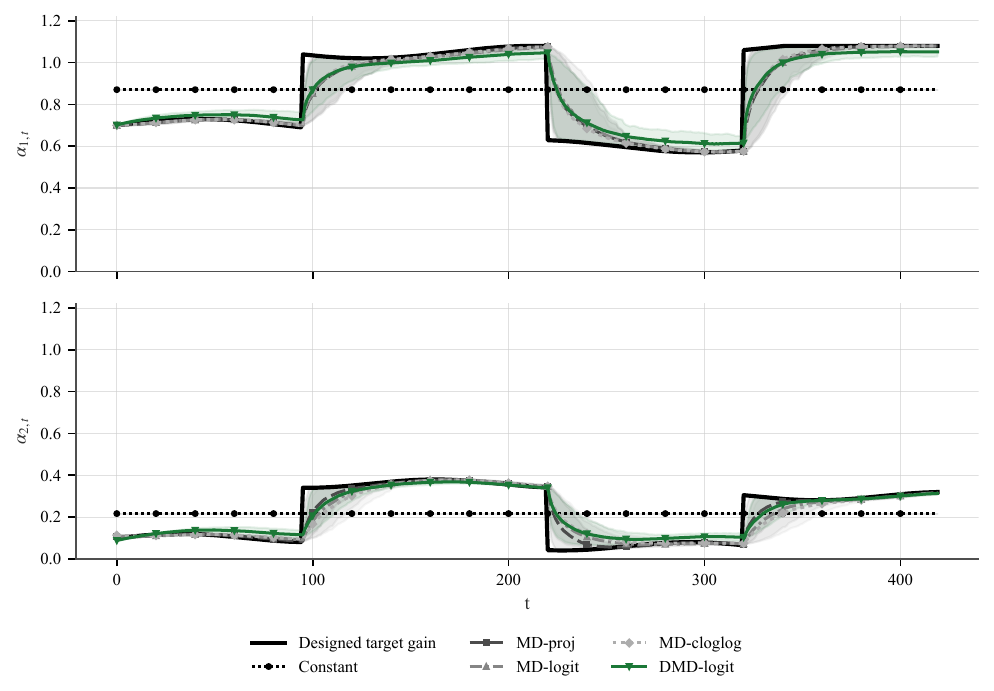}
\caption{\textbf{Diagonal gain paths in the heterogeneous inverse-Fisher-scaled design.} The panels report learned coordinate-specific gains and designed target gains. The first coordinate requires a high gain and the second a low gain; a common scalar gain must compromise.}
\label{fig:exp2_heterogeneous_gain_paths_biom}
\end{figure}

\subsection{Dynamic regret under moving comparators}\label{subsec::three}
This convex experiment examines how comparator variation affects dynamic regret,
holding estimation, score scaling, and link choice fixed. For $t=1,\ldots,T$, define
\begin{equation*}
    L_t(\alpha)=\frac{1}{2}(\delta_t-\alpha d_t)^2,\quad\quad \delta_t=\beta_t d_t + r_t,
\end{equation*}
where $d_t\in[-1,1]$, $r_t$ is bounded, and
$\mathcal A=[0.02,0.98]$. Projected Euclidean OMD uses
$\eta_t=0.35/\sqrt t$. Across 512 paths of length $T=800$, we compare static,
smooth-moving, and fast-moving comparator paths using total variation
$V_T(\beta)=\sum_{t=2}^T|\beta_t-\beta_{t-1}|$, cumulative regret, and the
realised mirror-step certificate derived in the proof of
Theorem~\ref{thm:dynamic_regret_biom}.
Table~\ref{tab:regret_bound_diagnostic_biom} shows mean regret rising from
$0.0041$ for the static comparator, whose path length is zero, to $2.6250$ and
$5.2364$ for path lengths $1.1178$ and $11.1780$. The mean certificate exceeds
mean regret in each design.
Figure~\ref{fig:regret_bound_diagnostic_biom} shows cumulative regret rising
more persistently as the comparator moves more rapidly.

For each path, cumulative regret is
$R_T=\sum_{t=1}^T\{L_t(\alpha_t)-L_t(\beta_t)\}$. The three designs use the
same bounded-loss construction and learning-rate schedule, with comparator
variation as the only feature deliberately changed across designs.

The realised certificate retains the observed gradient energy and comparator
variation, yielding a sharper bound than one based on a uniform gradient
envelope. At the end of the sample, its Monte Carlo mean is approximately twice
mean regret in each design; this ratio is specific to the experiment. The timing
of accumulation also matches the theory. After the initial learning phase,
regret is nearly flat under the static comparator, rises mainly when the smooth
comparator begins to move, and continues to increase as the fast comparator
changes throughout the sample. This experiment complements the preceding
gain-learning experiments by isolating the dynamic-regret mechanism.

\begin{table}[!ht]
\centering
\small
\begin{tabular}{lrrrr}
\toprule
Design & Path length & Final regret & Certificate & Cert./regret \\
\midrule
Static & 0.0000 & 0.0041 & 0.0084 & 2.0372 \\
Smooth & 1.1178 & 2.6250 & 5.2447 & 1.9980 \\
Fast & 11.1780 & 5.2364 & 10.4661 & 1.9987 \\
\bottomrule
\end{tabular}
\caption{\textbf{Dynamic regret for projected OMD on bounded quadratic gain losses.} The table reports Monte Carlo averages over $512$ paths of length $T=800$. The comparator is $\beta_t$, and path length is its total variation. Final regret is cumulative regret relative to $\beta_t$. The certificate is the realised mirror-step quantity computed from the observed gradients and comparator path.}
\label{tab:regret_bound_diagnostic_biom}
\end{table}

\begin{figure}[t]
\centering
\includegraphics[width=1\linewidth]{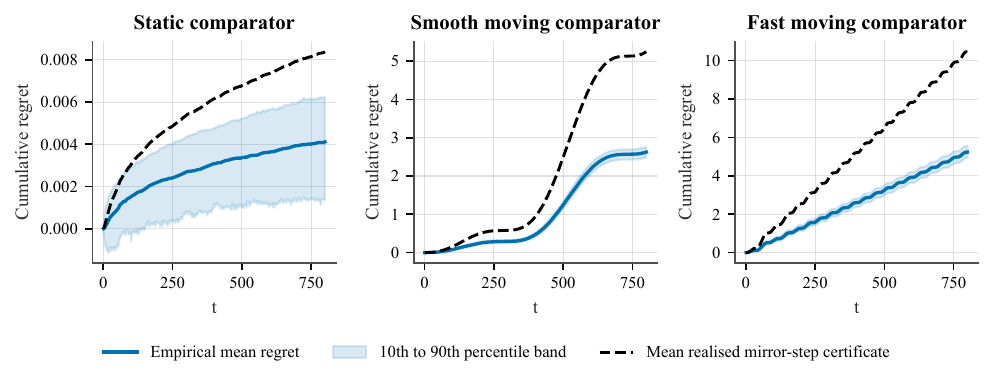}
\caption{\textbf{Dynamic regret under comparator variation.} Each panel reports cumulative regret for projected OMD on bounded quadratic gain losses. The solid line is the Monte Carlo mean of cumulative regret, the shaded region is the 10th--90th percentile pathwise band, and the dashed line is the Monte Carlo mean of the realised mirror-step certificate. The panels correspond to static, smooth-moving, and fast-moving comparator paths. Both cumulative regret and the certificate increase with comparator variation.}
\label{fig:regret_bound_diagnostic_biom}
\end{figure}

\FloatBarrier
\subsection{Switching local-level benchmark and relation to \textrm{aGAS}}\label{subsec::four}
Our final experiment revisits the switching local-level benchmark in Section 4
of Blasques~et~al.~\cite{blasques2019accelerating}:
\begin{equation*}
    y_t = \tilde \mu_t + \varepsilon_t,\quad t \in \mathbb{Z},
\end{equation*}
where $\varepsilon_t$ is i.i.d.\ standard Gaussian and $\tilde\mu_t$ switches
between $0$ and $\delta$ every $100\gamma$ periods. Because Fisher information is
constant, score scaling is fixed across specifications. At each grid point, we
estimate each specification separately by maximum likelihood on every path and
evaluate filtered-state RMSE. We use samples of length $T=1000$ and $M=1000$
Monte Carlo replications per grid cell. Table~\ref{tab:blasques_biom} reports
the $\gamma=2.5$ slice; Appendix Table~\ref{tab:blasques_full_biom} reports the
full grid.

The comparison includes the constant-gain filter, the exponential-link
\textrm{aGAS} specification \textrm{DMD-exp}, the persistent bounded rules
\textrm{DMD-logit} and \textrm{DMD-proj}, the memoryless \textrm{MD-logit} rule,
and \textrm{AdaGrad}.

\begin{table}[!ht]
\centering
\small
\begin{tabular}{ccrrrrrr}
\toprule
\(\delta\) & \(\gamma\) & Constant & \textrm{DMD-exp} & \textrm{DMD-logit} & \textrm{DMD-proj} & \textrm{MD-logit} & AdaGrad \\
\midrule
0.0 & 2.5 & 0.030 & 0.044 & 0.033 & \textbf{0.030} & 0.031 & \underline{0.030} \\
0.5 & 2.5 & 0.169 & 0.169 & \textbf{0.169} & 0.170 & 0.169 & \underline{0.169} \\
1.0 & 2.5 & 0.239 & \underline{0.232} & \textbf{0.231} & 0.235 & 0.238 & 0.238 \\
2.0 & 2.5 & 0.342 & \underline{0.290} & \textbf{0.285} & 0.293 & 0.339 & 0.340 \\
3.0 & 2.5 & 0.424 & 0.332 & \textbf{0.320} & \underline{0.325} & 0.412 & 0.418 \\
\bottomrule
\end{tabular}
\caption{\textbf{Switching local-level benchmark (\(\gamma=2.5\) slice).} Entries are Monte Carlo mean filtered-state RMSEs for the switching local-level design at regime-length parameter \(\gamma=2.5\) (regimes of \(250\) periods), with one row per break size \(\delta\). \textrm{DMD-exp} is the nominally unbounded exponential-link \textrm{aGAS} specification; the remaining adaptive rules are bounded mirror variants. Each specification is estimated path by path by MLE on samples of length \(T=1000\), using \(M=1000\) replications per grid cell. The best method in each row is shown in bold and the second best is underlined; ranking uses full precision, so the displayed values may coincide at three decimals. Appendix Table~\ref{tab:blasques_full_biom} reports the full \(\delta\times\gamma\) grid.}
\label{tab:blasques_biom}
\end{table}

At $\delta=0$, the constant-gain and bounded rules have similar RMSEs;
\textrm{DMD-exp} performs worst, consistent with greater noise sensitivity. The
constant-gain rule is competitive for small breaks. For larger breaks,
discounting reduces lag: \textrm{DMD-exp} outperforms the constant, and bounded
\textrm{DMD-logit} performs best. Standard errors of paired RMSE differences
from the constant are at most $0.0014$ (median $0.0003$); three-decimal ties are
practically equivalent.

Across the full grid, discounting allows gains to rise after a break and revert
within a regime. In stable regimes, the same reactivity can amplify noise, which
helps explain \textrm{DMD-exp}'s higher RMSE at $\delta=0$. Relative to the
constant-gain rule, both the large-break improvements and \textrm{DMD-exp}'s
higher RMSE at $\delta=0$ are several standard errors from zero, while
small-break differences remain modest.

The comparison also clarifies the connection to \textrm{aGAS}: the
exponential and logistic links yield \textrm{DMD-exp} and \textrm{DMD-logit},
while \textrm{MD-logit} removes persistence. Persistence is most helpful for
larger breaks, and bounded geometry improves performance in most cells:
\textrm{DMD-logit} has lower RMSE than \textrm{DMD-exp} in 18 of 20 cells,
including all eight large-break cells. Both exceptions involve small breaks, and
the two methods' RMSEs differ by less than $0.0003$. Section~\ref{sec:empirical}
examines this stability trade-off out of sample.

The experiment illustrates how the two regularisation mechanisms shape
adaptation: the admissible interval limits responses to breaks and isolated
shocks, while persistence governs reversion towards the reference gain. The
regret theorem complements this filter-level comparison by evaluating gain
adaptation along a fixed score-conditioned path.

\section{Empirical illustration}\label{sec:empirical}

We study bounded gain learning in volatility applications. We revisit the
in-sample design of Blasques~et~al.~\cite{blasques2019accelerating} with
full-sample maximum likelihood and BIC, then evaluate expanding-window
one-step-ahead density forecasts across twelve markets.

\subsection{In-sample return-density illustration}\label{sec:emp_insample}
\paragraph{\textbf{Data and model}}
We use Thai baht/US dollar (FRED \texttt{DEXTHUS}, 1985--2017), S\&P~500
(Yahoo \texttt{\textasciicircum GSPC}, 2004--2024), and euro/US dollar
(Yahoo \texttt{EURUSD=X}) returns.\footnote{Equity-index and exchange-rate price data are from Yahoo Finance (\url{https://finance.yahoo.com}); the Thai-baht/US-dollar rate is series \texttt{DEXTHUS} from FRED, Federal Reserve Bank of St.\ Louis (\url{https://fred.stlouisfed.org}). All series are publicly available.}
The Thai baht case revisits the July 1997 peg break in
Blasques~et~al.~\cite{blasques2019accelerating}. Returns are demeaned daily log
changes multiplied by 100. We model the log scale-squared state as
$h_{t+1}=\omega_\lambda+\beta_\lambda h_t+\alpha_t s_t$, with
$0<\beta_\lambda<1$. We fit Gaussian and Student-$t$ observation densities;
under Student-$t$, the score driver is bounded and saturating. The gain loss is
convex under this parametrisation. The comparison includes the constant-gain
filter; the unrestricted exponential-link \textrm{aGAS} benchmark, labelled
\textrm{DMD-exp} in the frozen outputs; the persistent bounded rules
\textrm{DMD-logit} and \textrm{DMD-proj}; the memoryless \textrm{MD-logit} rule;
and \textrm{AdaGrad}. The empirical exponential score-product coefficient
$a_f$ is unrestricted: positive values have the positive-step
\textrm{DMD} interpretation, zero removes score-product learning, and negative
values reverse the feedback direction and lie outside that interpretation.
The bounded-rule feedback coefficients are constrained directly to be
non-negative; a zero estimate is a nested no-feedback boundary fit and does
not satisfy the regret theorems' strict positive-step condition. The bounded
rules limit large gain spikes, while the exponential benchmark is numerically
capped.

After $y_t$ is observed, the score driving the update to $h_{t+1}$ is
$s_t=\partial\log p(y_t\mid h_t)/\partial h_t$. Under the Gaussian specification,
$y_t\mid\mathcal F_{t-1}\sim\mathcal N(0,\mathrm e^{h_t})$. Under the
Student-$t_{\nu_{\mathrm{df}}}$ specification, $\mathrm e^{h_t}$ is the scale
squared and the conditional variance is
$\mathrm e^{h_t}\nu_{\mathrm{df}}/(\nu_{\mathrm{df}}-2)$, for
$\nu_{\mathrm{df}}>2$.
In \textrm{DMD-exp}, the newly observed $y_t$ enters the $f_{t+1}$ recursion
through the consecutive-score product, with $\alpha_t=\mathrm e^{f_{t+1}}$.
Each bounded interval contains the fitted constant gain and discounted-logistic
path, with a small margin; its lower bound keeps score feedback active and its
upper bound limits extreme transient gains.

\paragraph{\textbf{In-sample results}}
Table~\ref{tab:emp_insample} reports the BIC comparison. Adaptive gains attain
the lowest BIC twice, both for Thai baht: \textrm{MD-logit}
under Gaussian and \textrm{DMD-logit} under Student-$t$. \textrm{MD-logit} lowers
the Gaussian BIC by $77$ points relative to the constant gain. The constant gain
is selected in all four S\&P~500 and euro/US-dollar columns; \textrm{DMD-exp} is
never selected. Student-$t$ has lower BIC than Gaussian for every series.
Overall, evidence for gain adaptation is mixed.

\begin{table}[!ht]
\centering
\small
\resizebox{\textwidth}{!}{\input{tables/tab-emp-insample.tex}}
\caption{\textbf{In-sample Bayesian information criterion by method, series, and observation density.}
Lower values are preferred; $N$ denotes Gaussian and $t$ Student-$t$. Bold marks the lowest BIC in each column.
Gaussian specifications estimate four parameters for the constant gain, five for \textrm{MD-logit} and \textrm{AdaGrad}, and six for \textrm{DMD-logit}, \textrm{DMD-proj}, and the nominally unbounded \textrm{DMD-exp}.
Student-$t$ specifications add the degrees-of-freedom parameter $\nu_{\mathrm{df}}$. Adaptive gains attain the lowest BIC in two of the six columns.}
\label{tab:emp_insample}
\end{table}

\paragraph{\textbf{The learned gain}}
Figure~\ref{fig:emp_thb_gain} compares \textrm{DMD-exp} and \textrm{DMD-logit}
for the Thai baht. At the largest shocks, the nominally unbounded exponential-link
\textrm{aGAS} gain rises by more than two orders of magnitude, reaching~$55$. The
bounded \textrm{DMD-logit} gain responds more moderately and remains within its
admissible interval, $[0.15,0.52]$. When consecutive scores have the same sign,
their positive product raises the gain, as implied by the gradient identity in
Section~\ref{sec::SectionTwo}, while the logistic mirror geometry bounds the
response.

\begin{figure}[t]\centering
\includegraphics[width=0.85\textwidth]{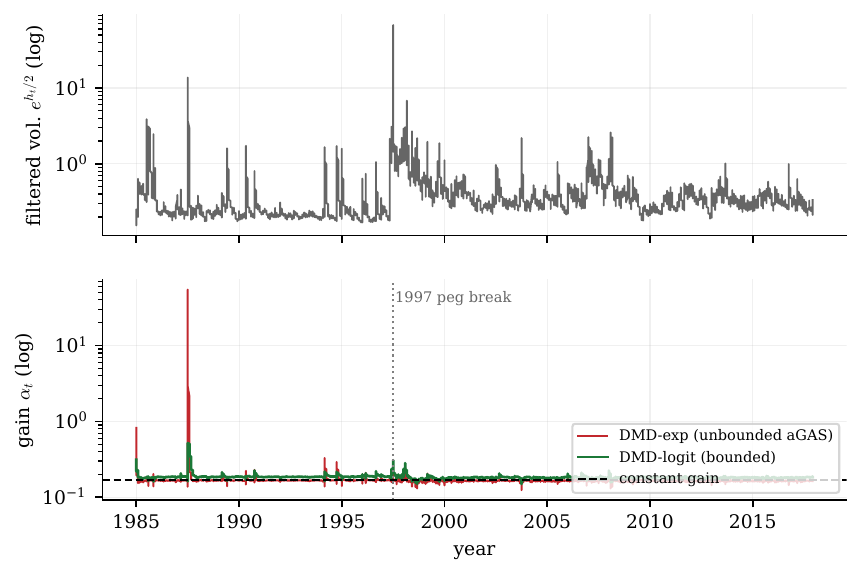}
\caption{\textbf{Thai baht/US dollar, 1985--2017.} Top: filtered volatility on a
logarithmic scale. Bottom: gains on a logarithmic scale. The \emph{nominally
unbounded} exponential-link \textrm{aGAS} gain (\textrm{DMD-exp}) reaches~$55$ at
the largest shocks; bounded \textrm{DMD-logit} remains within $[0.15,0.52]$, and
the constant gain is dashed.}
\label{fig:emp_thb_gain}
\end{figure}

\FloatBarrier
\subsection{Out-of-sample realised-variance panel}\label{sec:emp_oos}
We evaluate twelve equity indices through 2024. Eleven samples start in 2000,
whereas the EURO STOXX 50 sample starts in 2007. Table~\ref{tab:emp_panel_data}
lists the market symbols, countries, sample periods, observation counts, and
evaluation lengths.
The target is the logarithm of the Parkinson realised-variance
proxy~\cite{parkinson1980extreme}, modelled as a Gaussian score-driven level. We
assess expanding-window one-step-ahead forecasts by the Gaussian negative log score
on the log-variance scale and by QLIKE~\cite{patton2011volatility} on the variance
scale; inference uses Diebold--Mariano tests~\cite{diebold1995comparing} and Model
Confidence Sets~\cite{hansen2011model}. For each market, the bounded-rule
interval is selected from the initial training window and held fixed thereafter.
All rules use the same score formula and scaling convention. \ref{app::estimation}
details the model, expanding-window and refitting schemes, HAC standardisation,
and Model Confidence Set implementation.

Figure~\ref{fig:emp_sp500_gain} shows the bounded gain paths for the S\&P~500.
Both learned gains rise around the 2008 financial crisis and the 2020 pandemic
shock while remaining within the fixed admissible interval.

\begin{figure}[t]\centering
\includegraphics[width=0.85\textwidth]{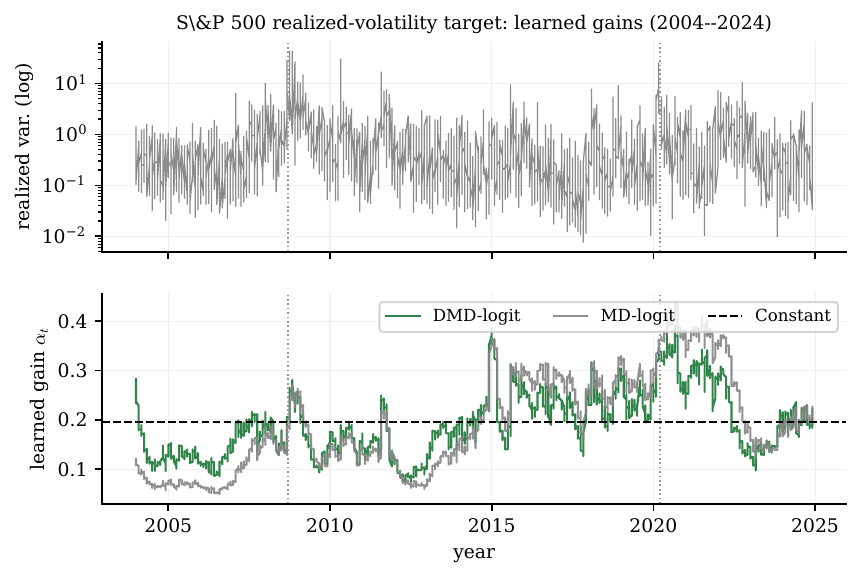}
\caption{\textbf{S\&P~500 log realised-variance proxy and learned gains, 2004--2024.}
Top: log realised-variance proxy. Bottom: gains learned by the
bounded \textrm{DMD-logit} and \textrm{MD-logit} rules, with the constant gain shown
as a dashed line. Dotted lines mark the 2008 financial crisis and the 2020 pandemic shock.}
\label{fig:emp_sp500_gain}
\end{figure}

In the daily-return analysis, the constant-gain specification remains in every
Model Confidence Set, so we focus on realised variance. An adaptive rule attains
the lowest mean negative log score in all twelve markets. \textrm{DMD-logit} has
a lower mean negative log score than the constant gain in eleven markets; DAX is
the sole exception, and six one-sided Diebold--Mariano statistics fall below
$-1.64$. The constant gain is excluded from the $90\%$ Model Confidence Set in
four markets under the negative log score and three under QLIKE.
Table~\ref{tab:emp_oos_panel} and Figure~\ref{fig:emp_panel} report the full results.

\begin{table}[t]
\centering
\small
\resizebox{\textwidth}{!}{%
\begin{tabular}{lrrrlrrc}
\toprule
Market & Constant & \textrm{DMD-logit} & \textrm{DMD-exp} & Best & DM & $p$ & MCS \\
\midrule
SPX     & 1.27654 & 1.27108 & 1.27195  & \textrm{DMD-logit} & $-2.27$ & 0.012 & in  \\
NDX     & 1.24401 & 1.24106 & 1.23992  & \textrm{DMD-exp}   & $-1.19$ & 0.117 & in  \\
DJI     & 1.25435 & 1.24925 & 1.25032  & \textrm{DMD-logit} & $-2.72$ & 0.003 & out \\
FTSE    & 1.19763 & 1.19705 & 1.19742  & \textrm{DMD-logit} & $-0.61$ & 0.272 & in  \\
DAX     & 1.25327 & 1.25352 & 1.25278  & \textrm{DMD-exp}   & $+0.17$ & 0.568 & in  \\
CAC     & 1.22874 & 1.22666 & 1.22672  & \textrm{DMD-proj}  & $-1.38$ & 0.083 & out \\
STOXX50 & 1.47414 & 1.45935 & 1.46058  & \textrm{DMD-logit} & $-4.41$ & $<0.001$ & out \\
NIKKEI  & 1.31670 & 1.31415 & 1.31527  & \textrm{DMD-logit} & $-1.25$ & 0.105 & in  \\
HSI     & 1.18561 & 1.18222 & 1.18325  & \textrm{DMD-logit} & $-1.70$ & 0.045 & in  \\
TSX     & 1.63699 & 1.63699 & 14.75468 & \textrm{MD-logit}  & $-1.97$ & 0.025 & in  \\
ASX     & 1.36643 & 1.33249 & 1.33913  & \textrm{DMD-logit} & $-7.78$ & $<0.001$ & out \\
BVSP    & 1.16250 & 1.16169 & 1.16197  & \textrm{DMD-logit} & $-0.78$ & 0.218 & in  \\
\bottomrule
\end{tabular}}
\caption{\textbf{Out-of-sample realised-variance results under the negative log score.}
Columns report the per-market mean predictive loss for the constant gain,
\textrm{DMD-logit}, and \textrm{DMD-exp}; the best adaptive method; the one-sided
Diebold--Mariano statistic and $p$-value for \textrm{DMD-logit} against the constant
gain; and whether the constant gain belongs to the $90\%$ Model Confidence Set.
Negative DM values favour \textrm{DMD-logit}. \textrm{DMD-proj} and \textrm{MD-logit}
are omitted from the loss columns; they are the best adaptive methods for CAC
($1.22466$) and TSX ($1.63698$), respectively. The unusually high TSX loss for
\textrm{DMD-exp} ($14.75468$) reflects the transient gain spikes discussed in
Section~\ref{sec:emp_bvu}.}
\label{tab:emp_oos_panel}
\end{table}

\paragraph{\textbf{Panel-level evidence}} Because the per-market $p$-values are
unadjusted, counts of individually significant markets are descriptive. Across the
seven methods, \textrm{DMD-logit} has the best average rank under both the negative
log score ($1.58$) and QLIKE ($1.75$), compared with $4.33$ and $4.00$ for the
constant gain. The log-HAR benchmark \cite{corsi2009}, fitted and evaluated over the
same windows, ranks last ($6.08$ and $6.17$); \textrm{DMD-logit} outperforms it in
every market under both criteria. To preserve cross-market and serial dependence,
we resample the $D=5480$ pooled evaluation dates in blocks of length $18$, using
$B=5000$ bootstrap replications. The pooled \textrm{DMD-logit}-minus-constant
difference in negative log score is $-0.0057$ ($95\%$ confidence interval
$[-0.0076,-0.0041]$; two-sided $p<0.001$); market-level differences are negative
in eleven of the twelve markets.

\FloatBarrier
\subsection{Bounded versus unbounded gain behaviour}\label{sec:emp_bvu}
The panel also reveals how bounded and nominally unbounded gain dynamics differ.
Averaged across markets, the time-series standard deviation of the learned gain is
$0.71$ for \textrm{DMD-exp} and $0.046$ for \textrm{DMD-logit}, a ratio of about
sixteen. For the TSX, \textrm{DMD-exp} reaches its numerical ceiling on $2.1\%$ of
out-of-sample days. Its median daily negative log score ($1.30$) is close to the
constant-gain value ($1.29$), yet the $3.6\%$ of losses exceeding five account for
$91\%$ of total loss and raise the mean to $14.75$. This concentration of loss
illustrates the transient-spike risk of the capped \textrm{DMD-exp} implementation.

The pilot-selected interval for the bounded TSX rules is narrow
(Table~\ref{tab:emp_gain_intervals}), so their gain paths remain close to the
constant-gain benchmark. The TSX results therefore mainly reflect the
transient-spike risk of \textrm{DMD-exp}. Across the twelve markets,
\textrm{DMD-logit} records a lower negative log score than \textrm{DMD-exp} in ten
and a lower QLIKE loss in nine.

\begin{figure}[t]\centering
\includegraphics[width=0.95\textwidth]{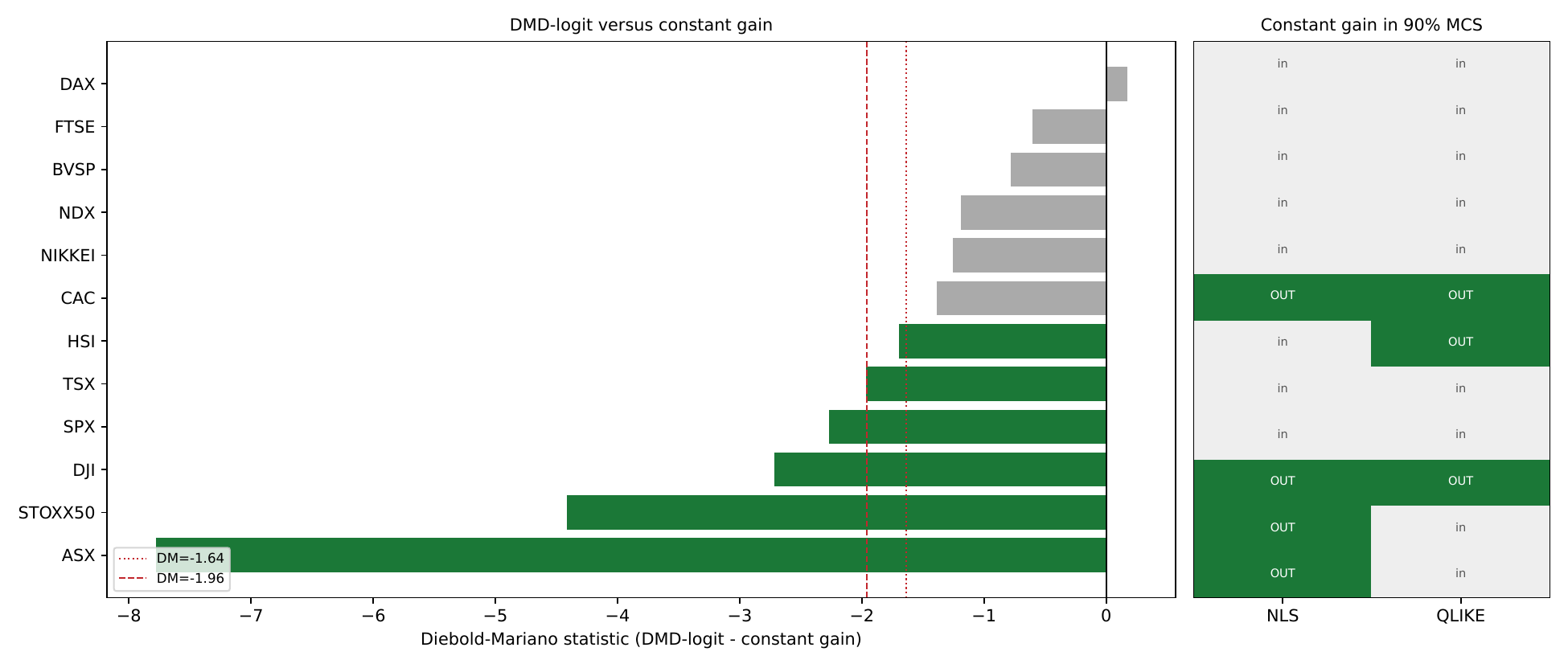}
\caption{\textbf{Out-of-sample realised-variance evidence across twelve equity indices.}
Left: per-market Diebold--Mariano statistics comparing \textrm{DMD-logit} with the
constant gain under the negative log score. Negative values favour \textrm{DMD-logit};
vertical lines mark the one-sided $10\%$ ($-1.64$) and $5\%$ ($-1.96$) critical values.
Right: membership of the constant gain in the $90\%$ Model Confidence Set under the
negative log score (NLS) and QLIKE; ``OUT'' denotes exclusion. Samples end in 2024;
eleven begin in 2000 and the EURO STOXX~50 begins in 2007.}
\label{fig:emp_panel}
\end{figure}

\FloatBarrier
\section{Conclusion}\label{sec::Conclusion}
We study online learning of the observation-driven scale parameter in score-driven
filters, which we call the gain. Given the realised score and scaling rule, each
admissible gain selects a reachable next state and its associated predictive
density. In the scalar unscaled case, the negative product of consecutive scores
gives the stochastic gradient of the next-period gain loss. Gain links induce
mirror geometry, while persistence acts as a Bregman pull towards a reference gain.

Under the compactness, convexity, integrability, and schedule conditions in
Section~\ref{sec::SectionFive}, projected or localised mirror updates satisfy
dynamic-regret bounds against \(\mathcal F_t\)-measurable moving gains. Under the
calendar convention \(\alpha_{\lambda,t}=g(f_{t+1})\), the
Blasques--Gorgi--Koopman recursion admits an exact discounted-MD representation
when \(\beta_f<1\), and also at the unit-persistence endpoint when
\(\omega_f=0\). Applying these bounds to the data-dependent learning rates of
the original aGAS recursion requires the additional conditions in
Remark~\ref{rem::agas}.

The guarantee compares candidate gains on the same realised reachable set at each
date. Complete counterfactual filter policies lie outside its scope, and reported
gain magnitudes depend on score scaling. The recursive empirical filters therefore
provide complementary evidence. In the twelve-market panel, \textrm{DMD-logit} has
a lower mean negative log score than the constant gain in eleven markets and avoids
the extreme transients of \textrm{DMD-exp}, although market-level statistical
evidence remains mixed.

\clearpage
\appendix

\section{Proofs}\label{app:proofs}
\setcounter{equation}{0}

\subsection{Proofs for Section~\ref{sec::SectionOne}: the conditional gain problem}
\label{app:proofs_conditional_gain}

\subsubsection{Proof of Proposition~\ref{prop::kl_representation}}\label{app::proof_kl_representation}
\begin{proof}
Conditional on \(\mathcal F_t\), \(y_{t+1}\) has density \(\tilde p_{t+1}\) under
\(\widetilde Q_{t+1}\). Therefore, Eq.~\eqref{eq::scoringruleexpected} gives
\[
    J_{t+1\mid t}(\alpha)
    =
    -\int_{\mathcal Y}\log p(y\mid\hat\lambda_t(\alpha))\,\tilde p_{t+1}(y)\,\nu(\ud y).
\]
Adding and subtracting
\(\int_{\mathcal Y}\tilde p_{t+1}\log\tilde p_{t+1}\,\nu(\ud y)=-C_{t+1}\),
which is finite almost surely by assumption, gives
\[
    J_{t+1\mid t}(\alpha)
    =
    -\int_{\mathcal Y}\tilde p_{t+1}\log\tilde p_{t+1}\,\nu(\ud y)
    +
    \int_{\mathcal Y}\tilde p_{t+1}\log\frac{\tilde p_{t+1}}{p(\cdot\mid\hat\lambda_t(\alpha))}\,\nu(\ud y)
    =
    C_{t+1}+\mathsf{KL}\big(\widetilde Q_{t+1},P_{\hat\lambda_t(\alpha)}\big).
\]
The rearrangement is valid because \(C_{t+1}\) and the cross-entropy are both well
defined by hypothesis. We use the conventions \(0\log(0/q)=0\) and
\(q\log(q/0)=+\infty\) for \(q>0\), covering points at which
\(\tilde p_{t+1}=0\) or \(p(\cdot\mid\hat\lambda_t(\alpha))=0\).
\end{proof}

\subsection{Proofs for Section~\ref{sec::SectionTwo}: gain-gradient identities}
\label{app:proofs_gain_gradient}

\subsubsection{Proof of Proposition~\ref{prop::proposition_1}}\label{app::proof_proposition_1}
\begin{proof}
Fix \(\alpha\in\inter(\mathcal A)\) and choose a ball centred at \(\alpha\) whose
closure is contained in \(U_\alpha\). Assumption~\ref{ass::local_condition}(ii)
bounds the coordinatewise difference quotients on this ball by the conditionally
integrable variable \(M_{t+1,\alpha}\), and
Assumption~\ref{ass::local_condition}(iii) ensures that \(L_{t+1}(\alpha)\) is
conditionally integrable. Conditional dominated convergence therefore gives
\[
    \nabla_\alpha J_{t+1\mid t}(\alpha)
    =
    \nabla_\alpha \mathsf E_t[L_{t+1}(\alpha)]
    =
    \mathsf E_t[\nabla_\alpha L_{t+1}(\alpha)].
\]
Using the pathwise identity \eqref{eq::gradientlogloss}, this gives
\[
    \nabla_\alpha J_{t+1\mid t}(\alpha)
    =
    \mathsf E_t[G_{t+1}(\alpha)]
    =
    -
    \mathsf E_t[
        H_t^{\top}s_{t+1}(\alpha)
    ],
\]
where \(H_t\) is the gain-loading matrix fixed at date \(t\). This proves
\eqref{eq::nabla_gradient}. For the uniform statement, choose a nested compact
convex exhaustion \((\mathcal K_n)_{n\geq1}\) of \(\inter(\mathcal A)\), with
\(\mathcal K_n\subset\inter(\mathcal K_{n+1})\). On \(\mathcal K_{n+1}\),
Assumption~\ref{ass::global_condition}, base-point integrability, and the gradient
envelope ensure that \(L_{t+1}(u)\) is
conditionally integrable for every \(u\). Parameterised conditional expectation
and conditional dominated differentiation then yield a jointly measurable version
on \(\Omega\times\mathcal K_n\),
\[
    \Gamma_{t,n}(\omega,u)
    :=\mathsf E_t[G_{t+1}(u)](\omega),
    \qquad (\omega,u)\in\Omega\times\mathcal K_n,
\]
which is continuous in \(u\) and equals \(\nabla_\alpha J_{t+1\mid t}(u)\) for
every \(u\in\mathcal K_n\), outside a common null set. On each overlap, the two
versions agree almost surely at every point of a countable dense subset. Continuity
extends this equality to the whole overlap. Since the exhaustion is countable, the
exceptional sets can be combined into a single null set. Outside this set, the
versions patch together to define a continuous \(\Gamma_t\) on
\(\inter(\mathcal A)\), proving the uniform identity.

On \(\mathcal A_\varepsilon\), joint measurability of \(\Gamma_t\) makes
\(\Gamma_t(\cdot,\alpha_t)\) measurable for every \(\mathcal F_t\)-measurable
\(\alpha_t\in\mathcal A_\varepsilon\), while the conditional envelope ensures
integrability. Therefore,
\[
    \nabla_\alpha J_{t+1\mid t}(\alpha_t)
    =
    \Gamma_t(\cdot,\alpha_t)
    =
    \mathsf E_t[G_{t+1}(\alpha_t)]
    =
    -\mathsf E_t[
        H_t^{\top}s(y_{t+1},\hat\lambda_t(\alpha_t))
    ].
\]
Substituting \(\lambda_{t+1}=\hat\lambda_t(\alpha_t)\) establishes
Eq.~\eqref{eq::nabla_gradient_2}.
\end{proof}

\subsubsection{Proof of Corollary~\ref{cor::corollary_1}}\label{app::proof_predictable_gradient}
\begin{proof}
Fix \(\alpha\). Repeating the preceding difference-quotient argument with
\(\mathsf E_{t-1}\) in place of \(\mathsf E_t\) gives
\[
    \nabla_\alpha\bar J_{t\mid t-1}(\alpha)
    =
    \mathsf E_{t-1}[G_{t+1}(\alpha)]
    =
    -\mathsf E_{t-1}[H_t^\top s_{t+1}(\alpha)].
\]
The final equality follows from Eq.~\eqref{eq::gradientlogloss}, establishing
Eq.~\eqref{eq::decision_time_gradient}.

Under the uniform condition, the same compact-exhaustion argument, now conditional
on \(\mathcal F_{t-1}\), yields a jointly measurable version
\(\bar\Gamma_{t-1}\) that is continuous in \(u\):
\[
    \bar\Gamma_{t-1}(\omega,u)
    :=
    \mathsf E_{t-1}[G_{t+1}(u)](\omega)
    =\nabla_\alpha\bar J_{t\mid t-1}(u),
    \qquad (\omega,u)\in\Omega\times\inter(\mathcal A).
\]
Therefore, for any \(\mathcal F_{t-1}\)-measurable
\(\alpha_t\in\mathcal A_\varepsilon\),
\[
    \nabla_\alpha\bar J_{t\mid t-1}(\alpha_t)
    =
    \mathsf E_{t-1}[G_{t+1}(\alpha_t)]
    =
    -\mathsf E_{t-1}[
        H_t^\top s(y_{t+1},\hat\lambda_t(\alpha_t))
    ].
\]
Substituting \(\lambda_{t+1}=\hat\lambda_t(\alpha_t)\) establishes
Eq.~\eqref{eq::decision_time_gradient_policy}.
\end{proof}

\subsection{Proofs for Section~\ref{sec::SectionThree}: mirror geometry}
\label{app:proofs_mirror_geometry}

\subsubsection{Proof of Proposition~\ref{prop::discounted_md}}\label{app::proof_discounted_md}
\begin{proof}
    Since $(1-\rho_t)+\rho_t=1$, collecting the two Bregman terms in Eq.~\eqref{eq::objective_gas} leaves $\Psi_g(\alpha)$ plus an affine function of $\alpha$. The objective is therefore strictly convex on $\mathcal I$, with gradient
    \begin{equation*}
        \eta_t \xi_t + (1-\rho_t)\{\nabla \Psi_g(\alpha)-\nabla \Psi_g(\bar \alpha)\} + \rho_t \{\nabla \Psi_g(\alpha)-\nabla \Psi_g(\alpha_t)\}.
    \end{equation*}
Since $\nabla \Psi_g=g^{-1}$, the first-order condition becomes
\begin{equation*}
    g^{-1}(\alpha)=(1-\rho_t)\bar \vartheta + \rho_t \vartheta_t - \eta_t\xi_t.
\end{equation*}
The right-hand side lies in $\mathbb R^k$, and applying $g$ yields $\alpha_{t+1}\in\mathcal I$. Strict convexity establishes that $\alpha_{t+1}$ is the unique minimiser.
\end{proof}

\subsection{Proofs for Section~\ref{sec::SectionFive}: regret bounds}
\label{app:proofs_regret}

\subsubsection{Proof of Proposition~\ref{prop::convexity_transfer}}\label{app::proof_convexity_transfer}
\begin{proof}
Fix $\alpha,\beta\in\mathcal A$ and $\tau\in[0,1]$. Convexity of $\mathcal A$ gives $\tau\alpha+(1-\tau)\beta\in\mathcal A$, while affineness of the state map gives
\begin{equation}\label{eq:convexity_transfer_affine}
    b_t + H_t(\tau \alpha + (1-\tau) \beta)
    = \tau (b_t + H_t\alpha) + (1-\tau) (b_t + H_t\beta).
\end{equation}
Thus $\mathcal R_t$, the affine image of $\mathcal A$, is convex.

For items (1) and (2), convexity of $\ell_{t+1}^{\mathrm{st}}$ or $R_{t+1}$ on $\mathcal R_t$, together with Eq.~\eqref{eq:convexity_transfer_affine}, establishes convexity of $L_{t+1}$ or $J_{t+1\mid t}$ on $\mathcal A$.

For item (3), $\alpha\ne\beta$ and $H_t(\alpha-\beta)\ne0$ make the two states on the right-hand side of Eq.~\eqref{eq:convexity_transfer_affine} distinct. For $\tau\in(0,1)$, strict convexity of the relevant state criterion then yields the strict-convexity inequality for the corresponding gain criterion.

For item (4), apply the strong-convexity inequality to $u=b_t+H_t\alpha$ and $v=b_t+H_t\beta$. Combining it with Eq.~\eqref{eq:convexity_transfer_affine} and $u-v=H_t(\alpha-\beta)$ yields Eq.~\eqref{eq::inequality}. If $\sigma_{\min}(H_t)^2\ge c_t>0$, then
\[
    \|H_t(\alpha-\beta)\|^2
    \ge \sigma_{\min}(H_t)^2\|\alpha-\beta\|^2
    \ge c_t\|\alpha-\beta\|^2.
\]
Hence $L_{t+1}$ is $\mu c_t$-strongly convex on $\mathcal A$. The same argument, with $R_{t+1}$ in place of $\ell_{t+1}^{\mathrm{st}}$, establishes $\mu c_t$-strong convexity of $J_{t+1\mid t}$.
\end{proof}

\subsubsection{Proof of Theorem~\ref{thm:dynamic_regret_biom}}\label{app::Appendix}
\begin{proof}
Throughout, $\Psi$ is $\sigma_\Psi$-strongly convex on $\mathcal K$, the step
sizes $(\eta_t)_{t=1}^{N}$ are positive and non-increasing, and the iterates and
inputs are those of Theorem~\ref{thm:dynamic_regret_biom}.

\medskip
\noindent\emph{Step 1: a purely geometric one-step mirror inequality.} The inequality below uses only the mirror geometry and allows an arbitrary input $\xi\in\mathbb{R}^{k}$. For $\eta>0$ and $\beta\in\mathcal K$, set
\[
\alpha^+=\argmin_{\alpha\in\mathcal K}
\{\eta\langle\xi,\alpha\rangle+\mathcal B_\Psi(\alpha,\beta)\}.
\]
Then, for every $u\in\mathcal K$,
\begin{equation}\label{eq::mirror_three_point}
    \langle \xi, \beta - u\rangle \leq \frac{\mathcal{B}_\Psi(u,\beta)-\mathcal{B}_\Psi(u,\alpha^{+})}{\eta}
  +
  \frac{\eta\|\xi\|_*^2}{2\sigma_\Psi}.
\end{equation}
Equation~\eqref{eq::mirror_three_point} is the standard three-point mirror inequality. The first-order condition for the displayed minimisation is
\[
\big\langle
\eta\xi+\nabla\Psi(\alpha^+)-\nabla\Psi(\beta),
u-\alpha^+
\big\rangle\ge0,
\qquad u\in\mathcal K.
\]
Combining this condition with the Bregman three-point identity gives Eq.~\eqref{eq::mirror_three_point} after applying Fenchel--Young and strong convexity to bound
\[
\langle\xi,\beta-\alpha^+\rangle
\le
\frac{\eta\|\xi\|_*^2}{2\sigma_\Psi}
+\frac{1}{\eta}\mathcal B_\Psi(\alpha^+,\beta).
\]

\medskip
\noindent\emph{Step 2: linearisation by convexity.} Fix $t\in\{1,\ldots,N\}$. Since $\xi_{t}\in\partial\ell_t(\alpha_t)$, convexity gives
\begin{equation}\label{eq::lin_step}
\ell_t(\alpha_t)-\ell_t(\gamma_t)
  \le \langle \xi_{t}, \alpha_t-\gamma_t\rangle.
\end{equation}

\medskip
\noindent\emph{Step 3: applying the mirror inequality at the update pair.} By the update definition, Eq.~\eqref{eq::mirror_three_point} applies with $\eta=\eta_{t}$, $\beta=\alpha_{t}$, $\xi=\xi_{t}$, $\alpha^{+}=\alpha_{t+1}$, and $u=\gamma_t$. Combining it with Eq.~\eqref{eq::lin_step} yields
\begin{equation}\label{eq::mirror_applied}
\ell_t(\alpha_t)-\ell_t(\gamma_t)
  \le
  \frac{
    \mathcal{B}_{\Psi}(\gamma_t,\alpha_{t})
    -
    \mathcal{B}_{\Psi}(\gamma_t,\alpha_{t+1})
  }{\eta_{t}}
  +
  \frac{\eta_{t}\|\xi_{t}\|_*^2}{2\sigma_\Psi}.
\end{equation}
\medskip
\noindent\emph{Step 4: summation and telescoping.} Summing Eq.~\eqref{eq::mirror_applied} over $t=1,\ldots,N$ produces the gradient-energy term
\begin{equation*}
    \frac{1}{2\sigma_\Psi}\sum_{t=1}^N\eta_t\|\xi_t\|_*^2.
\end{equation*}
Define the remaining Bregman sum as
\begin{equation*}
    \mathcal S_N:=\sum_{t=1}^{N}\frac{
    \mathcal{B}_{\Psi}(\gamma_t,\alpha_{t})
    -
    \mathcal{B}_{\Psi}(\gamma_t,\alpha_{t+1})
  }{\eta_{t}}.
\end{equation*}
Set $w_t:=1/\eta_{t}$, which is non-decreasing because $(\eta_t)$ is non-increasing. Reindexing the subtracted terms in $\mathcal S_N$ and discarding the final contribution $-w_N\mathcal{B}_\Psi(\gamma_N,\alpha_{N+1})\le0$ gives
\begin{equation*}
    \mathcal S_N
    \le
    w_1\mathcal{B}_\Psi(\gamma_1,\alpha_1)
    +
    \sum_{t=2}^{N}
    \Big\{
        w_{t-1}\big[\mathcal{B}_\Psi(\gamma_t,\alpha_{t})-\mathcal{B}_\Psi(\gamma_{t-1},\alpha_{t})\big]
        +
        (w_t-w_{t-1})\,\mathcal{B}_\Psi(\gamma_t,\alpha_{t})
    \Big\}.
\end{equation*}
By the definition in Eq.~\eqref{eq::bregmanincrements},
\(
    \mathcal{B}_\Psi(\gamma_t,\alpha_{t})-\mathcal{B}_\Psi(\gamma_{t-1},\alpha_{t})
    \le
    \Delta_{\Psi,\mathcal K}(\gamma_t,\gamma_{t-1})
\).
Evaluating the same supremum at $\alpha=v$ also gives
\(
    \Delta_{\Psi,\mathcal K}(u,v)
    \ge \mathcal B_\Psi(u,v)\ge0
\).
Since $w_t-w_{t-1}\ge0$, Eq.~\eqref{eq::Omega} bounds the corresponding Bregman terms by $\Omega_{\Psi,\mathcal K}$, and their coefficients telescope:
\(
    w_1\Omega_{\Psi,\mathcal K}+\sum_{t=2}^{N}(w_t-w_{t-1})\Omega_{\Psi,\mathcal K}=w_N\Omega_{\Psi,\mathcal K}=\Omega_{\Psi,\mathcal K}/\eta_{N}.
\)
Moreover, $\Delta_{\Psi,\mathcal K}(\gamma_t,\gamma_{t-1})\ge0$ and $w_{t-1}=1/\eta_{t-1}\le 1/\eta_t$, so we obtain
\begin{equation*}
    \mathcal S_N\le
    \frac{\Omega_{\Psi,\mathcal K}}{\eta_{N}}
    +\sum_{t=2}^{N}\frac{\Delta_{\Psi,\mathcal K}(\gamma_t,\gamma_{t-1})}{\eta_{t}}.
\end{equation*}
Combining the energy term with this bound on $\mathcal S_N$ yields exactly the right-hand side of Eq.~\eqref{eq:dynamic_regret_bound_biom}, which completes the proof.
\end{proof}

\subsubsection{Proof of Corollary~\ref{cor::corollary}}\label{app::proof_expected_regret}
\begin{proof}
Apply Theorem~\ref{thm:dynamic_regret_biom} pathwise with $\ell_t=L_{t+1}$.
Because $\alpha_t$ and $\gamma_t$ are $\mathcal F_t$-measurable, the tower
property gives
\(\mathsf E[L_{t+1}(\alpha_t)]=\mathsf E[J_{t+1\mid t}(\alpha_t)]\) and
\(\mathsf E[L_{t+1}(\gamma_t)]=\mathsf E[J_{t+1\mid t}(\gamma_t)]\).
Since $\Omega_{\Psi,\mathcal K}$ and $\sigma_\Psi$ are deterministic constants,
taking expectations in the pathwise bound yields the stated result.
\end{proof}

\subsubsection{Proof of Theorem~\ref{thm:dynamic_regret_dmd}}\label{app::AppendixDMD}
\begin{proof}
Write $\mathcal{B}:=\mathcal{B}_\Psi$. Throughout, $\Psi$ is
$\sigma_\Psi$-strongly convex on $\mathcal{K}$,
$\Omega_{\Psi,\mathcal K}<\infty$, and all iterates, the reference
$\bar\alpha$, and the comparator path lie in $\mathcal{K}$. We adapt the four
steps from the proof of Theorem~\ref{thm:dynamic_regret_biom} to the two-centre
(anchored) proximal update in Eq.~\eqref{eq::objective_gas}.

\medskip
\noindent\emph{Step 1: two-centre variational inequality.} Strict convexity
makes $\alpha_{t+1}$ the unique minimizer. The first-order
condition---Proposition~\ref{prop::discounted_md} in the interior case, or its
normal-cone counterpart under projection---implies, for every
$u\in\mathcal{K}$,
\begin{equation}\label{eq::dmd_VI}
    \big\langle\,\eta_t\xi_t+(1-\rho_t)[\nabla\Psi(\alpha_{t+1})-\nabla\Psi(\bar\alpha)]
    +\rho_t[\nabla\Psi(\alpha_{t+1})-\nabla\Psi(\alpha_t)],\;u-\alpha_{t+1}\,\big\rangle\ge0.
\end{equation}
Apply the Bregman three-point identity
$\langle\nabla\Psi(a)-\nabla\Psi(b),u-a\rangle=\mathcal{B}(u,b)-\mathcal{B}(u,a)-\mathcal{B}(a,b)$
to Eq.~\eqref{eq::dmd_VI}, using $(a,b)=(\alpha_{t+1},\alpha_t)$ for the
memory term and $(a,b)=(\alpha_{t+1},\bar\alpha)$ for the anchor term. Setting
$u=\gamma_t$ and rearranging gives
\begin{align}\label{eq::dmd_threepoint}
    \eta_t\langle\xi_t,\alpha_{t+1}-\gamma_t\rangle
    \le\;&\rho_t\big[\mathcal{B}(\gamma_t,\alpha_t)-\mathcal{B}(\gamma_t,\alpha_{t+1})-\mathcal{B}(\alpha_{t+1},\alpha_t)\big]\nonumber\\
    &+(1-\rho_t)\big[\mathcal{B}(\gamma_t,\bar\alpha)-\mathcal{B}(\gamma_t,\alpha_{t+1})-\mathcal{B}(\alpha_{t+1},\bar\alpha)\big].
\end{align}

\medskip
\noindent\emph{Step 2: linearisation by convexity.} Since $\ell_t$ is convex and $\xi_t\in\partial\ell_t(\alpha_t)$,
\begin{equation}\label{eq::dmd_lin}
    \ell_t(\alpha_t)-\ell_t(\gamma_t)\le\langle\xi_t,\alpha_t-\gamma_t\rangle
    =\langle\xi_t,\alpha_{t+1}-\gamma_t\rangle+\langle\xi_t,\alpha_t-\alpha_{t+1}\rangle.
\end{equation}

\medskip
\noindent\emph{Step 3: the composite one-step inequality.} Fenchel--Young and
the strong-convexity bound
$\mathcal{B}(\alpha_{t+1},\alpha_t)\ge\tfrac{\sigma_\Psi}{2}\|\alpha_{t+1}-\alpha_t\|^2$
give
\begin{equation}\label{eq::dmd_FY}
    \eta_t\langle\xi_t,\alpha_t-\alpha_{t+1}\rangle
    \le\frac{\eta_t^2\|\xi_t\|_*^2}{2\sigma_\Psi}+\frac{\sigma_\Psi}{2}\|\alpha_t-\alpha_{t+1}\|^2
    \le\frac{\eta_t^2\|\xi_t\|_*^2}{2\sigma_\Psi}+\mathcal{B}(\alpha_{t+1},\alpha_t).
\end{equation}
Multiply Eq.~\eqref{eq::dmd_lin} by $\eta_t$, and substitute
Eqs.~\eqref{eq::dmd_threepoint} and~\eqref{eq::dmd_FY} for the two inner
products. The memory-centre curvature terms then combine as
$-\rho_t\mathcal{B}(\alpha_{t+1},\alpha_t)+\mathcal{B}(\alpha_{t+1},\alpha_t)=(1-\rho_t)\mathcal{B}(\alpha_{t+1},\alpha_t)\ge0$,
and we retain this nonnegative remainder. Among the remaining terms, we retain
$-\rho_t\mathcal{B}(\gamma_t,\alpha_{t+1})$ for the memory telescope. The two
anchor terms $-(1-\rho_t)\mathcal{B}(\gamma_t,\alpha_{t+1})\le0$ and
$-(1-\rho_t)\mathcal{B}(\alpha_{t+1},\bar\alpha)\le0$ may be dropped, yielding
a weaker valid bound. Dividing by $\eta_t>0$ gives
\begin{align}\label{eq::dmd_onestep}
    \ell_t(\alpha_t)-\ell_t(\gamma_t)
    \le\;&\frac{\rho_t}{\eta_t}\big[\mathcal{B}(\gamma_t,\alpha_t)-\mathcal{B}(\gamma_t,\alpha_{t+1})\big]
    +\frac{1-\rho_t}{\eta_t}\,\mathcal{B}(\gamma_t,\bar\alpha)\nonumber\\
    &+\frac{1-\rho_t}{\eta_t}\,\mathcal{B}(\alpha_{t+1},\alpha_t)
    +\frac{\eta_t\|\xi_t\|_*^2}{2\sigma_\Psi}.
\end{align}
When $\rho_t=1$, the two terms weighted by $1-\rho_t$ vanish, and
Eq.~\eqref{eq::dmd_onestep} reduces to the one-step inequality
\eqref{eq::mirror_applied} used for Theorem~\ref{thm:dynamic_regret_biom}. When
$\rho_t<1$, the Fenchel--Young bound contributes
$\mathcal B(\alpha_{t+1},\alpha_t)$, while the variational inequality
contributes $-\rho_t\mathcal B(\alpha_{t+1},\alpha_t)$, leaving the
nonnegative remainder
$(1-\rho_t)\mathcal B(\alpha_{t+1},\alpha_t)$. Retaining this remainder
preserves the gradient-energy coefficient $1/(2\sigma_\Psi)$. Scaling
Fenchel--Young to match $\rho_t$ would instead produce
$1/(2\sigma_\Psi\rho_t)$. The retained term is the second Bregman component of
$\mathcal B_N^{\mathrm{ref}}$ in Eq.~\eqref{eq::dmd_bias}.

\medskip
\noindent\emph{Step 4: summation and telescoping.} Summing
Eq.~\eqref{eq::dmd_onestep} over $t=1,\ldots,N$ yields the gradient-energy term
$\frac{1}{2\sigma_\Psi}\sum_{t=1}^N\eta_t\|\xi_t\|_*^2$, as in
Theorem~\ref{thm:dynamic_regret_biom}. Define the memory-centre contribution by
\begin{equation*}
    \mathcal S_N:=\sum_{t=1}^N w_t\big[\mathcal{B}(\gamma_t,\alpha_t)-\mathcal{B}(\gamma_t,\alpha_{t+1})\big],\qquad w_t:=\frac{\rho_t}{\eta_t}.
\end{equation*}
This is the same telescoping sum as in the proof of
Theorem~\ref{thm:dynamic_regret_biom}, with $w_t=\rho_t/\eta_t$ in place of
$1/\eta_t$. Reindexing produces the terminal term
$-w_N\mathcal{B}(\gamma_N,\alpha_{N+1})$. Because $w_N\ge0$ and
$\mathcal{B}\ge0$, this term may be dropped, giving
\begin{equation*}
    \mathcal S_N\le w_1\mathcal{B}(\gamma_1,\alpha_1)
    +\sum_{t=2}^N\Big\{w_{t-1}\big[\mathcal{B}(\gamma_t,\alpha_t)-\mathcal{B}(\gamma_{t-1},\alpha_t)\big]
    +(w_t-w_{t-1})\,\mathcal{B}(\gamma_t,\alpha_t)\Big\}.
\end{equation*}
By the $\mathcal K$-specific path-increment definition in
Theorem~\ref{thm:dynamic_regret_dmd},
$\mathcal{B}(\gamma_t,\alpha_t)-\mathcal{B}(\gamma_{t-1},\alpha_t)\le
\Delta_{\Psi,\mathcal K}(\gamma_t,\gamma_{t-1})$. Since $w_1\ge0$ and
$w_t-w_{t-1}\ge0$, the diameter bound applies to the remaining Bregman terms,
and their coefficients telescope:
\(
w_1\Omega_{\Psi,\mathcal K}+\sum_{t=2}^N(w_t-w_{t-1})\Omega_{\Psi,\mathcal K}
=w_N\Omega_{\Psi,\mathcal K}=\rho_N\Omega_{\Psi,\mathcal K}/\eta_N
\).
Hence
\begin{equation}\label{eq::dmd_Sbound}
    \mathcal S_N\le\frac{\rho_N\Omega_{\Psi,\mathcal K}}{\eta_N}+\sum_{t=2}^N\frac{\rho_{t-1}}{\eta_{t-1}}\,\Delta_{\Psi,\mathcal K}(\gamma_t,\gamma_{t-1}).
\end{equation}
The two terms in Eq.~\eqref{eq::dmd_onestep} weighted by
$(1-\rho_t)/\eta_t$ accumulate pointwise. With
$\kappa_t=(1-\rho_t)/\eta_t$, their sum defines the reference-gain remainder
\begin{equation}\label{eq::dmd_bias}
    \mathcal{B}^{\mathrm{ref}}_N:=\sum_{t=1}^N\frac{1-\rho_t}{\eta_t}\big[\mathcal{B}(\gamma_t,\bar\alpha)+\mathcal{B}(\alpha_{t+1},\alpha_t)\big]
    =\sum_{t=1}^N\kappa_t\big[\mathcal{B}(\gamma_t,\bar\alpha)+\mathcal{B}(\alpha_{t+1},\alpha_t)\big].
\end{equation}
Combining the gradient-energy term with Eqs.~\eqref{eq::dmd_Sbound}
and~\eqref{eq::dmd_bias} gives \eqref{eq:dynamic_regret_dmd_bound}.
For the looser bound \eqref{eq:dynamic_regret_dmd_bound_loose}, use
$\mathcal{B}(\alpha_{t+1},\alpha_t)\le\Omega_{\Psi,\mathcal K}$, $\rho_N\le1$, and
$\rho_{t-1}/\eta_{t-1}\le1/\eta_t$. The last inequality follows because
$(\eta_t)$ is non-increasing and $\rho_{t-1}\le1$. This completes the proof.
\end{proof}

\subsubsection{Proof of Corollary~\ref{cor::dmd_sqrt}}\label{app::proof_dmd_sqrt}
\begin{proof}
Substituting $\rho_t=\rho$ and $\eta_t=\eta t^{-1/2}$ into
Eq.~\eqref{eq:dynamic_regret_dmd_bound} gives
$\kappa_t=(1-\rho)\sqrt{t}/\eta$. Moreover,
$\rho_N/\eta_N=\rho\sqrt{N}/\eta$,
$\rho_{t-1}/\eta_{t-1}\le \rho\sqrt{N}/\eta$, and
$\sum_{t=1}^N\eta_t\le2\eta\sqrt{N}$.
Together with $\|\xi_t\|_*\le G$ and
$\mathcal B_\Psi(\alpha_{t+1},\alpha_t)\le\Omega_{\Psi,\mathcal K}$,
these relations give the stated inequality.
\end{proof}

\subsubsection{Proof of Corollary~\ref{cor::dmd_coupled}}\label{app::proof_dmd_coupled}
\begin{proof}
Under the coupled schedule, $\kappa_t=(1-\rho_t)/\eta_t=\kappa$ and
$w_t=1/\eta_t-\kappa$ is non-decreasing. Since
$\eta_t=\eta t^{-1/2}$, we have
$\eta_N^{-1}=\eta^{-1}\sqrt N$,
$\eta_t^{-1}\le\eta^{-1}\sqrt N$, and
$\sum_{t=1}^N\eta_t\le2\eta\sqrt N$.
Substituting these relations and $\|\xi_t\|_*\le G$ into
Eq.~\eqref{eq:dynamic_regret_dmd_bound_loose} gives the first inequality in
Corollary~\ref{cor::dmd_coupled}; dividing by $N$ gives the second.
Under $V_N^{\Psi,\mathcal K}(\gamma)=o(\sqrt N)$ and
\[
\frac{1}{N}\sum_{t=1}^N\mathcal B_\Psi(\gamma_t,\bar\alpha)
\longrightarrow\overline{\mathcal B},
\]
the terminal, path-variation, and gradient-energy terms in the average bound
converge to zero, while the reference-gain term converges to
$\kappa(\Omega_{\Psi,\mathcal K}+\overline{\mathcal B})$. Taking the upper
limit gives the stated result.
\end{proof}

\section{Numerical and empirical implementation details}\label{app::numdetails}
\setcounter{equation}{0}
\setcounter{table}{0}
\setcounter{figure}{0}
This appendix provides the estimation procedures, parameter constraints,
empirical implementation details, and supplementary numerical results supporting
Sections~\ref{sec::SectionSix} and~\ref{sec:empirical}.

\subsection{Estimation protocol and pilot-interval selection}\label{app::protocol}
For each experiment, the bounded-gain domain is fixed before the method
comparison and shared across all bounded rules. Its construction reflects the
experiment's design. In Experiments~6.1 and~6.2, candidate intervals are based
on robust quantiles of a preliminary discounted-logistic gain path fitted over
a broad range. Candidate interval-link pairs are compared using the following
mobility-weighted path criterion, motivated by Eq.~\eqref{eq::mobility}, and the
selected pair is then refitted under the experiment-specific in-sample objective:
\begin{equation*}
    \frac{1}{2} \sum_{t,j}\frac{(\alpha_{j,t+1}^{\text{pilot}}-\alpha_{j,t}^{\text{pilot}})^2}{g_j'(g_j^{-1}(\alpha_{j,t}^{\text{pilot}}))}.
\end{equation*}
The switching benchmark fixes $\mathcal A=[0,1]$ for its bounded rules,
while the regret diagnostic fixes both its compact interval and learning-rate
schedule.

For each market in the empirical panel, the constant-gain and
\textrm{DMD-logit} specifications are fitted on the initial $1000$ observations
over $[0.001,1]$. The resulting interval covers the constant-gain estimate and
the \textrm{DMD-logit} pilot path, is widened by $10\%$ of their joint span, and
is clipped to the broad range. It is then held fixed across bounded rules and
all subsequent out-of-sample refits.

\paragraph{\textbf{Link local mobilities}} For any interval $\mathcal{A}=[\alpha_{L},\alpha_{H}]$, write $D = \alpha_{H}-\alpha_{L}$ and $u = (\alpha-\alpha_L)/D$. The local mobilities of the logistic, cloglog and reverse-cloglog links are, respectively, $D\, u\,(1-u)$, $D\,(1-u)\,(-\log(1-u))$, and $D\,u\,(-\log u)$. Each mobility measures where on the admissible interval a latent-coordinate step most moves the gain. For diagonal vector gains, these formulas apply coordinate-wise.

\subsection{Estimation objective, parameter constraints, and optimisation}\label{app::estimation}
The local-level, switching, and empirical comparisons estimate their static
parameters by minimising the cumulative negative log predictive density along
each filtered path,
\begin{equation}\label{eq::estimation_objective}
    \hat\theta=\argmin_{\theta}\ \sum_{t}-\log p\big(y_{t}\mid\lambda_{t}(\theta)\big).
\end{equation}
Here $\theta$ contains the static parameters relevant to each specification:
the state intercept $\omega_\lambda$ and persistence $\beta_\lambda$; the gain
persistence $\rho$ and learning rate $\eta$; either the gain-recursion intercept
$\omega$ or the reference gain $\bar\alpha$; and, when estimated, the initial
state and observation-density parameters. The two gain-location
parameterisations are equivalent through
$\bar\alpha=g(\omega/(1-\rho))$, so each specification estimates only one of
them. Experiment~6.2 instead minimises
the quadratic state loss defined in Section~6.2. In the empirical out-of-sample
comparison, QLIKE~\cite{patton2011volatility} is used as the evaluation loss.

Table~\ref{tab:constraints} lists the parameter constraints and the transformations
used to impose them.

\begin{table}[!ht]
\centering
\small
\begin{tabular}{llll}
\toprule
Parameter & Symbol & Constraint & Enforced by \\
\midrule
State intercept          & $\omega_\lambda$ & $\in\mathbb{R}$        & --- \\
State persistence        & $\beta_\lambda$  & $\in(0,1)$            & sigmoid \\
Gain (bounded rules)     & $\alpha_t$       & $\in[\alpha_L,\alpha_H]$        & link $g$ or projection \\
Gain-recursion location  & $\omega$ or $\bar\alpha$ & $\bar\alpha=g(\omega/(1-\rho))$ & coordinate choice \\
Gain persistence         & $\rho$           & $\in(0,1)$            & sigmoid \\
Synthetic fitted step    & $\eta$           & $>0$                  & exponential reparam. \\
Switching feedback step  & $\eta$           & $>0$                  & softplus \\
Empirical bounded feedback & $\eta$         & $\geq0$               & optimiser bound \\
Empirical exp.\ feedback & $a_f$            & $\in\mathbb R$        & unrestricted \\
Student-$t$ d.o.f.       & $\nu_{\mathrm{df}}$ & $>2$               & $2+\mathrm{softplus}$ \\
Gaussian scale           & $\sigma$         & $>0$                  & softplus \\
\bottomrule
\end{tabular}
\caption{\textbf{Parameter constraints.} Constraints use reparametrisations or
direct optimiser bounds as indicated; optimisation is not uniformly
unconstrained in transformed coordinates. Bounded gain paths are enforced by
their link or projection. The exponential benchmark has no model-level gain
interval and is clipped only for numerical stability.}
\label{tab:constraints}
\end{table}

For numerical stability, the switching implementation uses
$\alpha_t=\mathrm e^{f_{t+1}/2}$ with $f_{t+1}$ clipped to $[-12,10]$, while
the empirical implementation uses $\alpha_t=\mathrm e^{f_{t+1}}$ with
$f_{t+1}$ clipped to $[-25,4]$. The corresponding gain ceilings are
$\mathrm e^5\simeq148.4$ and $\mathrm e^4\simeq54.6$. Before clipping, the two
normalisations are equivalent after rescaling $f$, its intercept, and its
score-product coefficient; their actual clip intervals and ceilings are not
equivalent.

The retained empirical fits include the boundary and signed-feedback cases
allowed by these constraints. Among the $492$ expanding-window fits for each
method, the exponential benchmark selects $a_f<0$ in $101$ windows. The
bounded-rule feedback estimate is exactly zero in $169$ \textrm{MD-logit},
$78$ \textrm{DMD-logit}, $11$ \textrm{DMD-proj}, and $275$ \textrm{AdaGrad}
windows. One of the six in-sample exponential-benchmark fits also has
$a_f<0$. These are valid fitted recursions, but negative exponential feedback
and zero bounded feedback are outside the positive-step mirror-descent and
regret interpretations.

Static parameters are estimated by L-BFGS-B for smooth objectives. In the
switching experiment, the clipped \textrm{DMD-exp} and projected
\textrm{DMD-proj} objectives are instead optimised by multistart Powell search
followed by Nelder--Mead. In the empirical engine, the clipped
\textrm{DMD-exp} and \textrm{DMD-proj} fits also receive deterministic
Powell--Nelder--Mead confirmation and coordinate-profile checks. The switching
and empirical fits estimate their initial states, whereas Experiment~6.1 fixes
$m_1=0$ for every method, matching the Gaussian prior used by its Kalman oracle.

\paragraph{\textbf{Out-of-sample evaluation}}
Each method is estimated separately and therefore generates its own filtered path.
Static parameters are re-estimated on expanding windows, while each market's
admissible gain interval $[\alpha_L,\alpha_H]$ is selected once from a pilot path
fitted on the initial training window and then held fixed throughout the evaluation.
Diebold--Mariano statistics use Newey--West HAC standard errors. For each
market, the Model Confidence Set uses the range statistic and a moving-block
bootstrap with block length $\lceil n^{1/3}\rceil$, $B=3000$ replications, and
a $90\%$ confidence level. The panel comparison in
Section~\ref{sec:emp_oos} instead uses a whole-date block bootstrap, preserving
serial dependence and contemporaneous dependence across markets.

\begin{table}[!ht]
\centering
\small
\resizebox{\textwidth}{!}{%
\begin{tabular}{lllccr}
\toprule
Index & Yahoo symbol & Country & Sample & Obs & OOS \\
\midrule
S\&P 500          & \texttt{\textasciicircum GSPC}     & United States  & 2000--2024 & 6288 & 5288 \\
Nasdaq Composite  & \texttt{\textasciicircum IXIC}     & United States  & 2000--2024 & 6288 & 5288 \\
Dow Jones Ind.    & \texttt{\textasciicircum DJI}      & United States  & 2000--2024 & 6288 & 5288 \\
FTSE 100          & \texttt{\textasciicircum FTSE}     & United Kingdom & 2000--2024 & 6313 & 5313 \\
DAX               & \texttt{\textasciicircum GDAXI}    & Germany        & 2000--2024 & 6347 & 5347 \\
CAC 40            & \texttt{\textasciicircum FCHI}     & France         & 2000--2024 & 6389 & 5389 \\
EURO STOXX 50     & \texttt{\textasciicircum STOXX50E} & Euro area      & 2007--2024 & 4451 & 3451 \\
Nikkei 225        & \texttt{\textasciicircum N225}     & Japan          & 2000--2024 & 6125 & 5125 \\
Hang Seng         & \texttt{\textasciicircum HSI}      & Hong Kong      & 2000--2024 & 6159 & 5159 \\
S\&P/TSX Comp.    & \texttt{\textasciicircum GSPTSE}   & Canada         & 2000--2024 & 6278 & 5278 \\
S\&P/ASX 200      & \texttt{\textasciicircum AXJO}     & Australia      & 2000--2024 & 6315 & 5315 \\
Bovespa           & \texttt{\textasciicircum BVSP}     & Brazil         & 2000--2024 & 6190 & 5190 \\
\bottomrule
\end{tabular}}
\caption{\textbf{Out-of-sample panel data.} Twelve equity indices using daily
open--high--low--close prices from Yahoo Finance. After constructing the
same-day Parkinson proxy, we drop the first otherwise valid OHLC day to
preserve the production panel's date-index alignment convention. ``Obs'' is
the resulting number of retained trading-day observations; ``OOS'' is the
evaluation length after the initial $1000$-day window. The EURO STOXX 50
series begins in 2007. Prices are unadjusted daily OHLC; realised variance is
floored at $10^{-8}$ before taking logarithms and is not winsorised.}
\label{tab:emp_panel_data}
\end{table}

\begin{table}[!ht]
\centering
\small
\begin{tabular}{lr@{\qquad}lr}
\toprule
Market & $[\alpha_L,\alpha_H]$ & Market & $[\alpha_L,\alpha_H]$ \\
\midrule
SPX     & $[0.05744,0.44781]$ & STOXX50 & $[0.00100,0.76345]$ \\
NDX     & $[0.00100,0.53750]$ & NIKKEI  & $[0.05229,0.41904]$ \\
DJI     & $[0.03585,0.47992]$ & HSI     & $[0.00906,0.49189]$ \\
FTSE    & $[0.00412,0.26981]$ & TSX     & $[0.124316,0.124403]$ \\
DAX     & $[0.00511,0.18748]$ & ASX     & $[0.00100,1.00000]$ \\
CAC     & $[0.00100,0.61410]$ & BVSP    & $[0.04191,0.36281]$ \\
\bottomrule
\end{tabular}
\caption{\textbf{Pilot-selected admissible intervals for bounded panel methods.}
Each interval is selected from the initial $1000$ observations and held fixed
across bounded gain rules and out-of-sample refits. The narrow TSX
interval implies that the bounded methods remain close to the constant-gain
benchmark for that market.}
\label{tab:emp_gain_intervals}
\end{table}

\FloatBarrier
\paragraph{\textbf{Out-of-sample realised-variance target and window schedule}}
\looseness=-2
\enlargethispage{\baselineskip}
For the panel in Section~\ref{sec:emp_oos}, let
$z_t=\log\widehat{\mathrm{RV}}_t$ denote the logarithm of the daily Parkinson
range-based variance proxy~\cite{parkinson1980extreme}. We model $z_t$ as a
Gaussian score-driven level. Conditional on
$\mathcal F_{t-1}$, $z_t\sim\mathcal N(h_t,\sigma^2)$, where
$h_{t+1}=\omega_\lambda+\beta_\lambda h_t+\alpha_t s_t$ and
$s_t=(z_t-h_t)/\sigma^2$; $\sigma$ is estimated. Forecasts are evaluated by the
Gaussian negative log-density for $z_t$ and by QLIKE using the corresponding
lognormal mean $\widehat{\mathrm{RV}}^{\,f}_t=\exp(h_t+\tfrac12\sigma^2)$ as
the variance forecast.
Forecasts use an expanding window initialised with $1000$ days and refitted every
$126$ trading days using all available observations. At each refit, the full
history is re-filtered under the updated static parameters before forecasting
resumes.

\FloatBarrier
\subsection{Gaussian local-level target schedule}\label{app::fixedscore}
Figure~\ref{fig:exp1_fixed_score_line_paths_biom} reports the state and gain
paths for the experiment in Subsection~\ref{subsec::one}.

\begin{figure}[!ht]
\centering
\includegraphics[width=1\linewidth]{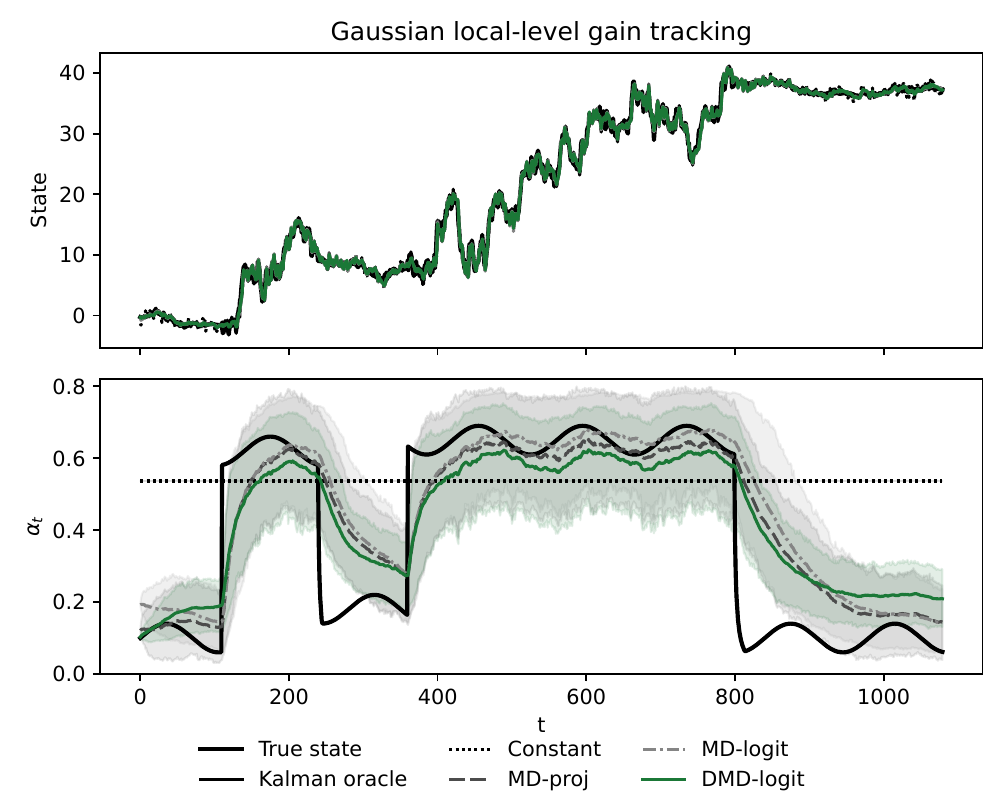}
\caption{\textbf{Gaussian local-level gain tracking.} The upper panel shows the latent state and filtered estimates; the lower panel shows the corresponding gain paths. Filters share the Gaussian score formula and observations, but their realised forecast errors are method-specific. The infeasible Kalman benchmark is the exact conditional-mean filter for the simulated prior and variance schedule.}
\label{fig:exp1_fixed_score_line_paths_biom}
\end{figure}

\paragraph{\textbf{Target schedule and Kalman-\(q\) inversion}} In the notation
of Eq.~\eqref{eq::schedule}, define the raw gain schedule as
\begin{equation*}
\begin{aligned}
  \widetilde\alpha_t
  &=\operatorname{clip}\{\mathcal P_t(0.10,0.62,0.18,0.65,0.10;110,240,360,800)\\
  &\hspace{4em}+0.04\sin\{2\pi(t-1)/140\},\,0.04,\,0.90\}.
\end{aligned}
\end{equation*}
The final target path \(\{\alpha_t^\star\}_{t=1}^T\) is obtained by projecting
\(\{\widetilde\alpha_t\}\) forward onto \(\mathcal A=[0.02,0.80]\) while enforcing
\begin{equation*}
  \alpha_{t+1}^\star\ge \frac{\alpha_t^\star}{1+\alpha_t^\star}+10^{-4}.
\end{equation*}
Let \(x_t\) denote
the latent state, \(m_t=\mathsf E[x_t\mid\mathcal F_{t-1}]\) the filter mean,
and \(P_t=\operatorname{Var}(x_t\mid\mathcal F_{t-1})\) the prediction
variance. For the Kalman filter associated with this target path,
\begin{equation*}
  \alpha_t^\star=\frac{P_t}{P_t+r},
  \qquad
  P_{t+1}=(1-\alpha_t^\star)P_t+q_t.
\end{equation*}
\looseness=-1
Inverting this Riccati recursion gives
\begin{equation}\label{eq::exp1_qfromalpha}
  q_t=P_{t+1}-r\,\alpha_t^\star=r\!\left(\frac{\alpha_{t+1}^\star}{1-\alpha_{t+1}^\star}-\alpha_t^\star\right)\ (t<T),\qquad q_T=q_{T-1},\qquad P_1=r\,\frac{\alpha_1^\star}{1-\alpha_1^\star}.
\end{equation}
For \(t<T\), the condition \(q_t\ge0\) is equivalent to
\(\alpha_{t+1}^\star\ge\alpha_t^\star/(1+\alpha_t^\star)\). The added
\(10^{-4}\) provides a strict numerical margin. We set \(r=1\). Because the
sinusoid is indexed by \(t-1\), \(\alpha_1^\star=0.1\), which implies \(P_1=1/9\) and an initial
standard deviation of \(1/3\). The DGP is \(y_t=x_t+\varepsilon_t\),
\(x_{t+1}=x_t+\upsilon_t\), with \(\varepsilon_t\sim\mathcal N(0,r)\) and
\(\upsilon_t\sim\mathcal N(0,q_t)\); the innovations are mutually independent
over time and independent of the initial state. The latent state is initialised as
\(x_1\sim\mathcal N(0,P_1)\), and every filter starts from \(m_1=0\). Under this
DGP and these initial conditions, the Kalman benchmark is the correctly specified
conditional-mean oracle.

\subsection{Bivariate scaling, link geometry and vector gains}\label{app::bivariate}
Figure~\ref{fig:exp2_score_geometry_biom} shows how the three score-scaling
conventions change the relative update magnitudes of the two coordinates.
Table~\ref{tab:exp2_geometry_memory_biom} compares link geometries and discounted
memory for a common scalar gain under unit scaling. Table~\ref{tab:exp2_unit_scalar_diagonal_biom} and
Figure~\ref{fig:exp2_diagonal_state_biom} compare common scalar and diagonal gains
under unit scaling. Figure~\ref{fig:exp2_heterogeneous_state_biom} shows a
representative state path for the heterogeneous inverse-Fisher-scaled design
summarized in Table~\ref{tab:exp2_heterogeneous_biom}.

\begin{figure}[!ht]
\centering
\includegraphics[width=1\linewidth]{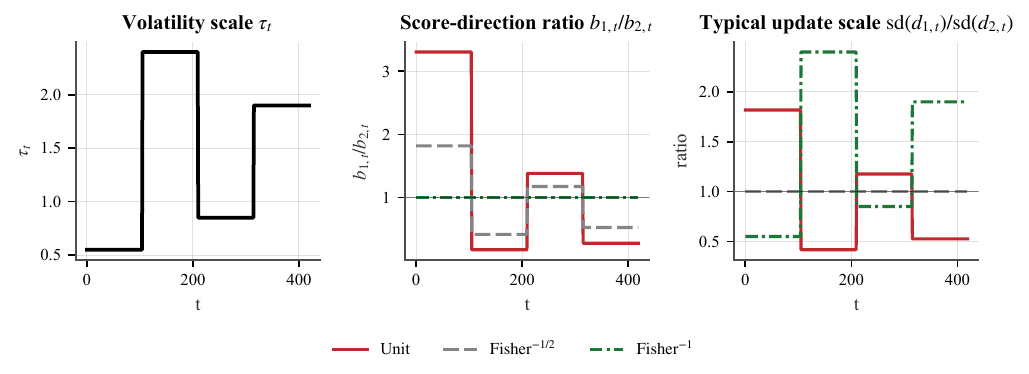}
\caption{\textbf{Bivariate score-scaling geometry.} The left panel shows the volatility scale $\tau_t$. The middle panel shows the scaling-coefficient ratio $b_{1,t}/b_{2,t}$ under unit, inverse-square-root Fisher, and inverse-Fisher scaling. The right panel shows the implied ratio of typical update scales $\mathrm{sd}(d_{1,t})/\mathrm{sd}(d_{2,t})$. Inverse-square-root Fisher scaling equalises these marginal update scales, whereas inverse-Fisher scaling expresses the designed state targets in common gain units.}
\label{fig:exp2_score_geometry_biom}
\end{figure}

For the link-geometry comparison at fixed unit scaling, \textrm{DMD-logit} has the lowest mean quadratic loss, whereas cloglog-OMD has the lowest state-target RMSE. Relative to the fitted constant gain, the paired \textrm{DMD-logit} loss reduction is $0.2570$ (standard error $0.0044$) and its paired pathwise state-target RMSE reduction is $0.2603$ (standard error $0.0029$). All methods share the designed target, score rule, scaling, and admissible gain interval. Link geometry governs local mobility, while discounting introduces reversion.

\begin{table}[t]
\centering
\small
\begin{tabular}{lrrrr}
\toprule
Method & Mean loss & Mean gain & State-target RMSE & \(\mu_2\) RMSE \\
\midrule
Constant & 0.3133 & 0.2402 & 1.0554 & 0.5421 \\
MD-proj & 0.0658 & 0.5724 & 0.8083 & 0.1653 \\
MD-logit & 0.0592 & 0.5854 & 0.8006 & 0.1614 \\
Cloglog-OMD & 0.0577 & 0.5985 & 0.7929 & 0.1625 \\
DMD-logit & 0.0563 & 0.6194 & 0.7936 & 0.1377 \\
\bottomrule
\end{tabular}
\caption{\textbf{Unit-scaled bivariate design: link geometry and discounted-memory comparison for a common scalar gain.} All methods use the same score and scaling rule and the same admissible gain interval. The differences arise from the gain update rule, link geometry, and memory structure.}
\label{tab:exp2_geometry_memory_biom}
\end{table}

\FloatBarrier

In the unit-scaled scalar-versus-diagonal comparison, the diagonal gain improves
the second coordinate while the first-coordinate feasibility gap remains. The
common \textrm{DMD-logit} gain has a state-target RMSE of $0.7936$, compared with
$0.7871$ for diagonal \textrm{DMD-logit}. The diagonal version reduces
second-coordinate RMSE from $0.1377$ to $0.1049$, while first-coordinate RMSE
changes little ($1.1138$ versus $1.1082$). Relative to common
\textrm{DMD-logit}, diagonal \textrm{DMD-logit} reduces second-coordinate
pathwise RMSE by $0.0305$ (standard error $0.0015$) and overall pathwise
state-target RMSE by $0.0063$ (standard error $0.0009$). The additional gain
dimension therefore yields a modest aggregate improvement concentrated in the
well-scaled second coordinate.

\begin{table}[!ht]
\centering
\scriptsize
\setlength{\tabcolsep}{3pt}
\begin{tabular}{llrrrr}
\toprule
Method & Gain & Mean loss & State-target RMSE & \(\mu_1\) RMSE & \(\mu_2\) RMSE \\
\midrule
Constant & common & 0.3133 & 1.0554 & 1.3907 & 0.5421 \\
MD-logit & common & 0.0592 & 0.8006 & 1.1207 & 0.1614 \\
DMD-logit & common & 0.0563 & 0.7936 & 1.1138 & 0.1377 \\
Constant & diagonal & 0.3146 & 1.0425 & 1.4229 & 0.3862 \\
MD-logit & diagonal & 0.1637 & 0.8039 & 1.1323 & 0.1021 \\
DMD-logit & diagonal & 0.1524 & 0.7871 & 1.1082 & 0.1049 \\
\bottomrule
\end{tabular}
\caption{\textbf{Common scalar versus diagonal vector gains under unit scaling.}
Entries are Monte Carlo averages over 256 simulated paths. The diagonal gain
improves the well-scaled second coordinate, while the first-coordinate feasibility
gap induced by unit scaling persists.}
\label{tab:exp2_unit_scalar_diagonal_biom}
\end{table}

\begin{figure}[t]
\centering
\includegraphics[width=1\linewidth]{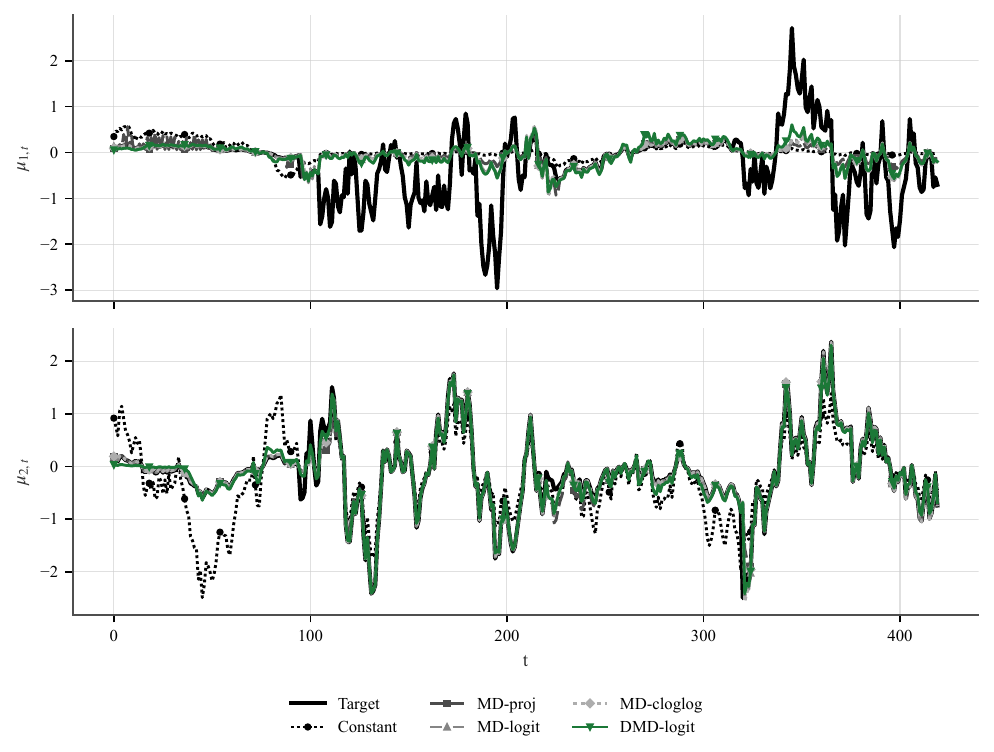}
\caption{\textbf{Representative state path under unit scaling with diagonal gains.} The diagonal gain improves the better-scaled second coordinate, while the first-coordinate target gap persists under unit scaling.}
\label{fig:exp2_diagonal_state_biom}
\end{figure}

\begin{figure}[t]
\centering
\includegraphics[width=1\linewidth]{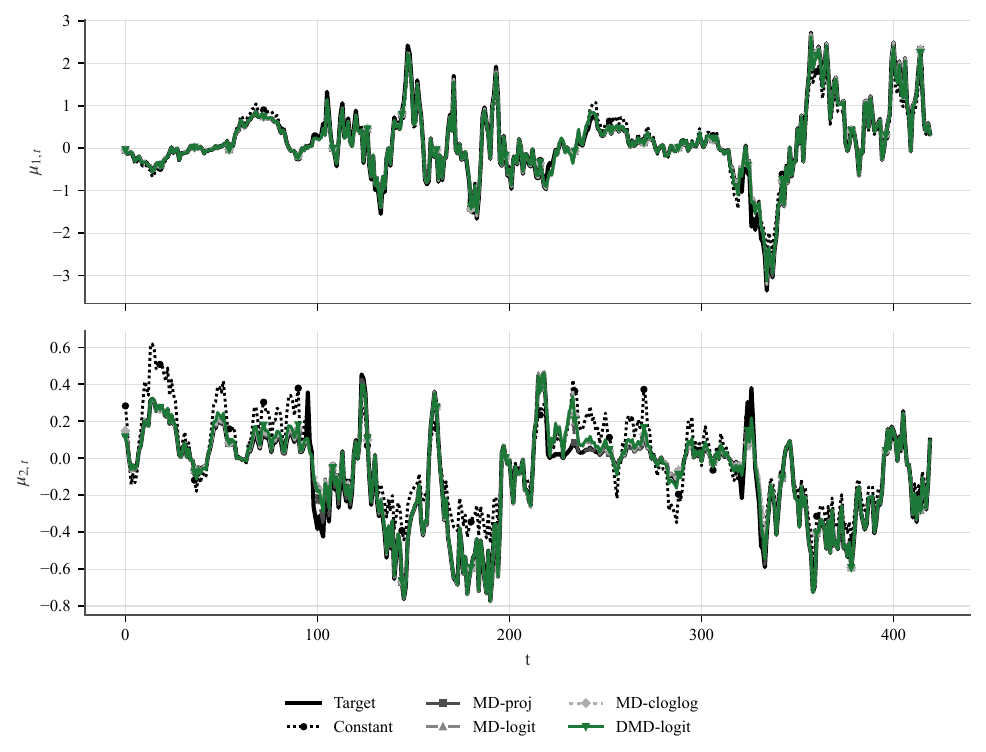}
\caption{\textbf{Representative state path in the heterogeneous inverse-Fisher-scaled design.} Under inverse-Fisher scaling and heterogeneous adjustment speeds, the adaptive filters track both state targets closely.}
\label{fig:exp2_heterogeneous_state_biom}
\end{figure}

\noindent\begin{minipage}{\linewidth}
\paragraph{\textbf{Data-generating process and designed loss}} For each of the
$N=256$ paths and each $j\in\{1,2\}$, draw $z_{j,1}\sim\mathcal N(0,1)$ and,
for $t=2,\ldots,T$, set
$z_{j,t}=\phi_jz_{j,t-1}+\sqrt{1-\phi_j^2}\varepsilon_{j,t}$. All initial
values and innovations are mutually independent, with
$\varepsilon_{j,t}\sim\mathcal N(0,1)$ and
$(\phi_1,\phi_2)=(0.97,0.92)$. Set
$e_{1,t}=\tau_tz_{1,t}$, $e_{2,t}=z_{2,t}$, and
$\Sigma_t=\diag(\tau_t^2,1)$. The designed state target is
$\delta_t^{\mathrm{state}}=\beta_t\odot e_t$. The AR process implies
$\mathsf E_t[e_{1,t+1}]=\phi_1(\tau_{t+1}/\tau_t)e_{1,t}$ and
$\mathsf E_t[e_{2,t+1}]=\phi_2e_{2,t}$. These conditional means generally
differ from $\delta_t^{\mathrm{state}}$, so evaluation uses the designed
quadratic state loss defined in Subsection~\ref{subsec::two}.
\end{minipage}

Scaling enters through $d_t=B_te_t$, where $B_t=\diag(b_{1,t},1)$,
$b_{1,t}=\tau_t^{-2}$ under unit scaling, $b_{1,t}=\tau_t^{-1}$ under
inverse-square-root Fisher scaling, and $b_{1,t}=1$ under inverse-Fisher
scaling. For a common scalar gain,
$L_t(\alpha)=\tfrac12\|\delta_t^{\mathrm{state}}-\alpha d_t\|_{\Sigma_t^{-1}}^2$.
Since $\chi_t>0$ almost surely, completing the square gives
$\tfrac12(\delta_t^{\mathrm{eff}}-\alpha e_t^{\mathrm{eff}})^2+c_t$, where
$e_t^{\mathrm{eff}}=\sqrt{\chi_t}$,
$\delta_t^{\mathrm{eff}}=\zeta_t/\sqrt{\chi_t}$,
$\chi_t=\sum_{j=1}^{2}d_{j,t}^2/(\Sigma_t)_{jj}$, and
$\zeta_t=\sum_{j=1}^{2}d_{j,t}\delta_{j,t}^{\mathrm{state}}/(\Sigma_t)_{jj}$.

\FloatBarrier
\subsection{Switching local-level benchmark: full grid}\label{app::blasquesfull}
Table~\ref{tab:blasques_full_biom} reports the full $\delta\times\gamma$ grid for the switching local-level benchmark of Subsection~\ref{subsec::four}; the main text reports the $\gamma=2.5$ slice (Table~\ref{tab:blasques_biom}).

\begin{table}[!ht]
\centering
\small
\begin{tabular}{ccrrrrrr}
\toprule
\(\delta\) & \(\gamma\) & Constant & \textrm{DMD-exp} & \textrm{DMD-logit} & \textrm{DMD-proj} & \textrm{MD-logit} & AdaGrad \\
\midrule
0.0 & 1.0 & 0.030 & 0.044 & 0.033 & \textbf{0.030} & 0.031 & \underline{0.030} \\
0.0 & 1.5 & 0.030 & 0.044 & 0.033 & \textbf{0.030} & 0.031 & \underline{0.030} \\
0.0 & 2.0 & 0.030 & 0.044 & 0.033 & \textbf{0.030} & 0.031 & \underline{0.030} \\
0.0 & 2.5 & 0.030 & 0.044 & 0.033 & \textbf{0.030} & 0.031 & \underline{0.030} \\
0.5 & 1.0 & \textbf{0.222} & 0.223 & 0.224 & 0.224 & 0.224 & \underline{0.223} \\
0.5 & 1.5 & \textbf{0.201} & 0.202 & 0.202 & \underline{0.202} & 0.202 & 0.202 \\
0.5 & 2.0 & \underline{0.181} & 0.181 & 0.181 & \textbf{0.181} & 0.181 & 0.181 \\
0.5 & 2.5 & 0.169 & 0.169 & \textbf{0.169} & 0.170 & 0.169 & \underline{0.169} \\
1.0 & 1.0 & 0.316 & \underline{0.314} & \textbf{0.314} & 0.315 & 0.316 & 0.316 \\
1.0 & 1.5 & 0.285 & \underline{0.282} & \textbf{0.282} & 0.284 & 0.285 & 0.285 \\
1.0 & 2.0 & 0.256 & \underline{0.251} & \textbf{0.251} & 0.253 & 0.256 & 0.256 \\
1.0 & 2.5 & 0.239 & \underline{0.232} & \textbf{0.231} & 0.235 & 0.238 & 0.238 \\
2.0 & 1.0 & 0.457 & \underline{0.435} & \textbf{0.432} & 0.437 & 0.457 & 0.457 \\
2.0 & 1.5 & 0.410 & \underline{0.377} & \textbf{0.372} & 0.379 & 0.410 & 0.410 \\
2.0 & 2.0 & 0.367 & \underline{0.323} & \textbf{0.317} & 0.326 & 0.366 & 0.367 \\
2.0 & 2.5 & 0.342 & \underline{0.290} & \textbf{0.285} & 0.293 & 0.339 & 0.340 \\
3.0 & 1.0 & 0.572 & 0.531 & \textbf{0.509} & \underline{0.519} & 0.572 & 0.572 \\
3.0 & 1.5 & 0.511 & 0.450 & \textbf{0.431} & \underline{0.440} & 0.510 & 0.511 \\
3.0 & 2.0 & 0.457 & 0.378 & \textbf{0.360} & \underline{0.368} & 0.451 & 0.455 \\
3.0 & 2.5 & 0.424 & 0.332 & \textbf{0.320} & \underline{0.325} & 0.412 & 0.418 \\
\bottomrule
\end{tabular}
\caption{\textbf{Switching local-level benchmark: full grid.} Entries are mean filtered-state RMSEs over \(M=1000\) replications of length \(T=1000\). The break size is \(\delta\), and each regime lasts \(\gamma\times10^2\) periods. The fitted specifications are estimated path by path by maximum likelihood; \textrm{DMD-exp} denotes the nominally unbounded exponential-link \textrm{aGAS} specification. Bold and underlined entries mark the lowest and second-lowest full-precision RMSEs in each row, so they may coincide after rounding. The main text reports the \(\gamma=2.5\) slice in Table~\ref{tab:blasques_biom}.}
\label{tab:blasques_full_biom}
\end{table}

\clearpage

\section*{Declaration of generative AI and AI-assisted technologies in the manuscript preparation process}

During the preparation of this work, the authors used ChatGPT for language
editing to improve the manuscript's clarity and readability. The authors
reviewed and edited the output as needed and take full responsibility for the
content of the published article.

\end{document}

%% file: tables/tab-emp-insample.tex
\begin{tabular}{lrrrrrr}
\toprule
Method & THB/USD (N) & THB/USD (t) & S\&P 500 (N) & S\&P 500 (t) & EUR/USD (N) & EUR/USD (t) \\ \midrule
Constant & 5547 & 921 & \textbf{18295} & \textbf{17984} & \textbf{9564} & \textbf{9063} \\
\textrm{DMD-exp} & 5510 & 922 & 18314 & 17995 & 9570 & 9076 \\
\textrm{MD-logit} & \textbf{5470} & 930 & 18304 & 17992 & 9573 & 9072 \\
\textrm{DMD-logit} & 5553 & \textbf{916} & 18297 & 17990 & 9582 & 9078 \\
\textrm{DMD-proj} & 5553 & 918 & 18302 & 17990 & 9579 & 9077 \\
\textrm{AdaGrad} & 5478 & 930 & 18304 & 17992 & 9573 & 9072 \\
\bottomrule
\end{tabular}